\documentclass[twoside]{article}

\usepackage[accepted]{aistats2019}
\usepackage[round]{natbib}

\usepackage{amssymb,amsmath,amsthm,bbm}
\usepackage{verbatim,float,url,dsfont}
\usepackage{graphicx,subfigure,psfrag}
\usepackage{algorithm,algorithmic}
\usepackage{mathtools,enumitem}
\usepackage{microtype}
\usepackage{xr-hyper}
\usepackage[colorlinks=true,citecolor=blue,urlcolor=blue,linkcolor=blue]{hyperref}

\newtheorem{theorem}{Theorem}
\newtheorem{lemma}{Lemma}

\theoremstyle{definition}
\newtheorem{remark}{Remark}

\newcommand{\minimize}{\mathop{\mathrm{minimize}}}

\def\C{\mathbb{C}}
\def\R{\mathbb{R}}
\def\E{\mathbb{E}}

\def\Cov{\mathrm{Cov}}

\def\ones{\mathds{1}}
\def\hbeta{\hat{\beta}}

\def\ie{i.e.}
\def\eg{e.g.}
\def\hSigma{\hat\Sigma}

\def\ini{\mathrm{in}}
\def\out{\mathrm{out}}
\def\ridge{\mathrm{ridge}}
\def\gf{\mathrm{gf}}

\def\Risk{\mathrm{Risk}}
\def\tr{\mathrm{tr}}
\def\cL{\mathcal{L}}
\def\Re{\mathrm{Re}}
\def\Im{\mathrm{Im}}

\newtheorem*{assumption*}{\assumptionnumber}
\providecommand{\assumptionnumber}{}
\makeatletter
\newenvironment{assumption}[2]{
  \renewcommand{\assumptionnumber}{Assumption #1#2}
  \begin{assumption*}
  \protected@edef\@currentlabel{#1#2}}
{\end{assumption*}}
\makeatother

\begin{document}

%

%

\twocolumn[

\aistatstitle{A Continuous-Time View of Early Stopping for Least Squares}    

\aistatsauthor{Alnur Ali \And J. Zico Kolter \And Ryan J. Tibshirani} 

\aistatsaddress{Carnegie Mellon University \And Carnegie Mellon University 
  \And Carnegie Mellon University}] 

\begin{abstract}
We study the statistical properties of the iterates generated by gradient
descent, applied to the fundamental problem of least squares regression.  We
take a continuous-time view, \ie, consider infinitesimal step sizes in gradient
descent, in which case the iterates form a trajectory called {\it gradient
  flow}.  Our primary focus is to compare the risk of gradient flow to that of
ridge regression.  Under the calibration $t=1/\lambda$---where $t$ is the time
parameter in gradient flow, and $\lambda$ the tuning parameter in ridge
regression---we prove that the risk of gradient flow is no more than 1.69 times
that of ridge, along the entire path (for all $t \geq 0$).  This holds in finite
samples with very weak assumptions on the data model (in particular, with no
assumptions on the features $X$).  We prove that the same relative risk
bound holds for prediction risk, in an average sense over the underlying
signal $\beta_0$. 
Finally, we examine limiting risk expressions (under standard Marchenko-Pastur  
asymptotics), and give supporting numerical experiments.
\end{abstract}

\section{INTRODUCTION}

Given the sizes of modern data sets, there is a growing preference towards
simple estimators that have a small computational footprint and are easy to
implement.  Additionally, beyond efficiency and tractability considerations,
there is mounting evidence that many simple and popular estimation methods
perform a kind of {\it implicit regularization}, meaning that they appear to
produce estimates exhibiting a kind of regularity, even though they do not
employ an explicit regularizer.  

Research interest in implicit regularization is growing, but the foundations of
the idea date back at least 30 years in machine learning, where early-stopped
gradient descent was found to be effective in training neural networks
\citep{morgan1989generalization}, and at least 40 years in applied mathematics,
where the same idea (here known as early-stopped Landweber iterations) was found  
ill-posed linear inverse problems \citep{strand1974theory}. After a wave of
research on boosting with early stopping
\citep{buhlmann2003boosting,rosset2004boosting,zhang2005boosting,yao2007early}, 
more recent work focuses on the regularity properties of particular algorithms
for underdetermined problems in matrix factorization, regression, and
classification
\citep{gunasekar2017implicit,wilson2017marginal,gunasekar2018characterizing}. More 
broadly, algorithmic regularization plays a key role in training deep neural
networks, via batch normalization, dropout, and other techniques.

In this paper, we focus on early stopping in gradient descent, when applied
specifically to least squares regression.  This is a basic problem and we are of
course not the only authors to consider it; there is now a large literature on
this topic (see references above, and more to come when we discuss related work
shortly). However, our perspective differs from existing work in a few
important ways: first, we study gradient descent in continuous-time (\ie, with
infinitesimal step sizes), leading to a path of iterates known as {\it gradient
  flow}; second, we examine the regularity properties along {\it the entire
  path}, not just its convergence point (as is the focus in most of the work on 
implicit regularization); and third, we focus on analyzing and comparing the
{\it risk} of gradient flow directly, which is arguably what we care about the
most, in many applications. 

A strength of the continuous-time perspective is that it facilitates the 
comparison between early stopping and $\ell_2$ regularization. While the
connection between these two mechanisms
has been studied by many authors (and from many angles), our paper
provides some of the strongest evidence for this connection to date.             

\paragraph{Summary of Contributions.} Our contributions in this paper are as
follows.    

\begin{itemize}
\item We prove that, in finite samples, under very weak assumptions on the data 
  model (and with no assumptions on the feature matrix $X$), the
  estimation risk of gradient flow at time $t$ is no more than 1.69 that of
  ridge regression at tuning parameter $\lambda=1/t$, for all $t \geq 0$. 

\item We show that the same result holds for in-sample prediction risk.

\item We show that the same result is also true for out-of-sample prediction
  risk, but now in an average (Bayes) sense, with respect to a spherical prior
  on the underlying signal $\beta_0$.

\item For Bayes risk, under optimal tuning, our results on estimation, in-sample
  prediction, and  out-of-sample prediction risks can all be tightened. We prove
  that the relative risk (measured in any of these three ways) of
  optimally-tuned gradient flow to optimally-tuned ridge is in between 1 and
  1.22.    

\item We derive exact limiting formulae for the risk of gradient flow, in a
  Marchenko-Pastur asymptotic model where $p/n$ (the ratio of the
  feature dimension to sample size) converges to a positive constant.  We
  compare these to known limiting formulae for ridge regression. 

\item We support our theoretical results with numerical simulations that show 
  the coupling between gradient flow and ridge can be extremely tight in
  practice (even tighter than suggested by theory). 
\end{itemize}

\paragraph{Related Work.} Various authors have made connections between $\ell_2$
regularization and the iterates generated by gradient descent (when applied to
different loss functions of interest): \citet{friedman2004gradient} were among
the first make this explicit, and gave supporting numerical experiments,
followed by \citet{ramsay2005parameter}, who adopted a continuous-time (gradient
flow) view, as we do. \citet{yao2007early} point out that early stopped gradient
descent is a spectral filter, just like $\ell_2$ regularization. Subsequent work
in nonparametric data models (specifically, reproducing kernel Hilbert space
models), studied early-stopped gradient descent from the perspective of risk
bounds, where it is shown to perform comparably to explicit $\ell_2$
regularization, when each method is optimally tuned
\citep{bauer2007regularization,logerfo2008spectral,raskutti2014early,wei2017early}. 
Other works have focused on the bias-variance trade-off in early-stopped
gradient boosting \citep{buhlmann2003boosting,zhang2005boosting}.

After completing this work, we became aware of the interesting recent paper by
\citet{suggula2018connecting}, who gave deterministic bounds between gradient 
flow and ridge regularized estimates, for problems in which the loss function is
strongly convex.  Their results are very different from ours: they apply to a
much wider variety of problem settings (not just least squares problems),
and are driven entirely by properties associated with strong convexity; our
analysis, specific to least squares regression, is much more precise, and covers
the important high-dimensional case (in which the strong convexity assumption is
violated). 


There is also a lot of related work on theory for ridge regression.  Recently,
\citet{dobriban2018high} studied ridge regression (and regularized discriminant
analysis) in a Marchenko-Pastur asymptotics model, deriving limiting risk
expressions, and the precise form of the limiting optimal tuning
parameter. \citet{dicker2016ridge} gave a similar asymptotic analysis for ridge,
but under a somewhat different problem setup.  \citet{hsu2012random} established
finite-sample concentration bounds for ridge risk. Low-dimensional theory for
ridge dates back much further, see \citet{goldenshluger2001adaptive} and others.
Lastly, we point out an interesting risk inflation result in that is vaguely
related to ours: \citet{dhillon2013risk} showed that risk of principal
components regression is at most four times that of ridge, under a natural
calibration between these two estimator paths (coupling the eigenvalue threshold
for the sample covariance matrix with the ridge tuning parameter).

\paragraph{Outline.} Here is an outline for the rest of the paper.  Section
\ref{sec:prelim} covers preliminary material, on the problem and estimators to
be considered.  Section \ref{sec:basic_results} gives basic results on gradient 
flow, and its relationship to ridge regression.  Section \ref{sec:risk_results}
derives expressions for the estimation risk and prediction risk of gradient flow
and ridge. Section \ref{sec:risk_relative} presents our main results on
relative risk bounds (of gradient flow to ridge).  Section \ref{sec:risk_asymp} 
studies the limiting risk of gradient flow under standard Marchenko-Pastur
asymptotics.  Section \ref{sec:exps} presents numerical examples that support
our theoretical results, and Section \ref{sec:disc} concludes with a short
discussion. 

\section{PRELIMINARIES}
\label{sec:prelim}

\subsection{Least Squares, Gradient Flow, and Ridge}

Let $y \in \R^n$ and $X \in \R^{n\times p}$ be a response vector and a matrix of
predictors or features, respectively. Consider the standard (linear) least
squares problem    
\begin{equation}
\label{eq:ls}
\minimize_{\beta \in \R^p} \; \frac{1}{2n} \|y-X\beta\|_2^2.
\end{equation}
Consider gradient descent applied to \eqref{eq:ls}, with a constant step size 
$\epsilon>0$, and initialized at \smash{$\beta^{(0)}=0$}, which repeats the 
iterations 
\begin{equation}
\label{eq:gd} 
\beta^{(k)} = \beta^{(k-1)} + \epsilon \cdot \frac{X^T}{n} 
(y - X \beta^{(k-1)}), 
\end{equation}
for $k=1,2,3,\ldots$.  Letting $\epsilon \to 0$, we get a continuous-time  
ordinary differential equation 
\begin{equation}
\label{eq:gf} 
\dot\beta(t) = \frac{X^T}{n} (y - X \beta(t)),
\end{equation}
over time $t \geq 0$, subject to an initial condition $\beta(0)=0$.
We call \eqref{eq:gf} the {\it gradient flow} differential equation for the
least squares problem \eqref{eq:ls}.  

To see the connection between \eqref{eq:gd} and \eqref{eq:gf}, we simply
rearrange \eqref{eq:gd} to find that 
$$
\frac{\beta^{(k)}-\beta^{(k-1)}}{\epsilon} =  
\frac{X^T}{n} (y - X \beta^{(k-1)}), 
$$
and setting \smash{$\beta(t)=\beta^{(k)}$} at time $t=k\epsilon$, we recognize
the left-hand side above as the discrete derivative of $\beta(t)$ at time $t$,
which approaches its continuous-time derivative as $\epsilon \to 0$.  

In fact, starting from the differential equation \eqref{eq:gf}, we can view
gradient descent \eqref{eq:gd} as one of the most basic numerical analysis
techniques---the {\it forward Euler method}---for discretely approximating the
solution \eqref{eq:gf}.

Now consider the $\ell_2$ regularized version of \eqref{eq:ls}, called ridge
regression \citep{hoerl1976ridge}:  
\begin{equation}
\label{eq:ridge}
\minimize_{\beta \in \R^p} \; \frac{1}{n} \|y-X\beta\|_2^2 + \lambda 
\|\beta\|_2^2, 
\end{equation}
where $\lambda > 0$ is a tuning parameter.  The explicit ridge solution is   
\begin{equation}
\label{eq:ridge_sol}
\hbeta^\ridge(\lambda) = (X^T X + n \lambda I)^{-1} X^T y. 
\end{equation}
Though apparently unrelated, the ridge regression solution path and gradient 
flow path share striking similarities, and their relationship is our central
focus. 

\subsection{The Exact Gradient Flow Solution Path}

Thanks to our focus on least squares, the gradient flow differential equation in
\eqref{eq:gf} is a rather special one: it is a continuous-time linear
dynamical system, and has a well-known exact solution.  

\begin{lemma} 
\label{lem:gf}
Fix a response $y$ and predictor matrix $X$.  Then the gradient flow problem 
\eqref{eq:gf}, subject to $\beta(0)=0$, admits the exact solution 
\begin{equation}
\label{eq:gf_sol}
\hbeta^\gf(t) = (X^T X)^+ (I - \exp(-t X^T X/n)) X^T y,
\end{equation}
for all $t \geq 0$. Here $A^+$ is the Moore-Penrose generalized
inverse of a matrix $A$, and $\exp(A) = I + A + A^2/2! +A^3/3! + \cdots$ is
the matrix exponential. 
\end{lemma}

\vspace{-10pt}
\begin{proof}
This can be verified by differentiating \eqref{eq:gf_sol} and using basic
properties of the matrix exponential. 
\end{proof}

In continuous-time, early stopping corresponds to taking the estimator
\smash{$\hbeta^\gf(t)$} in \eqref{eq:gf_sol} for any finite value of $t \geq 0$, 
with smaller $t$ leading to greater regularization.  We can already see that
\eqref{eq:gf_sol}, like \eqref{eq:ridge_sol}, applies a type of shrinkage to the
least squares solution; their similarities will become more evident when we
express both in spectral form, as we will do shortly in Section
\ref{sec:spec_compare}.    

\subsection{Discretization Error}

In what follows, we will focus on (continuous-time) gradient flow rather than 
(discrete-time) gradient descent. Standard results from numerical analysis 
give uniform bounds between discretizations like the forward Euler method
(gradient descent) and the differential equation path (gradient flow).  In
particular, the next result is a direct application of Theorem 212A in
\citet{butcher2016numerical}.    

\begin{lemma}
\label{lem:gd_error}
For least squares, consider gradient descent \eqref{eq:gd} initialized at
\smash{$\beta^{(0)}=0$}, and gradient flow \eqref{eq:gf_sol}, subject to
$\beta(0)=0$.  For any step size $\epsilon < 1/s_{\max}$ where
$s_{\max}$ is the largest eigenvalue of \smash{$X^T X/n$},
and any $K \geq 1$,
$$
\max_{k=1,\ldots,k} | \beta^{(k)} - \hbeta^\gf(k \epsilon) | \leq \\   
\frac{\epsilon\|X^T y\|_2}{2n} 
(\exp(K \epsilon s_{\max}) - 1). 
$$
\end{lemma}

The results to come can therefore be translated to the discrete-time 
setting, by taking a small enough $\epsilon$ and invoking Lemma
\ref{lem:gd_error}, but we omit details for brevity.  

\section{BASIC COMPARISONS}
\label{sec:basic_results}

\subsection{Spectral Shrinkage Comparison}
\label{sec:spec_compare}

To compare the ridge \eqref{eq:ridge_sol} and gradient flow \eqref{eq:gf_sol}
paths, it helps to rewrite them in terms of the singular value decomposition of
$X$.  Let \smash{$X=\sqrt{n} U S^{1/2} V^T$} be a singular value
decomposition, so that $X^T X / n = V S V^T$ is an eigendecomposition.  Then
straightforward algebra brings \eqref{eq:ridge_sol}, \eqref{eq:gf_sol}, on the
scale of fitted values, to  
\begin{align}
\label{eq:ridge_fit}
X \hbeta^\ridge(\lambda) &= U S (S + \lambda I)^{-1} U^T y, \\
\label{eq:gf_fit}
X \hbeta^\gf(t) &= U (I - \exp(-t S)) U^T y.
\end{align}
Letting $s_i$, $i=1,\ldots,p$ denote the diagonal entries of
$S$, and $u_i \in \R^n$, $i=1,\ldots,p$ denote the columns of  
$U$, we see that \eqref{eq:ridge_fit}, \eqref{eq:gf_fit} are both linear
smoothers (linear functions of $y$) of the form 
$$
\sum_{i=1}^p g(s_i,\kappa) \cdot u_i u_i^T y,
$$
for a spectral shrinkage map $g(\cdot,\kappa) : [0,\infty) \to [0,\infty)$ and
parameter $\kappa$.  This map is
\smash{$g^\ridge(s,\lambda) = s/(s+\lambda)$} for ridge, and \smash{$g^\gf(s,t) 
  = 1-\exp(-ts)$} for gradient flow. We see both apply more shrinkage  
for smaller values of $s$, \ie, lower-variance directions of $X^T X/n$, but do
so in apparently different ways. 

While these shrinkage maps agree at the extreme ends (\ie, set $\lambda=0$ and
$t=\infty$, or set $\lambda=\infty$ and $t=0$), there is no single
parametrization for $\lambda$ as a function of $t$, say $\phi(t)$, that equates
\smash{$g^\ridge(\cdot,\phi(t))$} with \smash{$g^\gf(\cdot,t)$}, for all $t \geq
0$. But the parametrization $\phi(t)=1/t$ gives the two shrinkage maps grossly
similar behaviors: see Figure \ref{fig:maps} for a visualization.  Moreover, as
we will show later in Sections \ref{sec:risk_relative}--\ref{sec:exps}, the
two shrinkage maps (under the calibration $\phi(t)=1/t$) lead to similar risk
curves for ridge and gradient flow. 

\begin{figure}[htb]
\centering
\includegraphics[width=\columnwidth]{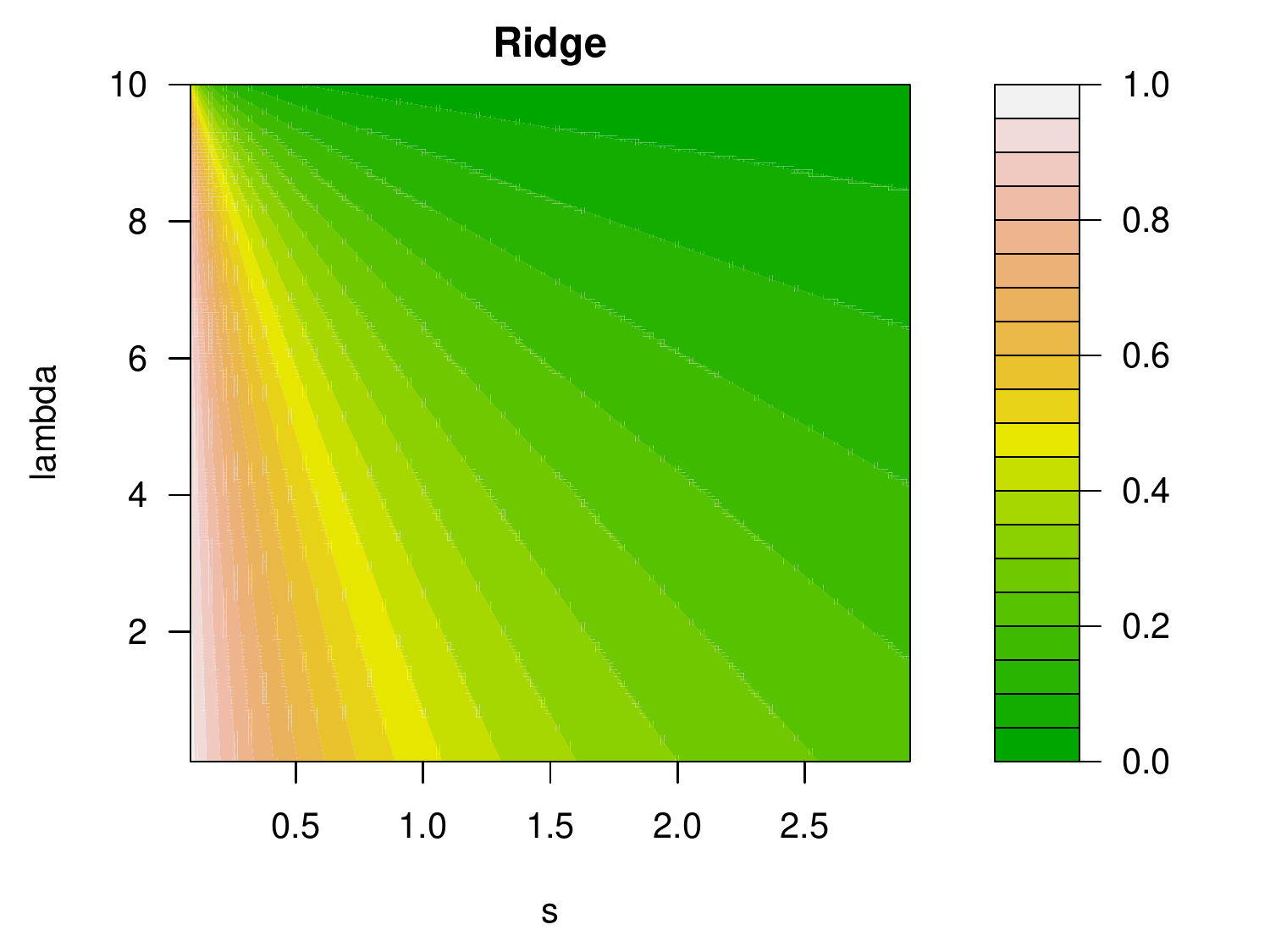} \\
\includegraphics[width=\columnwidth]{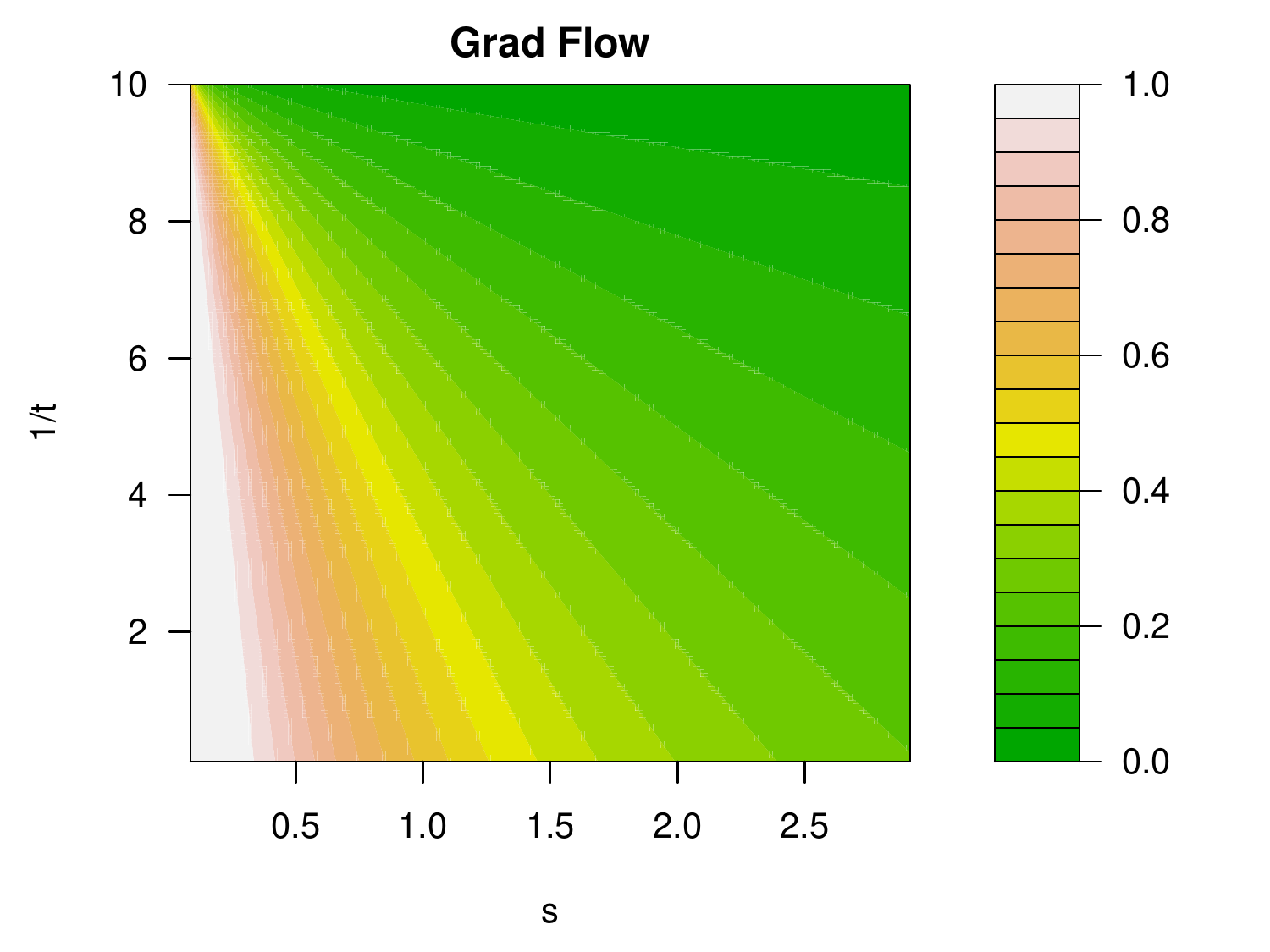}
\caption{\it \small Comparison of ridge and gradient flow spectral
  shrinkage maps, plotted as heatmaps over $(s,\lambda)$ (ridge) and $(s,t)$ 
  (gradient flow) with the calibration $\lambda=1/t$.}
\label{fig:maps}
\end{figure}

\subsection{Underlying Regularization Problems}
\label{sec:underlying}

Given our general interest in the connections between gradient descent and ridge 
regression, it is natural to wonder if gradient descent iterates can also be
expressed as solutions to a sequence of regularized least squares problems.  The
following two simple lemmas certify that this is in fact the case, in both
discrete- and continuous-time; their proofs may be found in the supplement. 

\begin{lemma} 
\label{lem:gd_opt}
Fix $y,X$, and let \smash{$X^TX/n=VSV^T$} be an eigendecomposition.  
Assume that we initialize \smash{$\beta^{(0)}=0$}, and we take the  
step size in gradient descent to satisfy $\epsilon < 1 / s_{\max}$, with
$s_{\max}$ denoting the largest eigenvalue of \smash{$X^T X/n$}. Then, for each
$k=1,2,3,\ldots$, the iterate \smash{$\beta^{(k)}$} from step $k$ in
gradient descent \eqref{eq:gd} uniquely solves the optimization problem
$$
\minimize_{\beta \in \R^p} \; \frac{1}{n} \| y - X \beta \|_2^2 + \beta^T Q_k
\beta, 
$$
where \smash{$Q_k =V S ((I - \epsilon S)^{-k} - I)^{-1} V^T$}.
\end{lemma}

\begin{lemma} 
\label{lem:gf_opt}
Fix $y,X$, and let \smash{$X^TX/n=VSV^T$} be an eigendecomposition.  
Under the initial condition $\beta(0)=0$, for all $t > 0$, the 
solution $\beta(t)$ of the gradient flow problem \eqref{eq:gf} uniquely solves
the optimization problem
$$
\minimize_{\beta \in \R^p} \; \frac{1}{n} \| y - X \beta \|_2^2 + \beta^T Q_t
\beta, 
$$
where \smash{$Q_t = V S (\exp(tS) - I)^{-1} V^T$}.
\end{lemma}

\begin{remark}
The optimization problems that underlie gradient descent and gradient flow, in  
Lemmas \ref{lem:gd_opt} and \ref{lem:gf_opt}, respectively, are both
quadratically regularized least squares problems.  In agreement 
with the intuition from the last subsection, we see that in both problems the
regularizers penalize the lower-variance directions of $X^T X/n$ more strongly,
and this is relaxed as $t$ or $k$ grow. The proof of the continuous-time is  
nearly immediate from \eqref{eq:gf_fit}; the proof of the discrete-time result 
requires a bit more work.  To see the link between the two results, set
$t=k\epsilon$, and note that as $k \to \infty$: 
$$
((1-ts/k)^{-k}-1)^{-1} \to (\exp(ts)-1)^{-1}. 
$$
\end{remark}

\section{MEASURES OF RISK}
\label{sec:risk_results}

\subsection{Estimation Risk}
\label{sec:estimation_risk}

We take the feature matrix $X \in \R^{n \times p}$ to be fixed and arbitrary,
and consider a generic response model,
\begin{equation}
\label{eq:model_y}
y | \beta_0 \sim (X\beta_0, \sigma^2 I),
\end{equation}
which we write to mean $\E(y | \beta_0) = X\beta_0$, $\Cov(y | \beta_0) =
\sigma^2 I$, for an underlying coefficient vector $\beta_0 \in \R^p$ and error
variance $\sigma^2 > 0$.  We consider a spherical prior,  
\begin{equation}
\label{eq:model_b0}
\beta_0 \sim (0, (r^2/p) I)
\end{equation}
for some signal strength $r^2 = \E \|\beta_0\|_2^2 > 0$.  

For an estimator \smash{$\hbeta$} (\ie, measurable function of $X,y$), we define
its estimation risk (or simply, risk) as 
$$
\Risk(\hbeta; \beta_0) = \E \big[\| \hbeta - \beta_0 \|_2^2 \, 
\big| \, \beta_0\big]. 
$$
We also define its Bayes risk as
\smash{$\Risk(\hbeta) = \E \| \hbeta - \beta_0 \|_2^2$}.

Next we give expressions for the risk and Bayes risk of gradient flow; the
derivations are straightforward and found in the supplement.  We denote by
$s_i$, $i=1,\ldots,p$ and $v_i$, $i=1,\ldots,p$ the eigenvalues and
eigenvectors, respectively, of $X^T X/n$.     

\begin{lemma} 
\label{lem:gf_risk}
Under the data model \eqref{eq:model_y}, for any $t \geq 0$, the risk of the 
gradient flow estimator \eqref{eq:gf_sol} is    
\begin{multline}
\label{eq:gf_risk}
\Risk(\hbeta^\gf(t); \beta_0) = \\
\sum_{i=1}^p \bigg(|v_i^T \beta_0|^2 \exp(-2 t s_i) + 
\frac{\sigma^2}{n} \frac{(1 - \exp(-t s_i))^2}{s_i} \bigg),
\end{multline} 

\vspace{-15pt}
and under the prior \eqref{eq:model_b0}, the Bayes risk is  
\begin{multline}
\label{eq:gf_risk_bayes}
\Risk(\hbeta^\gf(t)) = \\
\frac{\sigma^2}{n} \sum_{i=1}^p \bigg(\alpha \exp(-2 t s_i) +  
\frac{(1 - \exp(-t s_i))^2}{s_i} \bigg),
\end{multline}
where $\alpha = r^2 n / (\sigma^2 p)$. Here and henceforth, we take by
convention $(1-e^{-x})^2/x=0$ when $x=0$.    
\end{lemma}

\begin{remark}
Compare \eqref{eq:gf_risk} to the risk of ridge regression,
\begin{multline}
\label{eq:ridge_risk}
\Risk(\hbeta^\ridge(\lambda); \beta_0) = \\
\sum_{i=1}^p \bigg( |v_i^T \beta_0|^2 \frac{\lambda^2}{(s_i + \lambda)^2} +  
\frac{\sigma^2}{n} \frac{s_i}{(s_i + \lambda)^2} \bigg).
\end{multline}
and compare \eqref{eq:gf_risk_bayes} to the Bayes risk of ridge,
\begin{equation}
\label{eq:ridge_risk_bayes}
\Risk(\hbeta^\ridge(\lambda)) = 
\frac{\sigma^2}{n} \sum_{i=1}^p 
\frac{\alpha \lambda^2 + s_i}{(s_i + \lambda)^2},
\end{equation}
where $\alpha = r^2 n / (\sigma^2 p)$.  These ridge results follow from
standard calculations, found in many other papers; for completeness, we give
details in the supplement.  
\end{remark}

\begin{remark}
\label{rem:lambda_opt}
For ridge regression, the Bayes risk \eqref{eq:ridge_risk_bayes} is minimized at
$\lambda^*=1/\alpha$.  There are (at least) two easy proofs of this fact.  For
the first, we note the Bayes risk of ridge does not depend on the distributions
of $y|\beta_0$ and $\beta_0$ in \eqref{eq:model_y} and \eqref{eq:model_b0}
(just on the first two moments); in the special case that both distributions
are normal, we know that \smash{$\hbeta^\ridge(\lambda^*)$} is the Bayes
estimator, which achieves the optimal Bayes risk (hence certainly the lowest
Bayes risk over the whole ridge family).  For the second proof, following
\citet{dicker2016ridge}, we rewrite each summand in \eqref{eq:ridge_risk_bayes}
as 
$$
\frac{\alpha \lambda^2 + s_i}{(s_i + \lambda)^2} = 
\frac{\alpha}{s_i + \alpha} + 
\frac{s(\lambda \alpha - 1)^2}{(s_i + \lambda)^2 (s_i + \alpha)},
$$
and observe that this is clearly minimized at $\lambda^*=1/\alpha$. 
\end{remark}

\begin{remark}
As far as we can tell, deriving the tuning parameter value $t^*$ minimizing the
gradient flow Bayes risk \eqref{eq:gf_risk_bayes} is difficult.  Nevertheless,
as we will show in Section \ref{sec:risk_opt_relative}, we can still obtain
interesting bounds on the optimal risk itself,
\smash{$\Risk(\hbeta^\gf(t^*))$}. 
\end{remark}

\subsection{Prediction Risk}

We now define two predictive notions of risk.  Let 
\begin{equation}
\label{eq:model_x0}
x_0 \sim (0, \Sigma)
\end{equation}
for a positive semidefinite matrix $\Sigma \in \R^{p \times p}$, and assume  
$x_0$ is independent of $y|\beta_0$. We define in-sample prediction risk 
and out-of-sample prediction risk (or simply, prediction risk) as, respectively, 
\begin{align*}
\Risk^\ini(\hbeta; \beta_0) &= 
\frac{1}{n} \E \big[ \|X\hbeta - X\beta_0\|_2^2 \,\big|\, \beta_0 \big], \\ 
\Risk^\out(\hbeta; \beta_0) &= 
\E \big[ ( x_0^T \hbeta - x_0^T \beta_0)^2 \,\big|\, \beta_0 \big], 
\end{align*}
and their Bayes versions as, respectively,
\smash{$\Risk^\ini(\hbeta) =$}
\smash{$(1/n) \E \|X\hbeta - X\beta_0\|_2^2$},
\smash{$\Risk^\out(\hbeta) =
\E[(x_0^T \hbeta - x_0^T \beta_0)^2]$}. 

For space reasons, in the remainder, we will focus on out-of-sample 
prediction risk, and defer detailed discussion of in-sample prediction risk
to the supplement.  The next lemma, proved in the supplement, gives
expressions for the prediction risk and Bayes prediction risk of gradient
flow. We denote \smash{$\hSigma = X^T X/n$}.    

\begin{lemma}
\label{lem:gf_risk_out}
Under \eqref{eq:model_y}, \eqref{eq:model_x0}, the prediction risk of the
gradient flow estimator \eqref{eq:gf_sol} is
\begin{multline}
\label{eq:gf_risk_out}
\Risk^\out(\hbeta^\gf(t); \beta_0) = 
\beta_0^T \exp(-t\hSigma) \Sigma \exp(-t\hSigma) \beta_0 +{} \\
\frac{\sigma^2}{n} \tr\big[\hSigma^+ (I-\exp(-t\hSigma))^2 \Sigma\big], 
\end{multline} 
and under \eqref{eq:model_b0}, the Bayes prediction risk is 
\begin{multline}
\label{eq:gf_risk_out_bayes}
\Risk^\out(\hbeta^\gf(t)) =
\frac{\sigma^2}{n} \tr\big[\alpha \exp(-2t\hSigma) \Sigma + {} \\
\hSigma^+ (I-\exp(-t\hSigma))^2 \Sigma \big].
\end{multline}
\end{lemma}

\begin{remark}
Compare \eqref{eq:gf_risk_out} and \eqref{eq:gf_risk_out_bayes} to the
prediction risk and Bayes prediction risk of ridge, respectively,
\begin{multline}
\label{eq:ridge_risk_out}
\Risk^\out(\hbeta^\ridge(\lambda); \beta_0) = \\
\lambda^2 \beta_0^T (\hSigma+\lambda I)^{-1} \Sigma 
(\hSigma+\lambda I)^{-1} \beta_0 +{} \\ 
\frac{\sigma^2}{n} \tr\big[\hSigma (\hSigma+\lambda I)^{-2} \Sigma\big], 
\end{multline} 

\vspace{-30pt}
\begin{multline}
\label{eq:ridge_risk_out_bayes}
\Risk^\out(\hbeta^\ridge(\lambda)) = 
\frac{\sigma^2}{n} \tr\big[ \lambda^2 \alpha 
(\hSigma+\lambda I)^{-2} \Sigma + {} \\ 
\hSigma (\hSigma+\lambda I)^{-2} \Sigma \big]. 
\end{multline}
These ridge results are standard, and details are given in the supplement.   
\end{remark}

\begin{remark}
\label{rem:lambda_opt_out}
The Bayes prediction risk of ridge \eqref{eq:ridge_risk_out_bayes} is again
minimized at $\lambda^*=1/\alpha$. This is not at all clear analytically, but it
can be established by specializing to a normal-normal likelihood-prior pair,
where (for fixed $x_0$) we know that \smash{$x_0^T \hbeta^\ridge(\lambda^*)$} is 
the Bayes estimator for the parameter $x_0^T \beta_0$ (similar to the arguments
in Remark \ref{rem:lambda_opt} for the Bayes estimation risk).
\end{remark}

\section{RELATIVE RISK BOUNDS}
\label{sec:risk_relative}

\subsection{Relative Estimation Risk}
\label{sec:risk_est_relative}

We start with a simple but key lemma.  

\begin{lemma}
\label{lem:simple}
For all $x \geq 0$, we have (a) $e^{-x} \leq 1/(1+x)$ and (b) $1-e^{-x}   
\leq 1.2985 \, x/(1+x)$.  
\end{lemma}

\vspace{-10pt}
\begin{proof}
Fact (a) can by shown via Taylor series and (b) by numerically    
maximizing $x \mapsto (1-e^{-x})(1+x)/x$.
\end{proof}

A bound on the relative risk of gradient flow to ridge, under the
calibration $\lambda=1/t$, follows immediately.

\begin{theorem}
\label{thm:risk_est_relative}
Consider the data model \eqref{eq:model_y}. 

\begin{enumerate}[label=(\alph*), topsep=0pt, itemsep=2pt]
\item For all $\beta_0 \in \R^p$, and all $t \geq 0$,
  \smash{$\Risk(\hbeta^\gf(t); \beta_0) \leq$} 
  \smash{$1.6862 \cdot \Risk(\hbeta^\ridge(1/t); \beta_0)$}.  

\item The inequality in part (a) holds for the Bayes risk with respect to
  any prior on $\beta_0$.

\item The results in parts (a), (b) also hold for in-sample prediction risk. 
\end{enumerate}
\end{theorem}

\vspace{-10pt}
\begin{proof}
For part (a), set $\lambda=1/t$ and compare the $i$th summand in
\eqref{eq:gf_risk}, call it $a_i$, to that in \eqref{eq:ridge_risk}, call it
$b_i$.  Then  
\begin{align*}
a_i &= |v_i^T \beta_0|^2 \exp(-2 t s_i) + \frac{\sigma^2}{n} 
\frac{(1 - \exp(-t s_i))^2}{s_i} \\
&\leq |v_i^T \beta_0|^2 \frac{1}{(1 + t s_i)^2} +
\frac{\sigma^2}{n} 1.2985^2 \frac{t^2 s_i}{(1 + t s_i)^2} \\
&\leq 1.6862 \bigg(|v_i^T \beta_0|^2 \frac{(1/t)^2}{(1/t + s_i)^2} + 
\frac{\sigma^2}{n} \frac{s_i}{(1/t + s_i)^2} \bigg) \\
&= 1.6862 \, b_i,
\end{align*} 
where in the second line, we used Lemma \ref{lem:simple}. Summing over
$i=1,\ldots,p$ gives the desired result. 

Part (b) follows by taking an expectation on each side of the inequality 
in part (a).  Part (c) follows similarly, with details given in the supplement.  
\end{proof}

\begin{remark}
For any $t>0$, gradient flow is in fact a unique Bayes estimator, corresponding
to a normal likelihood in \eqref{eq:model_y} and normal prior $\beta_0 \sim N(0,  
(\sigma^2/n) Q_t^{-1})$, where $Q_t$ is as in Lemma \ref{lem:gf_opt}.  It is
therefore admissible. This means the result in part (a) in the theorem (and part
(b), for the same reason) cannot be true for any universal constant strictly
less than 1.  
\end{remark}


\subsection{Relative Prediction Risk}
\label{sec:risk_out_relative}

We extend the two simple inequalities in Lemma \ref{lem:simple} to matrix
exponentials.  We use $\preceq$ to denote the Loewner ordering on positive
semidefinite matrices, \ie, we use $A \preceq B$ to mean that $B-A$ is
positive semidefinite. 

\begin{lemma}
\label{lem:simple_matrix}
For all $X \succeq 0$, we have (a) $\exp(-2X) \preceq (I+X)^{-2}$ and
(b) $X^+ (I - \exp(-X))^2 \preceq 1.6862 \, X (I + X)^{-2}$.
\end{lemma}

\vspace{-10pt}
\begin{proof}
All matrices in question are simultaneously diagonalizable, so the claims
reduce to ones about eigenvalues, \ie, reduce to checking that $e^{-2x} \leq   
1/(1+x)^2$ and $(1-e^{-x})^2/x \leq 1.6862 \, x/(1+x)^2$, for $x \geq 0$, and  
these follow by manipulating the facts in Lemma \ref{lem:simple}.
\end{proof}

With just a bit more work, we can bound the relative Bayes prediction risk of
gradient flow to ridge, again under the calibration $\lambda=1/t$.

\begin{theorem}
\label{thm:risk_out_relative}
Consider the data model \eqref{eq:model_y}, prior \eqref{eq:model_b0}, and
(out-of-sample) feature distribution \eqref{eq:model_x0}.  For all $t \geq
0$, \smash{$\Risk^\out(\hbeta^\gf(t)) \leq 1.6862 \cdot
  \Risk^\out(\hbeta^\ridge(1/t))$}.  
\end{theorem}

\vspace{-10pt}
\begin{proof}
Consider the matrices inside the traces in \eqref{eq:gf_risk_out_bayes} and
\eqref{eq:ridge_risk_out_bayes}.  Applying Lemma \ref{lem:simple_matrix}, we
have 
\begin{align*}
&\alpha \exp(-2t \hSigma) + \hSigma^+ (I-\exp(-t\hSigma))^2 \\
&\qquad \preceq \alpha (I + t \hSigma)^{-2} + 1.6862 \, t^2 \hSigma (I + t
  \hSigma)^{-2} \\  
&\qquad \preceq 1.6862 \Big( \alpha (1/t)^2 (I/t + \hSigma)^{-2} + 
\hSigma (I/t + \hSigma)^{-2} \Big).
\end{align*}
Let $A,B$ be the matrices on the first and last lines in the above display,
respectively.  As $A \preceq B$ and $\Sigma \succeq 0$, we have $\tr(A\Sigma)
\leq \tr(B\Sigma)$, completing the proof.
\end{proof}

\begin{remark}
The Bayes perspective here is critical; the proof breaks down for prediction
risk, at an arbitrary fixed $\beta_0$, and it is not clear to us whether the
result is true for prediction risk in general.  
\end{remark}

\subsection{Relative Risks at Optima}
\label{sec:risk_opt_relative}

We present one more helpful inequality, and defer its proof to the supplement
(it is more technical than the proofs of Lemmas \ref{lem:simple} and
\ref{lem:simple_matrix}, but still straightforward).   

\begin{lemma}
\label{lem:simple_opt}
For all $X \succeq 0$, it holds that $\exp(-2X) + X^+ (I-\exp(-X))^2 \preceq  
1.2147 \, (I+X)^{-1}$.   
\end{lemma}

We now have the following result, on the relative Bayes risk (and Bayes
prediction risk), of gradient descent to ridge regression, when both are
optimally tuned.

\begin{theorem}
\label{thm:risk_opt_relative}
Consider the data model \eqref{eq:model_y}, prior \eqref{eq:model_b0}, and 
(out-of-sample) feature distribution \eqref{eq:model_x0}.

\begin{enumerate}[label=(\alph*), topsep=0pt, itemsep=2pt]
\item It holds that 
  \vspace{-6pt}
  $$
  1 \leq \frac{\inf_{t \geq 0} \Risk(\hbeta^\gf(t))}
  {\inf_{\lambda \geq0} \, \Risk(\hbeta^\ridge(\lambda))}
  \leq  1.2147.
  \vspace{-6pt}
  $$

\item The same result as in part (a) holds for both in-sample and out-of-sample
  prediction risk.  
\end{enumerate}
\end{theorem}

\vspace{-10pt}
\begin{proof}
For part (a), recall from Remark \ref{rem:lambda_opt} that the optimal ridge
tuning parameter is $\lambda^*=1/\alpha$ and further, in the special case of a 
normal-normal likelihood-prior pair, we know that
\smash{$\hbeta^\ridge(\lambda^*)$} is the Bayes estimator so the Bayes risk   
of \smash{$\hbeta^\gf(t)$}, for any $t \geq 0$, must be at least that of
\smash{$\hbeta^\ridge(\lambda^*)$}.  But because these Bayes risks
\eqref{eq:gf_risk_bayes}, \eqref{eq:ridge_risk_bayes} do not depend on the form
of likelihood and prior (only on their first two moments), we know that the same
must be true in general, which proves the lower bound on the risk
ratio.  For the upper bound, we take $t=\alpha$, and compare the $i$th summand  
in \eqref{eq:gf_risk_bayes}, call it $a_i$, to that in
\eqref{eq:ridge_risk_bayes}, call it $b_i$.  We have  
\begin{align*}
a_i &= \alpha \exp(-2 \alpha s_i) + \frac{(1 - \exp(-\alpha s_i))^2}{s_i} \\ 
&\leq 1.2147 \frac{\alpha}{1 + \alpha s_i} = 1.2147 \, b_i,
\end{align*}
where in the second line, we applied Lemma \ref{lem:simple_opt} (to the case of  
scalar $X$).  Summing over $i=1,\ldots,p$ gives the desired result.

Parts (b) follows similarly, with details in the supplement.  
\end{proof}

\section{ASYMPTOTIC RISK ANALYSIS}
\label{sec:risk_asymp}

\subsection{Marchenko-Pastur Asymptotics}

Notice the Bayes risk for gradient flow \eqref{eq:gf_risk_bayes} and ridge
regression \eqref{eq:ridge_risk_bayes} depend only on the predictor matrix $X$
via the eigenvalues of the (uncentered) sample covariance \smash{$\hSigma=X^T 
  X/n$}. Random matrix theory gives us a precise understanding of the behavior 
of these eigenvalues, in large samples.  The following assumptions are
standard ones in random matrix theory (\eg, \citealt{bai2010spectral}). Given a
symmetric matrix $A \in \R^{p \times p}$, recall that its {\it spectral
  distribution} is defined as \smash{$F_A(x) = (1/p) \sum_{i=1}^p
  \ones(\lambda_i(A) \leq x)$}, where $\lambda_i(A)$, $i=1,\ldots,p$ are the
eigenvalues of $A$, and $\ones(\cdot)$ denotes the 0-1 indicator function.   

\begin{assumption}{A}{1}
\label{as:x_dist}
The predictor matrix satisfies $X = Z \Sigma^{1/2}$, for a random matrix $Z
\in \R^{n \times p}$ of i.i.d.\ entries with zero mean and unit variance, and a
deterministic positive semidefinite covariance $\Sigma \in \R^{p \times p}$.   
\end{assumption}

\begin{assumption}{A}{2}
\label{as:aspect_ratio}
The sample size $n$ and dimension $p$ both diverge, \ie, $n,p \to  
\infty$, with $p/n \to \gamma \in (0,\infty)$.  
\end{assumption}

\begin{assumption}{A}{3}
\label{as:spec_measure}
The spectral measure $F_\Sigma$ of the predictor covariance $\Sigma$ converges
weakly as $n,p \to \infty$ to some limiting spectral measure $H$. 
\end{assumption}

Under the above assumptions, the seminal Marchenko-Pastur theorem 
describes the weak limit of the spectral
measure \smash{$F_{\hSigma}$} of the sample covariance \smash{$\hSigma$}.  

\begin{theorem}[\citealt{marchenko1967distribution,silverstein1995strong,bai2010spectral}]  
\label{thm:mp} 
Assuming \ref{as:x_dist}--\ref{as:spec_measure}, 
almost surely, the spectral measure \smash{$F_{\hSigma}$} of
\smash{$\hSigma$} converges weakly to a law \smash{$F_{H,\gamma}$},
called the {\em empirical spectral distribution}, that depends only on 
$H,\gamma$.      
\end{theorem}

\begin{remark}
In general, a closed form for the empirical spectral distribution
\smash{$F_{H,\gamma}$} is not known, except in very special cases (\eg, when 
$\Sigma=I$ for all $n,p$).  However, numerical methods for approximating
\smash{$F_{H,\gamma}$} have been proposed (see 
\citealt{dobriban2015efficient} and references therein).
\end{remark}

\subsection{Limiting Gradient Flow Risk}

The limiting Bayes risk of gradient flow is now immediate from the 
representation in \eqref{eq:gf_risk_bayes}.

\begin{theorem}
\label{thm:gf_risk_lim}
Assume \ref{as:x_dist}--\ref{as:spec_measure}, as well as a data model
\eqref{eq:model_y} and prior \eqref{eq:model_b0}.  Then as $n,p \to \infty$ with
$p/n \to \gamma \in (0,\infty)$, for each $t \geq 0$, the Bayes risk
\eqref{eq:gf_risk_bayes} of gradient flow converges almost surely to    
\begin{equation} 
\label{eq:gf_risk_lim}
\sigma^2 \gamma \int \bigg[ \alpha_0 \exp(-2ts) +
\frac{(1-\exp(-ts))^2}{s} \bigg] \, dF_{H,\gamma}(s),
\end{equation}
where $\alpha_0 = r^2/(\sigma^2 \gamma)$, and $F_{H,\gamma}$ is the empirical
spectral distribution from Theorem \ref{thm:mp}.
\end{theorem}

\vspace{-10pt}
\begin{proof}
Note that we can rewrite the Bayes risk in \eqref{eq:gf_risk_bayes} as  
\smash{$(\sigma^2 p)/n [\int \alpha h_1(s) \, dF_{\hSigma}(s) + \int h_2(s) \,  
  dF_{\hSigma}(s)]$}, where we let $h_1(s) = \exp(-2ts)$, $h_2(s) =
(1-\exp(-ts))^2/s$.  Weak convergence of \smash{$F_{\hSigma}$} to
$F_{H,\gamma}$, from Theorem \ref{thm:mp}, implies
\smash{$\int h(s) \, dF_{\hSigma}(s) \to \int h(s) \, dF_{H,\gamma}(s)$} for all
bounded, continuous functions $h$, which proves the result.   
\end{proof}

A similar result is available for the limiting Bayes in-sample prediction risk,
given in the supplement.  Studying the the limiting Bayes (out-of-sample)
prediction risk is much more challenging, as \eqref{eq:gf_risk_out_bayes} is not
simply a function of eigenvalues of \smash{$\hSigma$}. The proof of the next
result, deferred to the supplement, relies on a key fact on the Laplace
transform of the map $x \mapsto \exp(x A)$, and the asymptotic limit of a 
certain trace functional involving  \smash{$\hSigma,\Sigma$}, from
\citet{ledoit2011eigenvectors}.  

\begin{theorem}
\label{thm:gf_risk_out_lim}
Under the conditions of Theorem \ref{thm:gf_risk_lim}, also assume 
\smash{$\E(Z_{ij}^{12}) \leq C_1$}, $\|\Sigma\|_2 \leq C_2$, for all 
$n,p$ and constants $C_1,C_2>0$. For each $t \geq 0$, the Bayes  
prediction risk \eqref{eq:gf_risk_out_bayes} of gradient flow converges 
almost surely to       
\begin{equation}
\label{eq:gf_risk_out_lim}
\sigma^2 \gamma \bigg[ \alpha_0 f(2t) + 2 \int_0^t (f(u) - f(2u)) \, du\bigg],  
\end{equation}
where $f$ is the inverse Laplace transform of the function 
$$
x \mapsto \frac{1}{\gamma}
\bigg(\frac{1}{1-\gamma + \gamma x m(F_{H,\gamma})(-x)}  - 1\bigg),  
$$
and \smash{$m(F_{H,\gamma})$} is the Stieltjes transform of
\smash{$F_{H,\gamma}$} (defined precisely in the supplement).   
\end{theorem}

An interesting feature of the results \eqref{eq:gf_risk_lim},
\eqref{eq:gf_risk_out_lim} is that they are asymptotically {\it exact} (no 
hidden constants). Analogous results for ridge (by direct arguments, and 
\citealt{dobriban2018high}, respectively) are compared in the supplement, 
for space reasons.  

\section{NUMERICAL EXAMPLES}
\label{sec:exps}

We give numerical evidence for our theoretical results: both our relative risk
bounds in Section \ref{sec:risk_relative}, and our asymptotic risk expressions
in Section \ref{sec:risk_asymp}. We generated features via $X = \Sigma^{1/2} Z$, 
for a matrix $Z$ with i.i.d.\ entries from a distribution $G$ (with mean zero
and unit variance), for three choices of $G$: standard Gaussian, Student $t$
with 3 degrees of freedom, and Bernoulli with probability 0.5 (the last two
distributions were standardized).  We took $\Sigma$ to have all diagonal entries
equal to 1 and all off-diagonals equal to $\rho=0$ (\ie, $\Sigma=I$), or
$\rho=0.5$.  For the problem dimensions, we considered $n=1000$, $p=500$ and
$n=500$, $p=1000$.  For both gradient flow and ridge, we used a range
of 200 tuning parameters equally spaced on the log scale from $2^{-10}$ to
$2^{10}$. Lastly, we set $\sigma^2=r^2=1$, where $\sigma^2$ is the noise
variance in \eqref{eq:model_y} and $r^2$ is the prior radius in
\eqref{eq:model_b0}. For each configuration of $G,\Sigma,n,p$, we computed the 
Bayes risk and Bayes prediction risk gradient flow and ridge, as in
\eqref{eq:gf_risk_bayes}, \eqref{eq:ridge_risk_bayes}, 
\eqref{eq:gf_risk_out_bayes},  \eqref{eq:ridge_risk_out_bayes}.
For $\Sigma=I$, the empirical spectral distribution from Theorem \ref{thm:mp}
has a closed form, and so we computed the limiting Bayes risk for gradient flow
\eqref{eq:gf_risk_lim} via numerical integration (and similarly for ridge,
details in the supplement). 

\begin{figure}[t]
\centering
\includegraphics[width=0.415\textwidth]{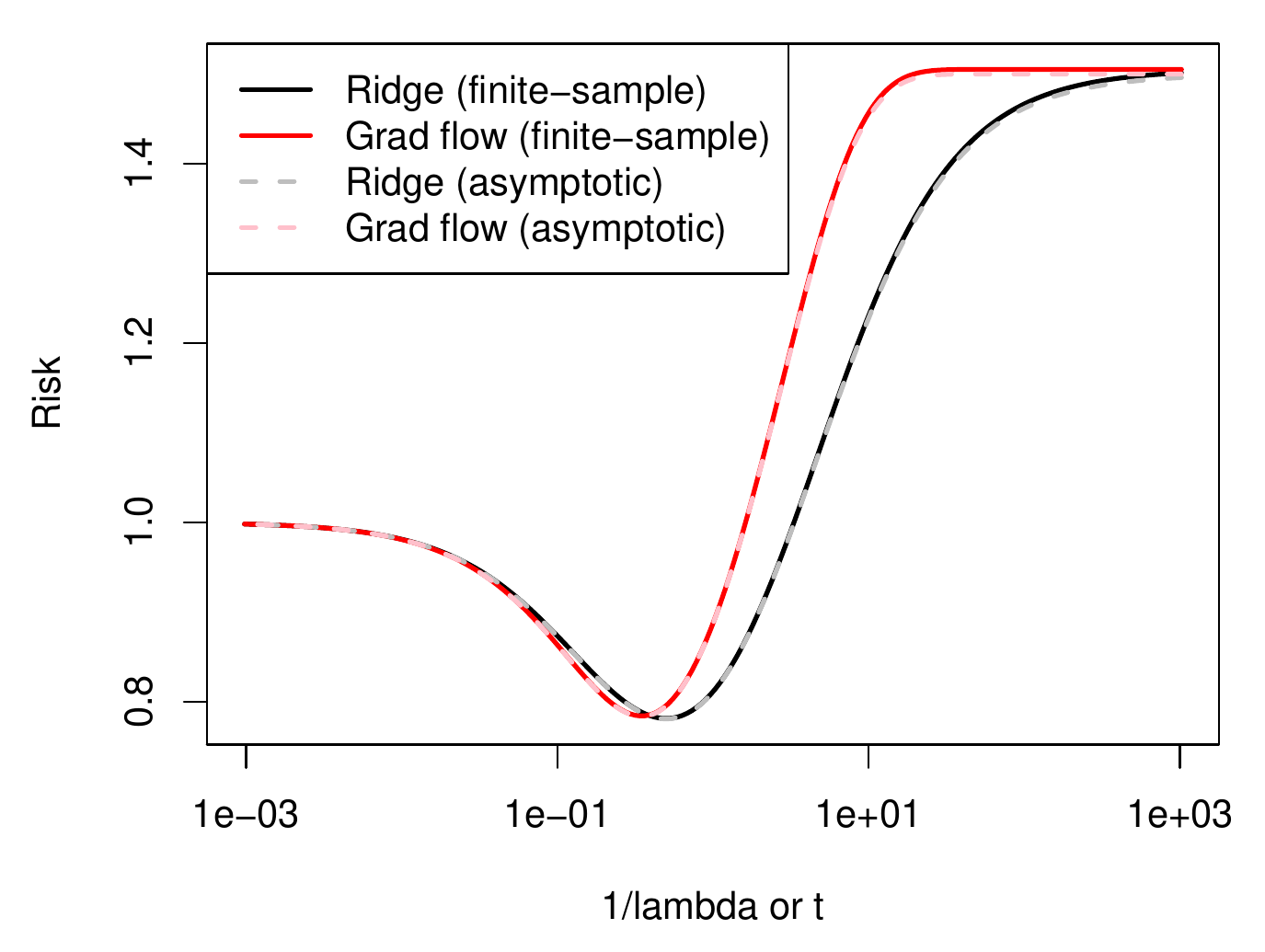} \\
\includegraphics[width=0.415\textwidth]{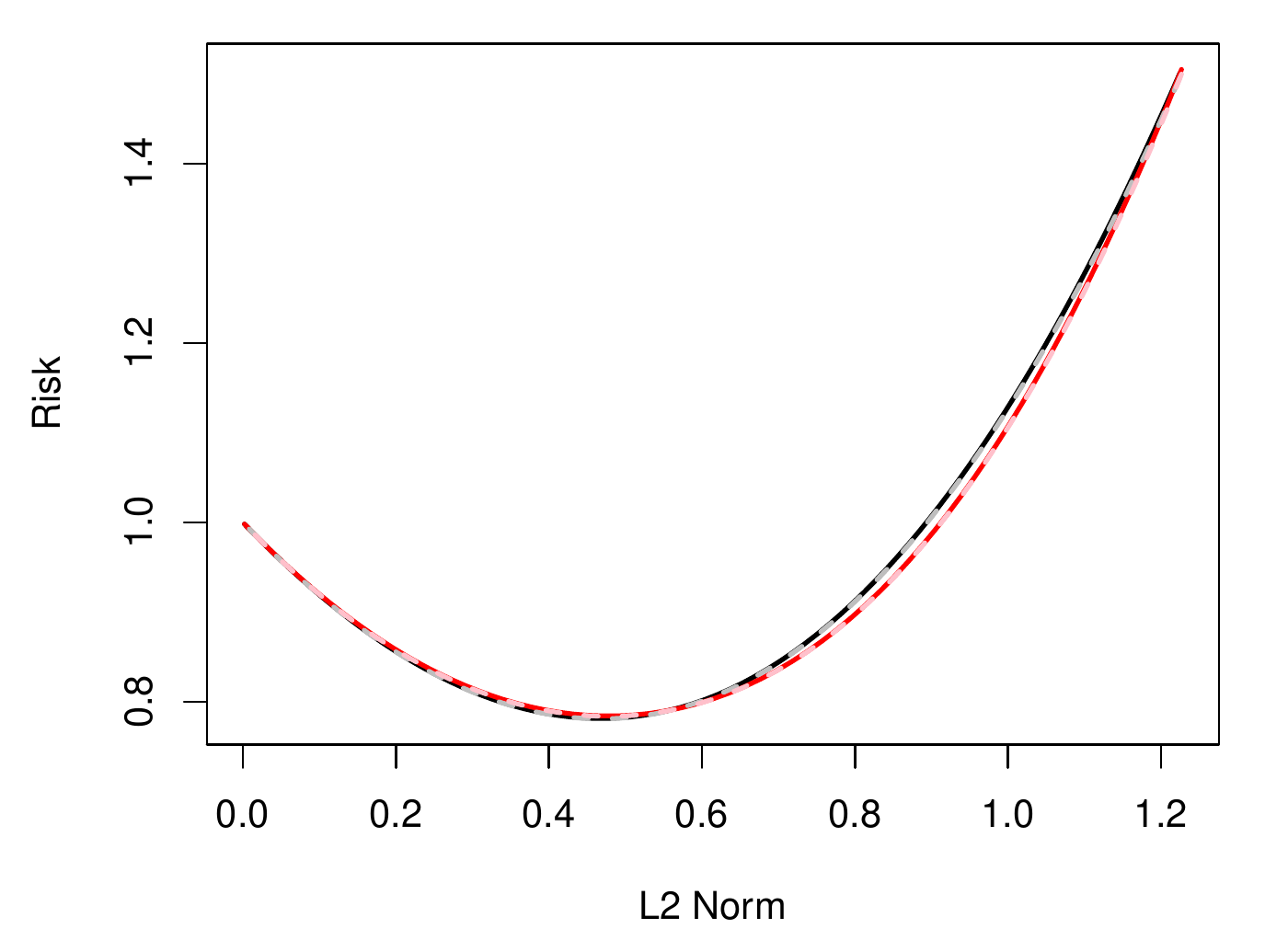}
\vspace{-5pt}
\caption{\it \small Comparison of Bayes risks for gradient flow and ridge, with 
  Gaussian features, $\Sigma=I$, $n=500$, $p=1000$.}
\label{fig:risk} 
\vspace{-10pt}
\end{figure}

Figure \ref{fig:risk} shows the results for Gaussian features, $\Sigma=I$,
$n=500$, and $p=1000$; the supplement shows results for all other 
cases (the results are grossly similar).   The top plot shows the risk
curves when calibrated according to $\lambda=1/t$ (as per our theory).
Here we see fairly strong agreement between the two risk curves, especially
around their minimums; the maximum ratio of gradient flow to ridge risks is
1.2164 over the entire path (cf.\ the upper bound of 1.6862 from Theorem 
\ref{thm:risk_est_relative}), and the ratio of the minimums is 1.0036 (cf.\
the upper bound of 1.2147 from Theorem \ref{thm:risk_opt_relative}).  The bottom 
plot shows the risks when parametrized by the $\ell_2$ norms of the
underlying estimators.  We see remarkable agreement over the whole path, with 
a maximum ratio of 1.0050. Moreover, in both plots, we can see that the
finite-sample (dotted lines) and asymptotic risk curves (solid lines) are
identical, meaning that the convergence in Theorem \ref{thm:gf_risk_lim} is very
rapid (and similarly for ridge).

\section{DISCUSSION}
\label{sec:disc}

In this work, we studied gradient flow (\ie, gradient descent with
infinitesimal step sizes) for least squares, and pointed out a number of
connections to ridge regression.  We showed that, under minimal assumptions on
the data model, and using a calibration $t=1/\lambda$---where $t$ denotes the
time parameter in gradient flow, and $\lambda$ the tuning parameter in
ridge---the risk of gradient flow is no more than 1.69 times that of ridge, for
all $t \geq 0$. We also showed that the same holds for prediction risk, in an 
average (Bayes) sense, with respect to any spherical prior.  Though we did not
pursue this, it is clear that these risk couplings could be used to port risk
results from the literature on ridge regression (\eg,
\citealt{hsu2012random,raskutti2014early,dicker2016ridge,dobriban2018high},
etc.) to gradient flow.  


Our numerical experiments revealed that calibrating the risk curves by the
underlying $\ell_2$ norms of the estimators results in a much tighter coupling;
developing theory to explain this phenomenon is an 
important challenge left to future work. Other interesting directions are
to analyze the risk of a continuum version of stochastic gradient descent, or  
to study gradient flow beyond least squares, \eg, for logistic regression.

\smallskip
\noindent
{\bf Acknolwedgements.} 
We thank Veeranjaneyulu Sadhanala, whose insights led us to 
completely revamp the main results in our paper.
AA was supported by DoE CSGF no.\ DE-FG02-97ER25308.
ZK was supported by DARPA YFA no.\ N66001-17-1-4036.

\bibliographystyle{plainnat}
\bibliography{refs}

\appendix
\renewcommand\thesection{S.\arabic{section}}  
\renewcommand\theequation{S.\arabic{equation}}
\renewcommand\thelemma{S.\arabic{lemma}}
\renewcommand\thetable{S.\arabic{table}}
\renewcommand\thefigure{S.\arabic{figure}}
\renewcommand\thealgorithm{S.\arabic{algorithm}}
\setcounter{figure}{0}
\setcounter{section}{0}
\setcounter{equation}{0}
\setcounter{lemma}{0}
\setcounter{algorithm}{0}
\onecolumn

\section*{Supplementary Material}

\section{Proof of Lemma \ref{lem:gd_opt}}

Let $X^T X/n = V S V^T$ be an eigendecomposition of $X^T X/n$.  Then we can
rewrite the gradient descent iteration \eqref{eq:gd} as
$$
\beta^{(k)} = \beta^{(k-1)} + \frac{\epsilon}{n} \cdot X^T(y - X \beta^{(k-1)}) 
= (I - \epsilon V S V^T) \beta^{(k-1)} + \frac{\epsilon}{n} \cdot X^T y.
$$
Rotating by $V^T$, we get
$$
\tilde \beta^{(k)} = (I - \epsilon S) \tilde \beta^{(k-1)} + \tilde y,
$$
where we let \smash{$\tilde \beta^{(j)} = V^T \beta^{(j)}$}, $j=1,2,3,\ldots$
and \smash{$\tilde y = (\epsilon/n) V^T X^T y$}. Unraveling the preceding
display, we find that 
$$
\tilde \beta^{(k)} = (I - \epsilon S)^k \tilde \beta^{(0)} + 
\sum_{j=0}^{k-1} (I - \epsilon S)^j \tilde y. 
$$
Furthermore applying the assumption that the initial point $\beta^{(0)} = 0$
yields 
$$
\tilde \beta^{(k)} = \sum_{j=0}^{k-1} (I - \epsilon S)^j \tilde y 
= (\epsilon S)^{-1} (I - (I - \epsilon S)^k) \tilde y,
$$
with the second equality following after a short inductive argument.  

Now notice that \smash{$\beta^{(k)} = V \tilde \beta^{(k)}$}, since  
$V V^T$ is the projection onto the row space of $X$, and $\beta^{(k)}$ lies in 
the row space. Rotating back to the original space then gives
$$
\beta^{(k)} = V (\epsilon S)^{-1} (I - (I - \epsilon S)^k) \tilde y
= \frac{1}{n} V S^{-1} (I - (I - \epsilon S)^k) V^T X^T y.
$$
Compare this to the solution of the optimization problem in
Lemma \ref{lem:gd_opt}, which is  
$$
(X^T X + nQ_k)^{-1} X^T y = \frac{1}{n}(VSV^T + Q_k)^{-1} X^T y. 
$$
Equating the last two displays, we see that we must have
$$
V S^{-1} (I - (I - \epsilon S)^k) V^T = (VSV^T + Q_k)^{-1}.
$$
Inverting both sides and rearranging, we get
$$
Q_k = V S (I - (I - \epsilon S)^k)^{-1} V^T - VSV^T,
$$
and an application of the matrix inversion lemma shows that 
$(I - (I - \epsilon S)^k)^{-1} = I + ((I - \epsilon S)^{-k} - I)^{-1}$, so 
$$
Q_k = V S ((I - \epsilon S)^{-k} - I)^{-1} V^T,
$$
as claimed in the lemma.

\section{Proof of Lemma \ref{lem:gf_opt}}

Recall that Lemma \ref{lem:gf} gives the gradient flow solution at time $t$, in
\eqref{eq:gf_sol}.  Compare this to the solution of the optimization problem in
Lemma \ref{lem:gf_opt}, which is  
$$
(X^T X + nQ_t)^{-1} X^T y.
$$
To equate these two, we see that we must have
$$
(X^T X)^+ (I - \exp(-t X^T X/n)) = (X^T X + nQ_t)^{-1},
$$
\ie, writing $X^T X/n=V S V^T$ as an eigendecomposition of $X^T X/n$, 
$$
V S^+ (I - \exp(-tS)) V^T = (V S V^T + Q_t)^{-1}.
$$
Inverting both sides and rearranging, we find that
$$
Q_t = V S (I - \exp(-tS))^{-1} V^T - V S V^T,
$$
which is as claimed in the lemma.

\section{Proof of Lemma \ref{lem:gf_risk}}

For fixed $\beta_0$, and any estimator \smash{$\hbeta$}, recall the
bias-variance decomposition   
$$
\Risk(\hbeta; \beta_0) = \|\E(\hbeta) - \beta_0\|_2^2 + \tr[\Cov(\hbeta)]. 
$$
For the gradient flow estimator in \eqref{eq:gf_sol}, we have  
\begin{align}
\nonumber
\E[\hbeta^\gf(t)] 
&= (X^T X)^+ (I - \exp(-t X^T X/n)) X^T X \beta_0 \\  
\nonumber
&= (X^T X)^+ X^T X (I - \exp(-t X^T X/n)) \beta_0 \\ 
\label{eq:gf_sol_exp}
&= (I - \exp(-t X^T X/n)) \beta_0.
\end{align}
In the second line, we used the fact that $X^T X$ and $(I - \exp(-t X^T 
X/n))$ are simultaneously diagonalizable, and so they commute; in the third
line, we used the fact that $(X^T X)^+ X^T X = X^+ X$ is the projection onto the
row space of $X$, and the image of $I - \exp(-t X^T X/n)$ is already in the row
space.  Hence the bias is, abbreviating \smash{$\hSigma = X^TX/n$}, 
\begin{equation}
\label{eq:gf_risk_bias}
\big\|\E [\hbeta^\gf(t)] - \beta_0\big\|_2^2 =
 \|\exp(-t \hSigma) \beta_0\|_2^2 
= \sum_{i=1}^p |v_i^T \beta_0|^2 \exp(-2t s_i). 
\end{equation}
As for the variance, we have
\begin{align}
\nonumber
\tr \big( \Cov[\hbeta^\gf(t)] \big) &= \sigma^2 \tr \big[ (X^T X)^+ 
(I - \exp(-t \hSigma)) (X^T X) (I - \exp(-t \hSigma)) (X^T X)^+ \big] \\ 
\nonumber
&= \frac{\sigma^2}{n} \tr \big[ \hSigma^+ (I - \exp(-t \hSigma))^2 \big] \\
\label{eq:gf_risk_var}
&= \frac{\sigma^2}{n} \sum_{i=1}^p \frac{(1 - \exp(-t s_i))^2}{s_i},
\end{align}
where in the second line we used the fact that \smash{$\hSigma^+$} and  
\smash{$(I - \exp(-t \hSigma))$} are simultaneously diagonalizable, and
hence commute, and also the fact that \smash{$\hSigma^+ \hSigma 
\hSigma^+ = \hSigma^+$}.  Putting together \eqref{eq:gf_risk_bias} and
\eqref{eq:gf_risk_var} proves the result in \eqref{eq:gf_risk}.

When $\beta_0$ follows the prior in \eqref{eq:model_b0}, the variance
\eqref{eq:gf_risk_var} remains unchanged. The expectation of the bias 
\eqref{eq:gf_risk_bias} (over $\beta_0$) is
$$
\E \big[ \beta_0^T \exp(-2t \hSigma) \beta_0 \big] =
\tr \big[ \E(\beta_0 \beta_0^T) \exp(-2t \hSigma) \big] = \frac{r^2}{p}
\sum_{i=1}^p \exp(-2t s_i),
$$
which leads to \eqref{eq:gf_risk_bayes}, after the appropriate definition of 
$\alpha$. 

\section{Derivation of \eqref{eq:ridge_risk}, \eqref{eq:ridge_risk_bayes}}

As in the calculations in the last section, consider for the ridge estimator in
\eqref{eq:ridge_sol},
\begin{equation}
\label{eq:ridge_sol_exp}
\E[\hbeta^\ridge(\lambda)] =
(X^T X + n \lambda I)^{-1} X^T X \beta_0 =
(\hSigma + \lambda I)^{-1} \hSigma \beta_0,
\end{equation}
where we have again abbreviated \smash{$\hSigma = X^TX/n$}.  The bias is thus 
\begin{align}
\nonumber
\big\|\E [\hbeta^\ridge(\lambda)] - \beta_0\big\|_2^2 &= 
\big\| (\hSigma + \lambda I)^{-1} (\hSigma - I) \beta_0 \big\|_2^2 \\
\nonumber
&= \big\| \lambda  (\hSigma + \lambda I)^{-1} \beta_0 \big\|_2^2 \\ 
\label{eq:ridge_risk_bias}
&= \sum_{i=1}^p |v_i^T \beta_0|^2 \frac{\lambda^2}{(s_i + \lambda)^2},
\end{align}
the second equality following after adding and subtracting $\lambda I$ to the
second term in parentheses, and expanding.  For the variance, we compute
\begin{align}
\nonumber
\tr \big( \Cov[\beta^\ridge(\lambda)] \big) &= \sigma^2 \tr \big[ 
(X^T X + n \lambda I)^{-1} X^T X (X^T X + n \lambda I)^{-1} \big] \\
\nonumber
&= \frac{\sigma^2}{n} \tr \big[ \hSigma (\hSigma + \lambda I)^{-2} \big] \\ 
\label{eq:ridge_risk_var}
&= \frac{\sigma^2}{n} \sum_{i=1}^p \frac{s_i}{(s_i + \lambda)^2},
\end{align}
the second equality following by noting that \smash{$\hSigma$} and
\smash{$(\hSigma + \lambda I)^{-1}$} are simultaneously diagonalizable, and
therefore commute. Putting together \eqref{eq:ridge_risk_bias} and
\eqref{eq:ridge_risk_var} proves the result in \eqref{eq:ridge_risk}.  The Bayes
result \eqref{eq:ridge_risk_bayes} follows by taking an expectation of the bias 
\eqref{eq:ridge_risk_bias} (over $\beta_0$), just as in the last section for 
gradient flow.  

\section{Proof of Lemma \ref{lem:gf_risk_out}}

First, observe that for fixed $\beta_0$, and any estimator \smash{$\hbeta$}, 
$$
\Risk^\out(\hbeta; \beta_0) = \E \|\hbeta - \beta_0\|_\Sigma^2,
$$
where $\| z \|_A^2 = z^T A z$.  The bias-variance decomposition for
out-of-sample prediction risk is hence 
$$
\Risk^\out(\hbeta; \beta_0) =
\|\E(\hbeta) - \beta_0\|_\Sigma^2 + \tr[\Cov(\hbeta) \Sigma]. 
$$
For gradient flow, we can compute the bias, from \eqref{eq:gf_sol_exp}, 
\begin{equation}
\label{eq:gf_risk_out_bias}
\big\|\E[\hbeta^\gf(t)] - \beta_0\big\|_\Sigma^2 = 
\|\exp(-t \hSigma) \beta_0\|_\Sigma^2 =
\beta_0^T \exp(-t\hSigma) \Sigma \exp(-t\hSigma) \beta_0,
\end{equation}
and likewise the variance, 
\begin{align}
\nonumber
\tr \big( \Cov[\beta^\gf(t)] \big) &= \sigma^2 \tr \big[ (X^T X)^+ 
(I - \exp(-t \hSigma)) (X^T X) (I - \exp(-t \hSigma)) (X^T X)^+ \Sigma \big] \\   
\label{eq:gf_risk_out_var}
&= \frac{\sigma^2}{n} \tr \big[ \hSigma^+ (I - \exp(-t \hSigma))^2 \Sigma \big].  
\end{align}
Putting together \eqref{eq:gf_risk_out_bias} and \eqref{eq:gf_risk_out_var} proves
the result in \eqref{eq:gf_risk_out}.  The Bayes result
\eqref{eq:gf_risk_out_bayes}  follows by taking an 
expectation over the bias, as argued previously. 

We note that the in-sample prediction risk is given by the same formulae except
with $\Sigma$ replaced by \smash{$\hSigma$}, which leads to  
\begin{align}
\nonumber
\Risk^\ini(\hbeta^\gf(t); \beta_0) &= 
\beta_0^T \exp(-t\hSigma) \hSigma \exp(-t\hSigma) \beta_0 +
\frac{\sigma^2}{n} \tr\big[\hSigma^+ (I-\exp(-t\hSigma))^2 \hSigma \big] \\ 
\label{eq:gf_risk_in}
&= \sum_{i=1}^p \bigg(|v_i^T \beta_0|^2 s_i \exp(-2 t s_i) + 
\frac{\sigma^2}{n} (1 - \exp(-t s_i))^2 \bigg),
\end{align} 
and
\begin{align}
\nonumber
\Risk^\out(\hbeta^\gf(t)) &=
\frac{\sigma^2}{n} \tr\big[\alpha \exp(-2t\hSigma) \hSigma + 
\hSigma^+ (I-\exp(-t\hSigma))^2 \hSigma \big] \\
\label{eq:gf_risk_in_bayes}
&= \frac{\sigma^2}{n} \sum_{i=1}^p \big[ \alpha s_i \exp(-2 t s_i) +  
(1 - \exp(-t s_i))^2 \big].
\end{align}

\section{Derivation of \eqref{eq:ridge_risk_out},
  \eqref{eq:ridge_risk_out_bayes}}

For ridge, we can compute the bias, from \eqref{eq:ridge_sol_exp},  
\begin{equation}
\label{eq:ridge_risk_out_bias}
\big\|\E [\hbeta^\ridge(\lambda)] - \beta_0\big\|_\Sigma^2 
= \big\| \lambda  (\hSigma + \lambda I)^{-1} \beta_0 \big\|_\Sigma^2 
= \lambda^2 \beta_0^T (\hSigma+\lambda I)^{-1} \Sigma 
(\hSigma+\lambda I)^{-1} \beta_0,
\end{equation}
and also the variance, 
\begin{align}
\nonumber
\tr \big( \Cov[\beta^\ridge(\lambda)] \Sigma \big) &= \sigma^2 \tr \big[ 
(X^T X + n \lambda I)^{-1} X^T X (X^T X + n \lambda I)^{-1} X^T \Sigma \big] \\ 
\label{eq:ridge_risk_out_var}
&= \frac{\sigma^2}{n} \tr \big[ \hSigma (\hSigma + \lambda I)^{-2} \Sigma
  \big]. 
\end{align}
Putting together \eqref{eq:ridge_risk_out_bias} and \eqref{eq:ridge_risk_out_var} proves
\eqref{eq:ridge_risk_out}, and the Bayes result
\eqref{eq:ridge_risk_out_bayes} follows by taking an 
expectation over the bias, as argued previously. 

Again, we note that the in-sample prediction risk expressions is given by replacing
$\Sigma$ replaced by \smash{$\hSigma$}, yielding 
\begin{align}
\nonumber
\Risk^\ini(\hbeta^\ridge(\lambda); \beta_0) &= 
\lambda^2 \beta_0^T (\hSigma+\lambda I)^{-1} \hSigma 
(\hSigma+\lambda I)^{-1} \beta_0 +
\frac{\sigma^2}{n} \tr\big[\hSigma (\hSigma+\lambda I)^{-2} \hSigma\big] \\
\label{eq:ridge_risk_in}
&= \sum_{i=1}^p \bigg( |v_i^T \beta_0|^2 \frac{\lambda^2 s_i}{(s_i + \lambda)^2}
  + \frac{\sigma^2}{n} \frac{s_i^2}{(s_i + \lambda)^2} \bigg),
\end{align}
and 
\begin{align}
\nonumber
\Risk^\ini(\hbeta^\ridge(\lambda)) &= 
\frac{\sigma^2}{n} \tr\big[ \lambda^2 \alpha 
(\hSigma+\lambda I)^{-2} \hSigma +
\hSigma (\hSigma+\lambda I)^{-2} \hSigma \big] \\
\label{eq:ridge_risk_in_bayes}
&= \frac{\sigma^2}{n} \sum_{i=1}^p 
\frac{\alpha \lambda^2 s_i + s_i^2}{(s_i + \lambda)^2}.
\end{align}

\section{Proof of Theorem \ref{thm:risk_est_relative}, Part (c)}

As we can see from comparing \eqref{eq:gf_risk}, \eqref{eq:ridge_risk} to
\eqref{eq:gf_risk_in}, \eqref{eq:ridge_risk_in}, the only difference in the
latter in-sample prediction risk expressions is that each summand has been
multiplied by $s_i$.  Therefore the exact same relative bounds apply termwise,
\ie, the arguments for part (a) apply here.  The Bayes result again follows just
by taking expectations.  

\section{Proof of Lemma \ref{lem:simple_opt}}

As in the proof of Lemma \ref{lem:simple_matrix}, because all matrices here are
simultaneously diagonalizable, the claim reduces to one about eigenvalues, and
it suffices to check that $e^{-2x} + (1-e^{-x})^2/x \leq 1.2147 / (1+x)$ for all $x
\geq 0$.  Completing the square and simplifying, 
\begin{align*}
e^{-2x} + \frac{(1-e^{-x})^2}{x}
&= \frac{(1 + x) e^{-2x} - 2 e^{-x}+ 1}{x} \\ 
&= \frac{(\sqrt{1 + x} e^{-x} - \frac{1}{\sqrt{1 + x}})^2}{x}
+ \frac{x}{1+x}.
\end{align*}
Now observe that, for any constant $C > 0$,
\begin{align}
\label{eq:desired_ineq}
\frac{(\sqrt{1 + x} e^{-x} - \frac{1}{\sqrt{1 + x}})^2}{x}
+ \frac{x}{1+x} & \leq (1+C^2) \frac{1}{1+x} \\
\nonumber
\iff | (1 + x) e^{-x} - 1| & \leq C \sqrt{x} \\ 
\nonumber
\iff 1 - (1 + x) e^{-x} & \leq C \sqrt{x},
\end{align}
the last line holding because the basic inequality $e^x \geq 1+x$ implies that
$e^{-x} \leq 1/(1+x)$, for $x >-1$.  We see that for the above line to hold, we
may take  
$$
C = \max_{x \geq 0} \, \big[ 1 - (1 + x) e^{-x}\big] / \sqrt{x} = 0.4634, 
$$
which has been computed by numerical maximization, \ie, we find that the desired
inequality \eqref{eq:desired_ineq} holds with $(1+C^2) = 1.2147$. 

\section{Proof of Theorem \ref{thm:risk_opt_relative}, Part (b)}

The lower bounds for the in-sample and out-of-sample prediction risks follow by
the same arguments as in the estimation risk case (the ridge estimator here is
the Bayes estimator in the case of a normal-normal likelihood-prior pair, and
the risks here do not depend on the specific form of the likelihood and prior). 

For the upper bounds, for in-sample prediction risk, we can see from comparing
\eqref{eq:gf_risk_bayes}, \eqref{eq:ridge_risk_bayes} 
to \eqref{eq:gf_risk_in_bayes}, \eqref{eq:ridge_risk_in_bayes}, the only
difference in the latter expressions is that each summand has been multiplied by
$s_i$, and hence the same relative bounds apply termwise, \ie, the arguments for 
part (a) carry over directly here. 

And for out-of-sample prediction risk, the matrix inside the trace in
\eqref{eq:gf_risk_out_bayes} when $t=\alpha$ is
$$
\alpha \exp(-2\alpha\hSigma) + \hSigma^+ (I-\exp(-\alpha\hSigma))^2,
$$ 
and the matrix inside the trace in \eqref{eq:ridge_risk_out_bayes} when
$\lambda=1/\alpha$ is 
$$
1/\alpha (\hSigma+(1/\alpha) I)^{-2} + \hSigma (\hSigma+(1/\alpha) I)^{-2} 
= \alpha (\alpha \hSigma + I)^{-1}. 
$$
By Lemma \ref{lem:simple_opt}, we have
$$
\alpha \exp(-2\alpha\hSigma) + \hSigma^+ (I-\exp(-\alpha\hSigma))^2 \preceq 
1.2147 \alpha (\alpha \hSigma + I)^{-1}. 
$$
Letting $A,B$ denote the matrices on the left- and right-hand sides above,
since $A \preceq B$ and $\Sigma \succeq 0$, it holds that $\tr(A\Sigma)
\leq \tr(B\Sigma)$, which gives the desired result. 

\section{Proof of Theorem \ref{thm:gf_risk_out_lim}}
\label{sec:gf_risk_out_lim}

Denote $\C_- = \{ z \in \C : \Im(z) < 0 \}$.  By Lemma 2 in
\citet{ledoit2011eigenvectors}, under the conditions stated in the theorem,
for each $z \in \C_-$, we have 
\begin{equation}
\label{eq:lp_result}
\lim_{n,p \to \infty} \frac{1}{p} \tr \big[ (\hSigma + z I)^{-1} \Sigma \big]
\to \theta(z) := \frac{1}{\gamma} 
\bigg(\frac{1}{1-\gamma + \gamma z m(F_{H,\gamma})(-z)}  - 1\bigg),    
\end{equation}
almost surely, where \smash{$m(F_{H,\gamma})$} denotes the {\it Stieltjes
  transform} of the empirical spectral distribution \smash{$F_{H,\gamma}$}, 
\begin{equation}
\label{eq:stieltjes}
m(F_{H,\gamma})(z) = \int \frac{1}{u - z} \, dF_{H,\gamma}(u). 
\end{equation}
It is evident that \eqref{eq:lp_result} is helpful for understanding the Bayes
prediction risk of ridge regression \eqref{eq:ridge_risk_out_bayes}, where the
resolvent functional \smash{$\tr[(\hSigma + z I)^{-1} \Sigma]$} plays a
prominent role.  

For the Bayes prediction risk of gradient flow \eqref{eq:gf_risk_out_bayes}, the
connection is less clear.  However, the Laplace transform is the key link
between \eqref{eq:gf_risk_out_bayes} and \eqref{eq:lp_result}.  In particular,
defining $g(t) = \exp(tA)$, it is a standard fact that its Laplace transform
\smash{$\cL(g)(z) = \int e^{-tz} g(t) \, dt$} (meaning elementwise integration)
is in fact 
\begin{equation}
\label{eq:mat_exp_lap}
\cL(\exp(tA))(z) = (A - zI)^{-1}.
\end{equation}
Using linearity (and invertibility) of the Laplace transform, this means 
\begin{equation}
\label{eq:gf_bias_lap}
\exp(-2 t \hSigma) \Sigma = \cL^{-1} 
\big( (\hSigma + z I)^{-1} \Sigma \big)(2t), 
\end{equation}
Therefore, we have for the bias term in \eqref{eq:gf_risk_out_bayes},
\begin{align}
\nonumber
\frac{\sigma^2 \alpha}{n} \tr \big[ \exp(-2 t \hSigma) \Sigma \big]  
& = \frac{\sigma^2 \alpha}{n} \tr \big[ \cL^{-1} \big((\hSigma + z I)^{-1}  
  \Sigma \big)(2t) \big] \\ 
\label{eq:gf_bias_lap2}
& = \frac{\sigma^2 p \alpha}{n} \cL^{-1} 
\Big( \tr \big[ p^{-1} (\hSigma + z I)^{-1} \Sigma \big] \Big)(2t), 
\end{align}
where in the second line we again used linearity of the (inverse) Laplace
transform.  In what follows, we will show that we can commute the limit as $n,p
\to \infty$ with the inverse Laplace transform in \eqref{eq:gf_bias_lap2},
allowing us to apply the Ledoit-Peche result \eqref{eq:lp_result}, to derive an 
explicit form for the limiting bias.  We first give a more explicit
representation for the inverse Laplace transform in terms of a line integral in
the complex plane
$$
\frac{\sigma^2 p \alpha}{n} 
\cL^{-1} \Big( \tr \big[ p^{-1} (\hSigma + z I)^{-1} \Sigma \big] \Big)(2t)
= \frac{\sigma^2 p \alpha}{n} \frac{1}{2 \pi i} 
\int _{a - i \infty}^{a + i \infty} \tr \big[ p^{-1} (\hSigma + z I)^{-1} \Sigma
\big] \exp(2 t z) \, dz,   
$$
where $i = \sqrt{-1}$, and $a \in \R$ is chosen so that the line 
$[a - i \infty, a + i \infty]$ lies to the right of all singularities of
the map $z \mapsto \tr [p^{-1} (\hSigma + z I)^{-1} \Sigma]$.  Thus, we may fix 
any $a>0$, and reparametrize the integral above as
\begin{align}
\nonumber
\frac{\sigma^2 p \alpha}{n} 
\cL^{-1} \Big( \tr \big[ p^{-1} (\hSigma + z I)^{-1} \Sigma \big] \Big)(2t) 
&= \frac{\sigma^2 p \alpha}{n} \frac{1}{2 \pi} \int_{-\infty}^\infty 
\tr \big[ p^{-1} \big(\hSigma + (a+ib) I\big)^{-1} \Sigma \big] \exp(2t(a + ib))
  \, db \\  
\label{eq:gf_bias_lap3}
&= \frac{\sigma^2 p \alpha}{n} \frac{1}{\pi} \int_{-\infty}^0 
\Re \Big(\tr \big[ p^{-1} \big(\hSigma + (a+ib) I\big)^{-1} \Sigma \big]
  \exp(2t(a + ib)) \Big) \, db.  
\end{align}
The second line can be explained as follows.  A straightforward calculation,
given in Lemma \ref{lem:conj_prop_res}, shows that the function
\smash{$h_{n,p}(z)=\tr [p^{-1} \big(\hSigma + z I\big)^{-1} \Sigma] \exp(2tz)$}
satisfies \smash{$h_{n,p}(\overline{z}) = \overline{h_{n,p}(z)}$}; another short
calculation, deferred to Lemma \ref{lem:conj_prop_int}, shows that for any
function with such a property, its integral over a vertical line in the complex
plane reduces to the integral of twice its real part, over the line segment
below the real axis.  Now, noting that the integrand above satisfies
$$
|h_{n,p}(z)| \leq \|(\hSigma + z I)^{-1}\|_2 \| \Sigma \|_2 \leq C_2/a,
$$
for all $z \in [a - i\infty, a + i \infty]$, we can take limits in
\eqref{eq:gf_bias_lap3} and apply the dominated convergence theorem, to yield
that almost surely,
\begin{align}
\nonumber
\lim_{n,p \to \infty} \frac{\sigma^2 p \alpha}{n} 
\cL^{-1} \Big( \tr \big[ p^{-1} (\hSigma + z I)^{-1} &\Sigma \big] \Big)(2t) \\ 
\nonumber
&= \sigma^2 \gamma \alpha_0 \frac{1}{\pi} \int_{-\infty}^0
 \lim_{n,p \to \infty} \Re \Big(\tr \big[ p^{-1} \big(\hSigma + (a+ib)
  I\big)^{-1} \Sigma \big] \exp(2t(a + ib)) \Big) \, db \\
\nonumber
&= \sigma^2 \gamma \alpha_0 \frac{1}{\pi} 
\int_{-\infty}^0 \Re \big(\theta(a+ib) \exp(2t(a + ib)) \big) \, db \\ 
\nonumber
&= \sigma^2 \gamma \alpha_0 \frac{1}{2\pi} \int_{-\infty}^\infty 
  \theta(a+ib) \exp(2t(a + ib)) \, db \\  
\label{eq:gf_bias_lap4}
&= \sigma^2 \gamma \alpha_0 \cL^{-1}(\theta)(2t).
\end{align}
In the second equality, we used the Ledoit-Peche result \eqref{eq:lp_result}, 
which applies because $a+ib \in \C_-$ for $b$ in the range of integration. 
In the third and fourth equalities, we essentially reversed the arguments
leading to \eqref{eq:gf_bias_lap2}, but with $h(z)=\theta(z) \exp(2tz)$ in
place of \smash{$h_{n,p}$} (note that $h$ must also satisfy
\smash{$h(\overline{z}) = \overline{h(z)}$}, as it is the pointwise limit of
\smash{$h_{n,p}$}, which has this same property). 

As for the variance term in \eqref{eq:gf_risk_out_bayes}, consider
differentiating with respect to $t$, to yield  
\begin{align*}
\frac{d}{dt} \frac{\sigma^2}{n} \tr \big[ \hSigma^+ (I - \exp(-t \hSigma))^2
  \Sigma \big] 
& = \frac{2 \sigma^2}{n} \tr \big[ \hSigma^+ \hSigma (I - \exp(-t \hSigma))
  \exp(-t \hSigma) \Sigma \big] \\ 
& = \frac{2 \sigma^2}{n} \tr \big[ ( I - \exp(-t \hSigma) ) \exp(-t \hSigma)
  \Sigma \big],  
\end{align*}
with the second line following because the column space of \smash{$I - \exp(-t 
  \hSigma)$} matches that of \smash{$\hSigma$}.  The fundamental theorem of
calculus then implies that the variance equals
\begin{multline*}
\frac{2 \sigma^2}{n} \int_0^t \tr \big[ (\exp(-u \hSigma) - \exp(-2 u \hSigma)) 
\Sigma \big] \, du = \\ 
\frac{2 \sigma^2p}{n} \int_0^t 
\Big[\cL^{-1} \Big( \tr \big[ p^{-1} (\hSigma + z I)^{-1} \Sigma \big] \Big)(u)
- \cL^{-1} \Big( \tr \big[ p^{-1} (\hSigma + z I)^{-1} \Sigma \big] \Big)(2u)
\Big] \, du,
\end{multline*}
where the equality is due to inverting the Laplace transform fact
\eqref{eq:mat_exp_lap}, as done in \eqref{eq:gf_bias_lap} for the bias.  The
same arugments for the bias now carry over here, to imply 
\begin{multline}
\label{eq:gf_var_lap}
\lim_{n,p \to \infty} \frac{2 \sigma^2}{n} \int_0^t 
\Big[\cL^{-1} \Big( \tr \big[ p^{-1} (\hSigma + z I)^{-1} \Sigma \big] \Big)(u)
- \cL^{-1} \Big( \tr \big[ p^{-1} (\hSigma + z I)^{-1} \Sigma \big] \Big)(2u)
\Big] \, du = \\ 
2 \sigma^2 \gamma \int_0^t \big(\cL^{-1}(\theta)(u) - \cL^{-1}(\theta)(2u)\big)
\, du. 
\end{multline}
Putting together \eqref{eq:gf_bias_lap4} and \eqref{eq:gf_var_lap} completes the
proof. 

\section{Supporting Lemmas}

\begin{lemma} 
\label{lem:conj_prop_res}
For any real matrices $A, B \succeq 0$ and $t \geq 0$, define   
$$
f(z) = \tr\big[(A + zI)^{-1} B\big] \exp(2 t z),
$$
over $z \in \C_+ = \{z \in \C : \Im(z) > 0\}$.  Then 
\smash{$f(\overline{z}) = \overline{f(z)}$}. 
\end{lemma}

\begin{proof}
First note that \smash{$\exp(2t \overline{z}) = \overline{\exp(2tz)}$} by
Euler's formula. As the conjugate of a product is the product of conjugates, it 
suffices to show that \smash{$\tr[(A + \overline{z}I)^{-1} B] = \overline{\tr[(A
    + zI)^{-1} B]}$}.  To this end, denote \smash{$C_z = (A + z I)^{-1}$}, and
denote by $C_z^*$ its adjoint (conjugate transpose).  Note that
\smash{$\overline{\tr(C_z B)} = \tr(C_z^* B)$}; we will show that \smash{$C_z^* 
  =  C_{\overline{z}}$}, which would then imply the desired result. 
Equivalent to \smash{$C_z^* =  C_{\overline{z}}$} is \smash{$\langle C_z x, y
  \rangle = \langle x, C_{\overline{z}} y \rangle$} for all complex vectors
$x,y$ (where $\langle \cdot, \cdot \rangle$ denotes the standard inner product).
Observe 
\begin{align*}
\langle C_z x, y \rangle 
&= \langle C_z x, \, (A + \overline{z}I) C_{\overline{z}} y \rangle \\
&= \langle (A + \overline{z}I)^* C_z x, \, C_{\overline{z}} y \rangle \\
&= \langle (A + zI) C_z x, \, C_{\overline{z}} y \rangle \\
&= \langle x, C_{\overline{z}} y \rangle,
\end{align*}
which completes the proof.
\end{proof}

\begin{lemma} 
\label{lem:conj_prop_int}
If $f : \C \to \C$ satisfies \smash{$f(\overline{z}) = \overline{f(z)}$}, then 
for any $a \in \R$,
$$
\int_{-\infty}^\infty f(a+ib) \, db = 2 \int_{-\infty}^0 \Re(f(a+ib)) \, db. 
$$
\end{lemma}

\begin{proof}
The property \smash{$f(\overline{z}) = \overline{f(z)}$} means that
$\Re(f(a-ib)) = \Re(f(a+ib))$, and $\Im(f(a-ib))=-\Im(f(a+ib))$.  Thus
\begin{align*}
\int_{-\infty}^\infty f(a+ib) \, db 
&= \int_{-\infty}^\infty \Re(f(a+ib)) \, db + 
i \int_{-\infty}^\infty \Im(f(a+ib)) \, db \\
&= 2 \int_{-\infty}^0 \Re(f(a+ib)) \, db + 0,
\end{align*}
which completes the proof.
\end{proof}

\section{Asymptotics for Ridge Regression} 

Under the conditions of Theorem \ref{thm:gf_risk_lim}, for each $\lambda \geq
0$, the Bayes risk \eqref{eq:ridge_risk_bayes} of ridge regression converges
almost surely to 
\begin{equation}
\label{eq:ridge_risk_lim} 
\sigma^2 \gamma \int \frac{\alpha_0 \lambda^2 + s}{(s + \lambda)^2} \,
dF_{H,\gamma}. 
\end{equation}
This is simply an application of weak convergence of \smash{$F_{\hSigma}$} to
\smash{$F_{H,\gamma}$} (as argued the proof of Theorem \ref{thm:gf_risk_lim}),
and can also be found in, \eg, Chapter 3 of \citet{tulino2004random}.  

The limiting Bayes prediction risk is a more difficult calculation.  It is shown in
\citet{dobriban2018high} that, under the conditions of Theorem
\ref{thm:gf_risk_out_lim},  for each $\lambda \geq
0$, the Bayes prediction risk \eqref{eq:ridge_risk_out_bayes} of ridge
regression converges almost surely to 
\begin{equation}
\label{eq:ridge_risk_out_lim} 
\sigma^2 \gamma \big[
\theta(\lambda) + \lambda (1 - \alpha_0 \lambda) \theta'(\lambda) \big], 
\end{equation}
where $\theta(\lambda)$ is as defined in \eqref{eq:lp_result}.  The calculation
\eqref{eq:ridge_risk_out_bayes} makes use of the Ledoit-Peche result
\eqref{eq:lp_result}, and Vitali's theorem (to assure the convergence of the 
derivative of the resolvent functional in \eqref{eq:lp_result}).   

It is interesting to compare the limiting Bayes prediction risks
\eqref{eq:ridge_risk_out_lim} and \eqref{eq:gf_risk_out_lim}.  For concreteness,
we can rewrite the latter as
\begin{equation}
\label{eq:gf_risk_out_lim2}
\sigma^2 \gamma \bigg[ \alpha_0 \cL^{-1}(\theta)(2t) + 2 \int_0^t
(\cL^{-1}(\theta) (u) - \cL^{-1}(\theta) (2u)) \, du\bigg].
\end{equation}
We see that \eqref{eq:ridge_risk_out_lim} features $\theta$ and its 
derivative, while \eqref{eq:gf_risk_out_lim2} features the 
inverse Laplace transform $\cL^{-1}(\theta)$ and its antiderivative.  

In fact, a similar structure can be observed by rewriting the limiting
risks \eqref{eq:ridge_risk_lim} and \eqref{eq:gf_risk_lim}.   By
simply expanding $s=(s+\lambda)-\lambda$ in the numerator in
\eqref{eq:ridge_risk_lim}, and using the definition of the Stieltjes
transform \eqref{eq:stieltjes}, the limiting Bayes risk of ridge becomes
\begin{equation}
\label{eq:ridge_risk_lim2} 
\sigma^2 \gamma \big[ m(F_{H,\gamma})(-\lambda) -
\lambda (1 - \alpha_0 \lambda) m(F_{H,\gamma})'(-\lambda) \big].
\end{equation}
By following arguments similar to the treatment of the variance term in the
proof of Theorem \ref{thm:gf_risk_out_lim}, in Section
\ref{sec:gf_risk_out_lim}, the limiting Bayes risk of gradient flow becomes
\begin{equation}
\label{eq:gf_risk_lim2} 
\sigma^2 \gamma \bigg[ \alpha_0 \cL(f_{H,\gamma})(2t) + 2 \int_0^t
(\cL(f_{H,\gamma}) (u) - \cL(f_{H,\gamma}) (2u)) \, du\bigg],
\end{equation}
where \smash{$f_{H,\gamma}=dF_{H,\gamma}/ds$} denotes the density of the
empirical spectral distribution \smash{$F_{H,\gamma}$}, and
\smash{$\cL(f_{H,\gamma})$} its Laplace transform.  We see
\eqref{eq:ridge_risk_lim2} features \smash{$m(F_{H,\lambda})$} and its 
derivative, and \eqref{eq:gf_risk_lim2} features \smash{$\cL(f_{H,\gamma})$} and 
its antiderivative.  But indeed \smash{$\cL(\cL(f_{H,\gamma}))(\lambda) =
  m(F_{H,\lambda})(-\lambda)$}, since we can (in general) view the Stieltjes 
transform as an iterated Laplace transform.  This creates a symmetric link 
between \eqref{eq:ridge_risk_lim2}, \eqref{eq:gf_risk_lim2} and 
\eqref{eq:ridge_risk_out_lim}, \eqref{eq:gf_risk_out_lim2}, where
\smash{$m(F_{H,\gamma})(-\lambda)$} in the former plays the role of
$\theta(\lambda)$ in the latter.

\section{Additional Numerical Results}

Here we show the complete set of numerical results comparing gradient flow and
ridge regression.  The setup is as described in Section \ref{sec:exps}.  Figure 
\ref{fig:risk_gaus_lo} shows the results for Gaussian features in the
low-dimensional case ($n=1000$, $p=500$).  The first row shows the estimation
risk when $\Sigma=I$, with the left plot using $\lambda=1/t$ calibration, and
the right plot using $\ell_2$ norm calibration (details on this calibration
explained below).  The second row shows the estimation risk when $\Sigma$ has
all off-diagonals equal to $\rho=0.5$.  The third row shows the prediction risk
for the same $\Sigma$ (n.b., the prediction risk when $\Sigma=I$ is the same as
the estimation risk, so it is redundant to show both).  The conclusions
throughout are similar to that made in Section \ref{sec:exps}. Calibration by
$\ell_2$ norm gives extremely good agreement: the maximum ratio of gradient flow
to ridge risk (over the entire path, in any of the three rows) is 1.0367.
Calibration by $\lambda=1/t$ is still quite good, but markedly worse: the
maximum ratio of gradient flow to ridge risk (again over the entire path, in any
of the three rows) is 1.4158. 

Figures \ref{fig:risk_gaus_hi} shows analogous results for Gaussian features in
the high-dimensional case ($n=500$, $p=1000$).  Figures
\ref{fig:risk_t_lo}--\ref{fig:risk_bern_hi} show the results for Student
$t$ and Bernoulli features.  The results are similar throughout: the maximum
ratio of gradient flow to ridge risk, under $\ell_2$ norm calibration (over the
entire path, in any setting), is 1.0371; the maximum ratio, under $\lambda=1/t$ 
calibration (over the entire path, in any setting), is 1.4154.
(One noticeable, but unremarkable difference between the settings is that the 
finite-sample risks seem to be converging slower to their asymptotic analogs in
the case of $t$ features.  This is likely due to the fact that the tails here
are very fat---they are as fat as possible for the $t$ family, subject to the
second moment being finite.)

It helps to give further details for a few of the calculations.  For $\ell_2$
norm calibration, note that we can compute the expected squared $\ell_2$ norm of
the ridge and gradient flow estimators under the data model \eqref{eq:model_y}
and prior \eqref{eq:model_b0}: 
\begin{align*}
\E \|\hbeta^\ridge(\lambda)\|_2^2 &=
\frac{1}{n} \Big( \tr \big[\alpha (\hSigma + \lambda I)^{-2} \hSigma^2\big] 
+ \tr \big[ (\hSigma + \lambda I)^{-2} \hSigma\big] \Big) \\ 
&= \frac{1}{n} \sum_{i=1}^p \frac{\alpha s_i^2 + s_i}{(s_i+\lambda)^2}, \\
\E \|\hbeta^\gf(t)\|_2^2 &=
\frac{1}{n} \Big( \tr \big[ \alpha (I-\exp(-t\hSigma))^2\big] +
\tr \big[(I-\exp(-t\hSigma))^2 \hSigma^+\big] \Big) \\
&= \frac{1}{n} \sum_{i=1}^p \bigg( \alpha  (1-\exp(-ts_i))^2 + 
\frac{(1-\exp(-ts_i))^2}{s_i} \bigg). 
\end{align*} 
We thus calibrate according to the square root of the quantities above (this is
what is plotted on the x-axis in the left columns of all the figures).  The
above expressions have the following limits under the asymptotic model studied
in Theorem \ref{thm:gf_risk_lim}:
\begin{align*}
\E \|\hbeta^\ridge(\lambda)\|_2^2 &\to \gamma \int \frac{\alpha_0 s^2 + s}
{(s+\lambda)^2} \, dF_{H,\gamma}(s), \\   
\E \|\hbeta^\gf(t)\|_2^2 &\to \gamma \int \bigg( \alpha_0 (1-\exp(-ts))^2 +  
\frac{(1-\exp(-ts))^2}{s} \bigg) \, dF_{H,\gamma}(s). 
\end{align*} 

Furthermore, we note that when $\Sigma=I$, the empirical spectral distribution
from Theorem \ref{thm:mp}  abbreviated as $F_\gamma$, sometimes called the {\em
  Marchenko-Pastur (MP) law} and has a closed form.  For $\gamma \leq 1$, its
density is  
$$
\frac{d F_\gamma(s)}{ds} = \frac{1}{2 \pi \gamma s} \sqrt{(b-s)(s-a)},
$$
and is supported on $[a,b]$, where \smash{$a = (1-\sqrt\gamma)^2$} and
\smash{$b = (1+\sqrt\gamma)^2$}.  For $\gamma > 1$, the MP law
$F_\gamma$ has an additional point mass at zero of probability $1-1/\gamma$.   
This allows us to evaluate the integrals in \eqref{eq:gf_risk_lim},
\eqref{eq:ridge_risk_lim} via numerical integration, to compute limiting risks 
for gradient flow and ridge regression.  (It also allows us to compute the
integrals in the second to last display, to calibrate according to limiting
$\ell_2$ norms.)

\begin{figure*}[p]
\centering
\includegraphics[width=0.475\textwidth]{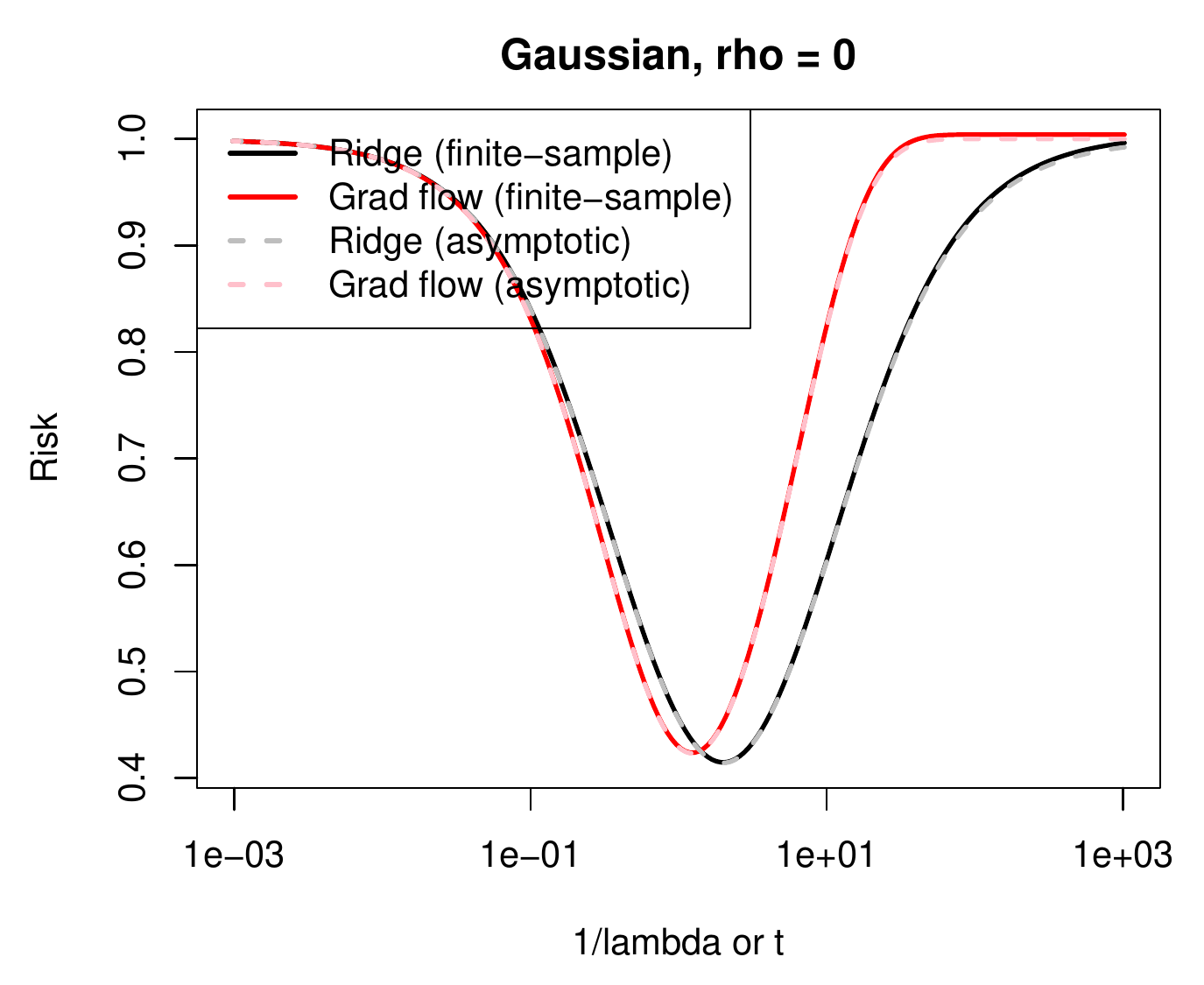} 
\includegraphics[width=0.475\textwidth]{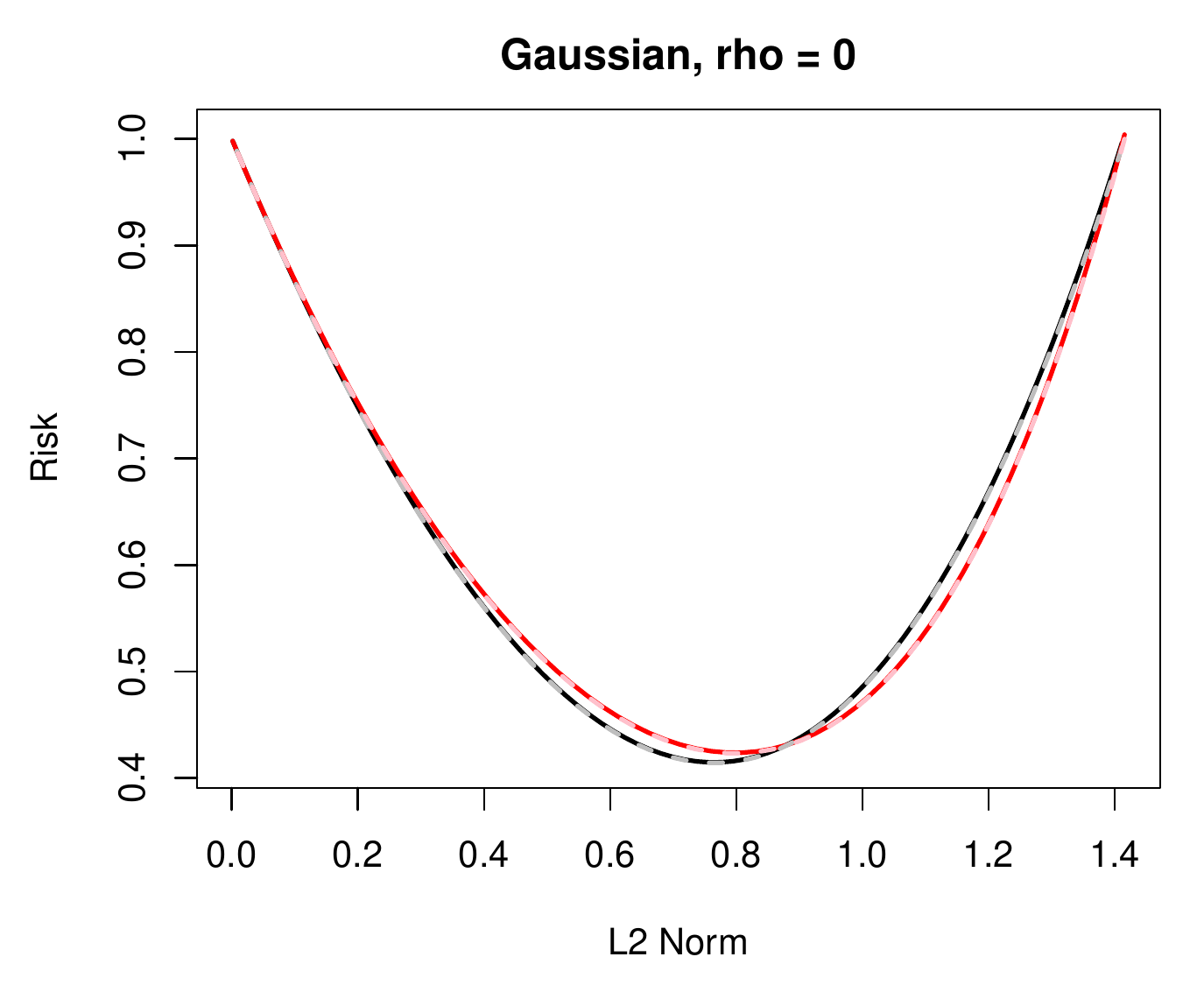}
\includegraphics[width=0.475\textwidth]{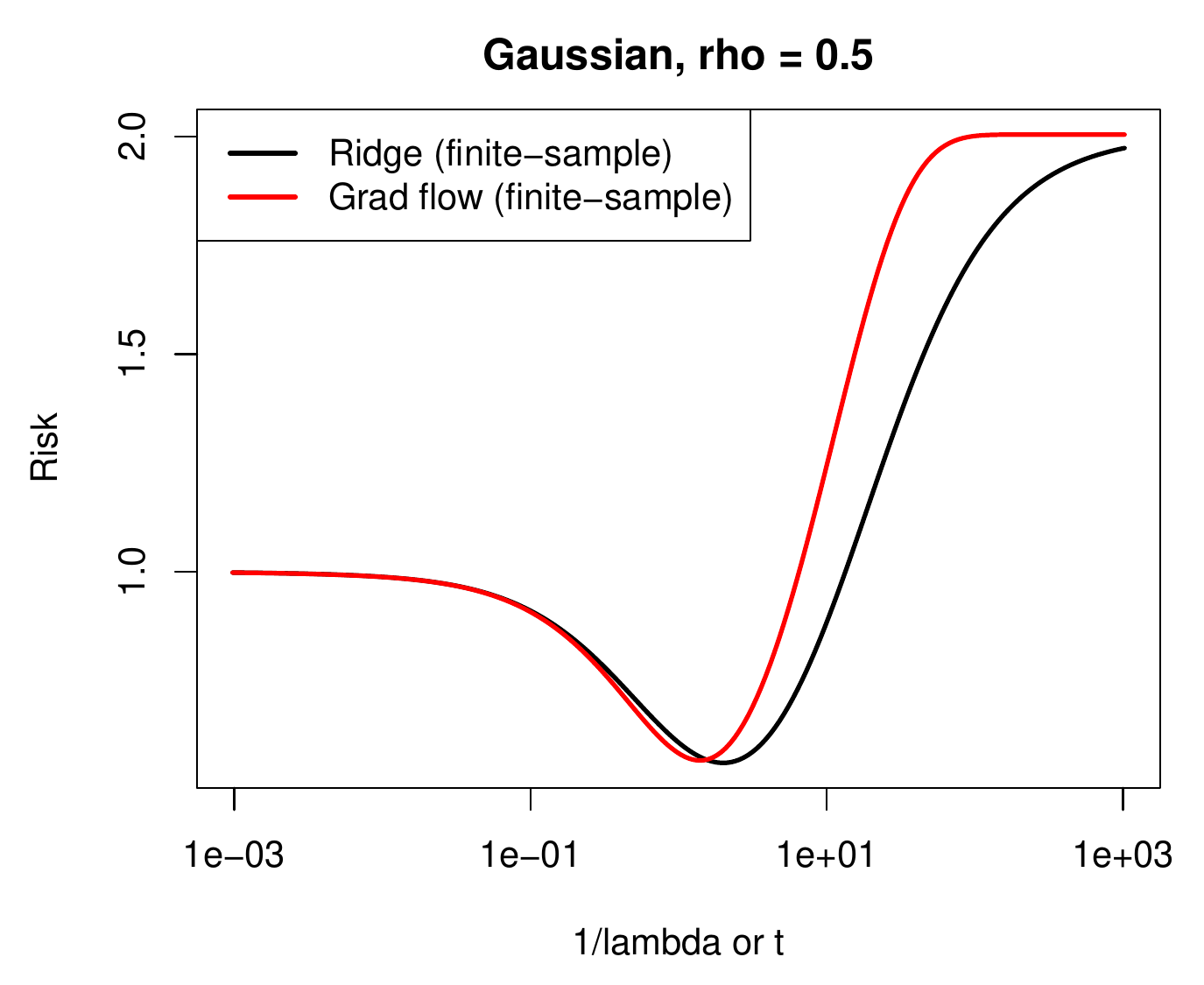} 
\includegraphics[width=0.475\textwidth]{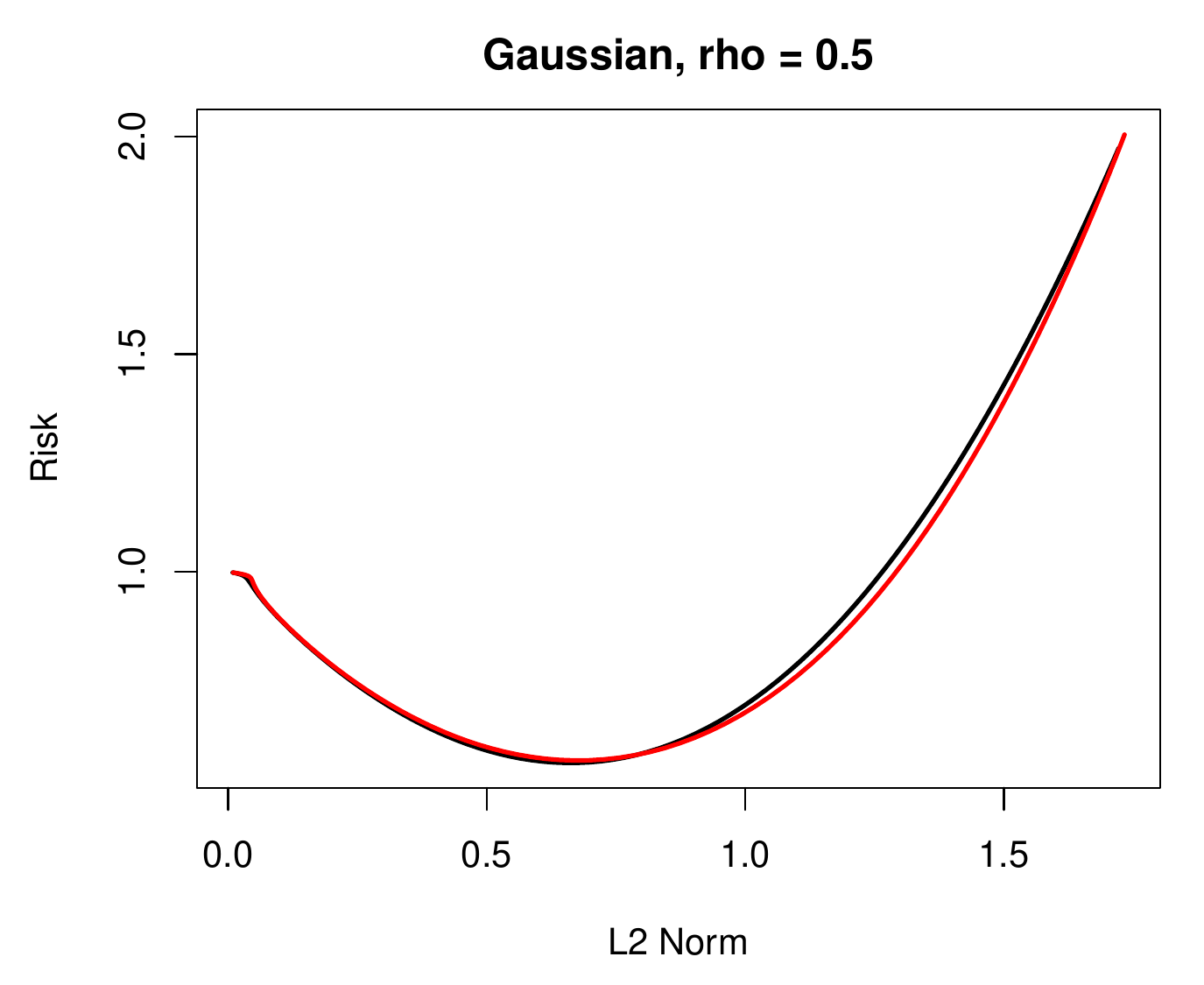}
\includegraphics[width=0.475\textwidth]{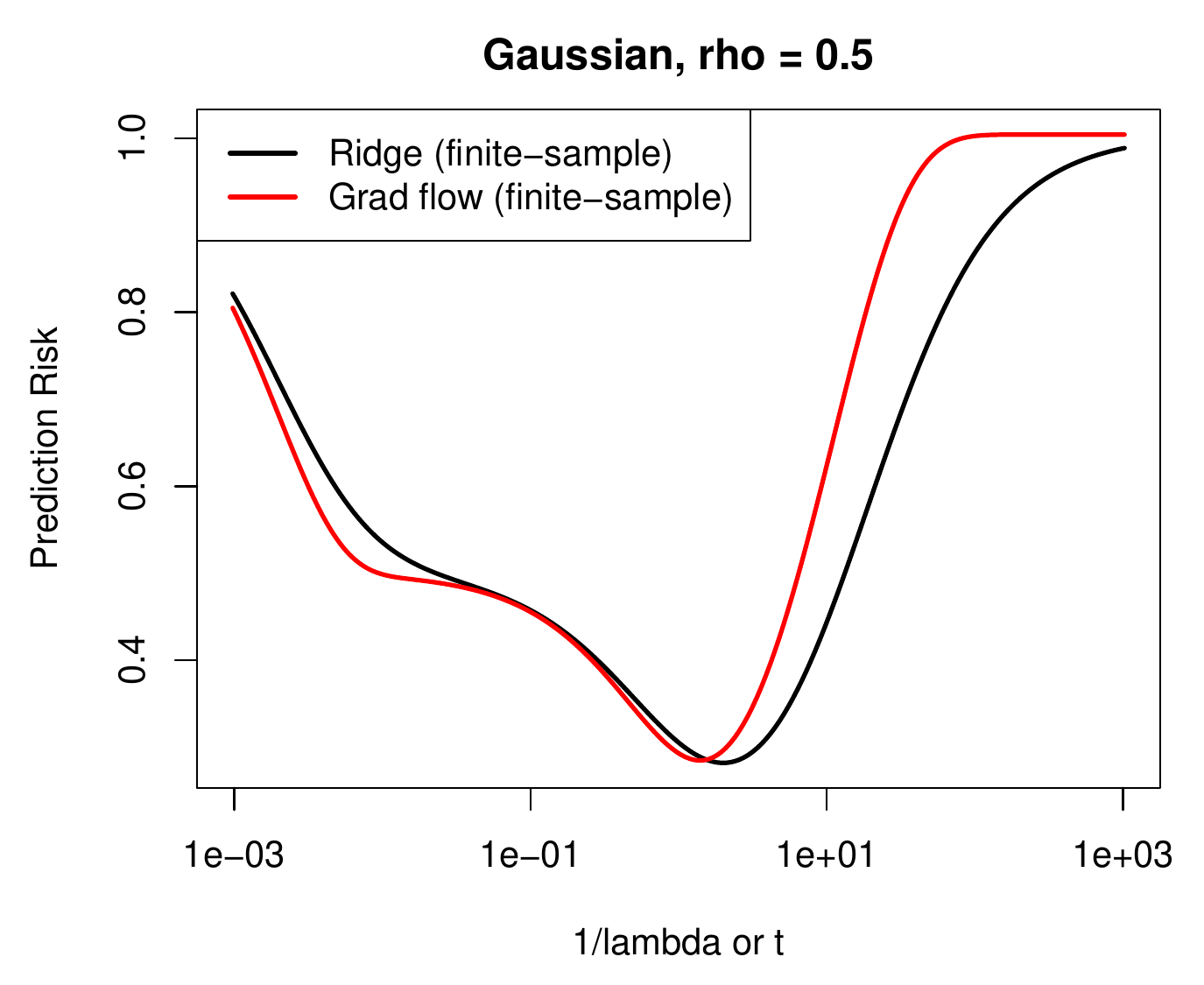} 
\includegraphics[width=0.475\textwidth]{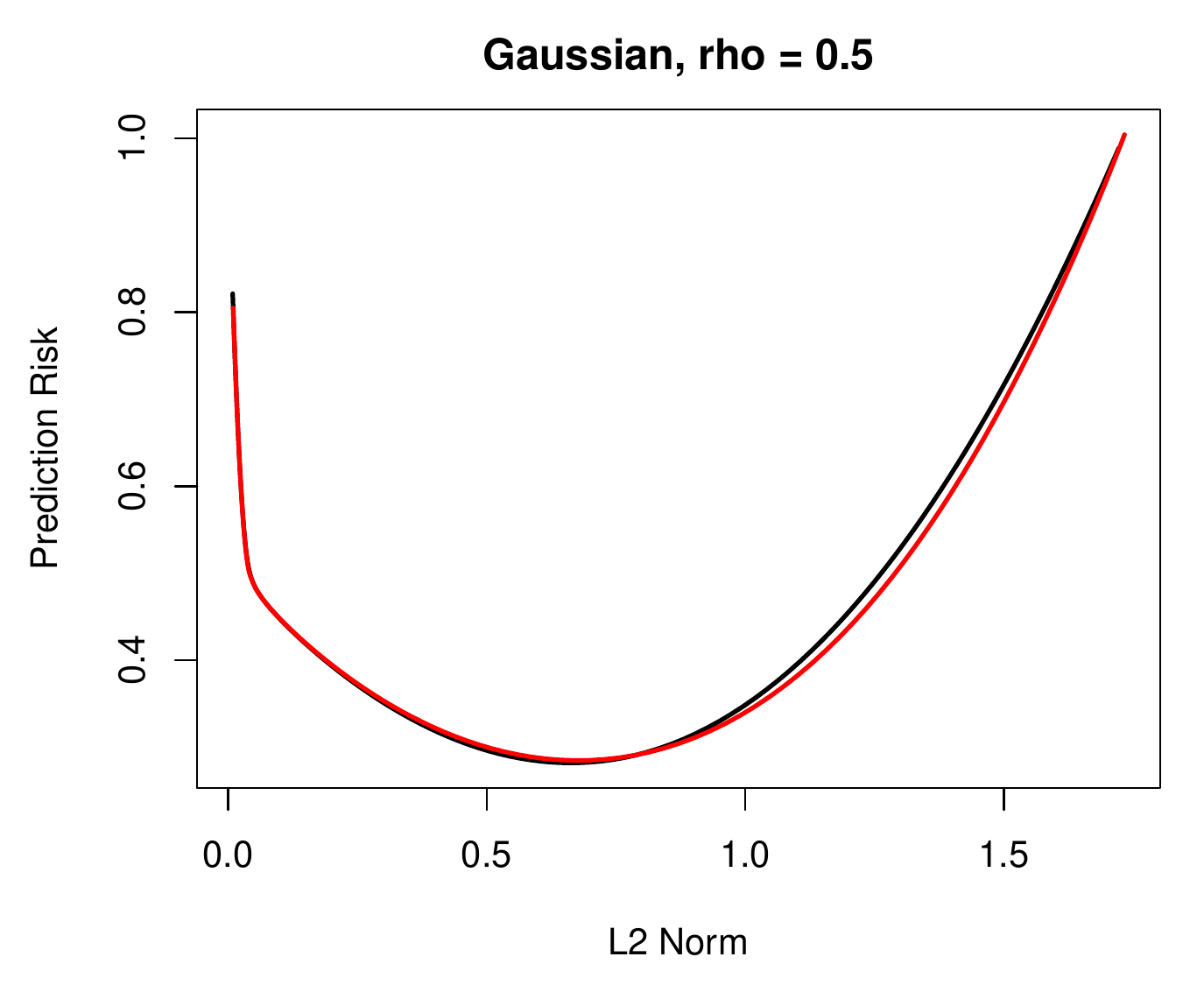}
\caption{\it \small Gaussian features, with $n=1000$ and $p=500$.}
\label{fig:risk_gaus_lo} 
\end{figure*}

\begin{figure*}[p]
\centering
\includegraphics[width=0.475\textwidth]{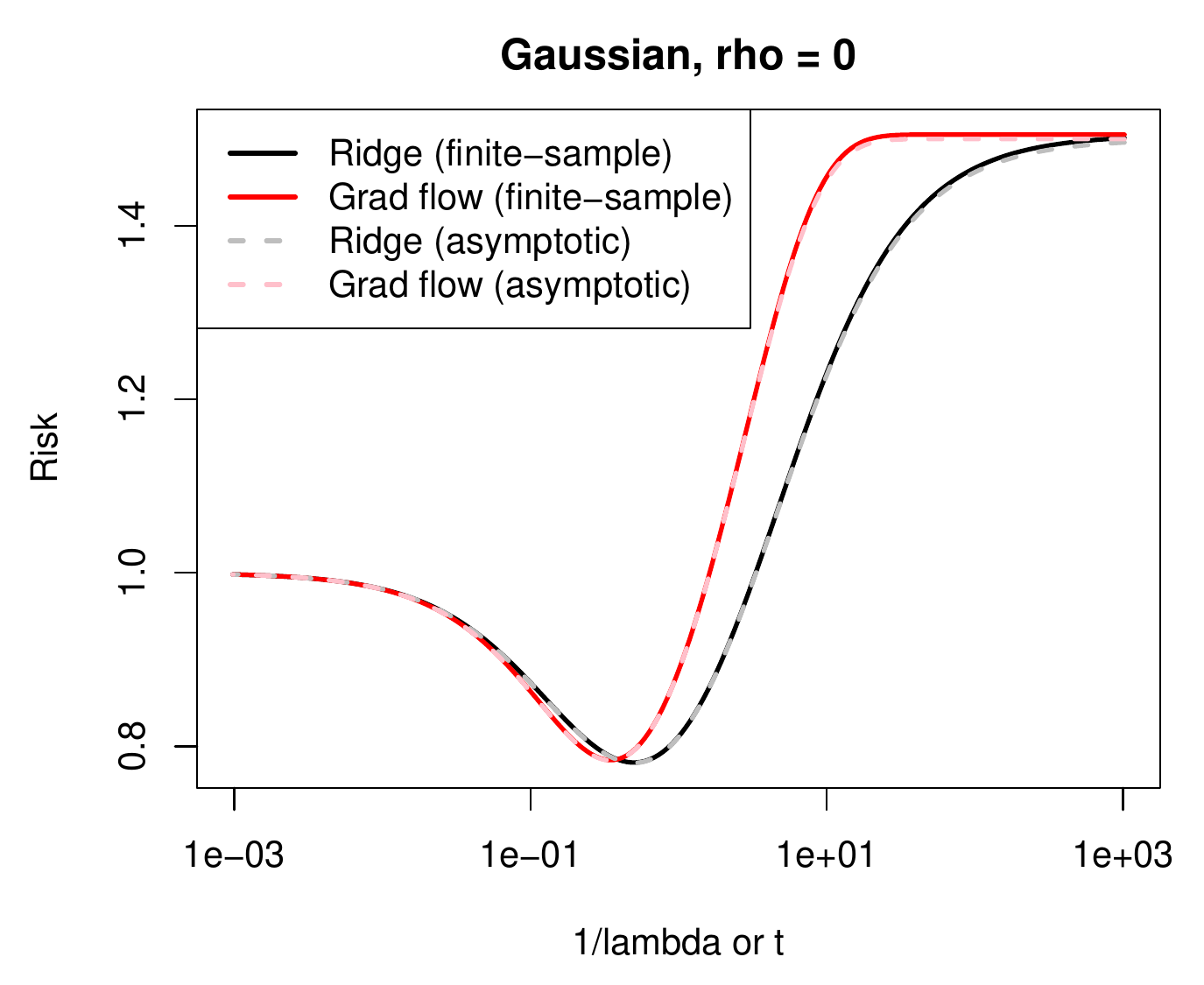} 
\includegraphics[width=0.475\textwidth]{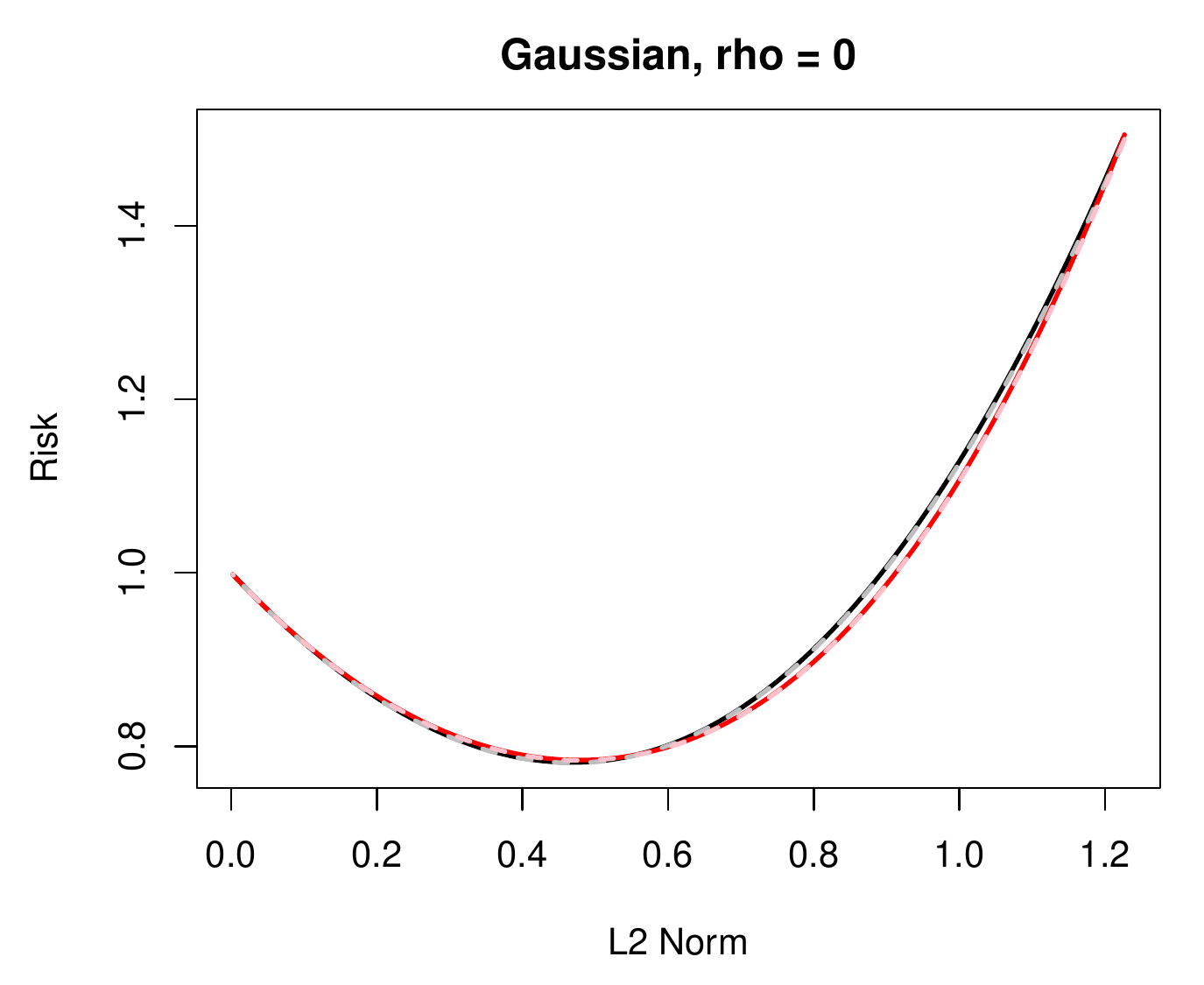}
\includegraphics[width=0.475\textwidth]{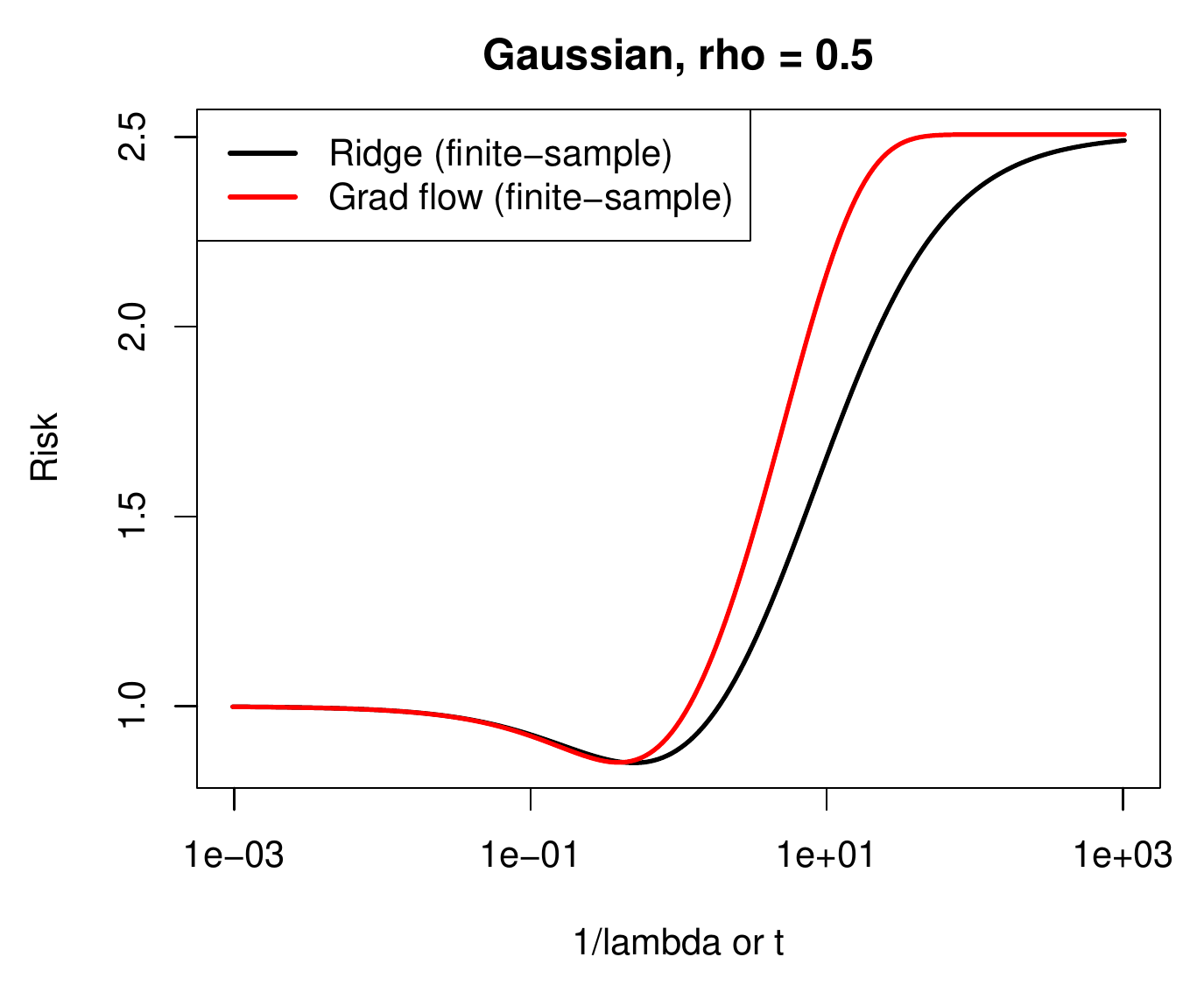} 
\includegraphics[width=0.475\textwidth]{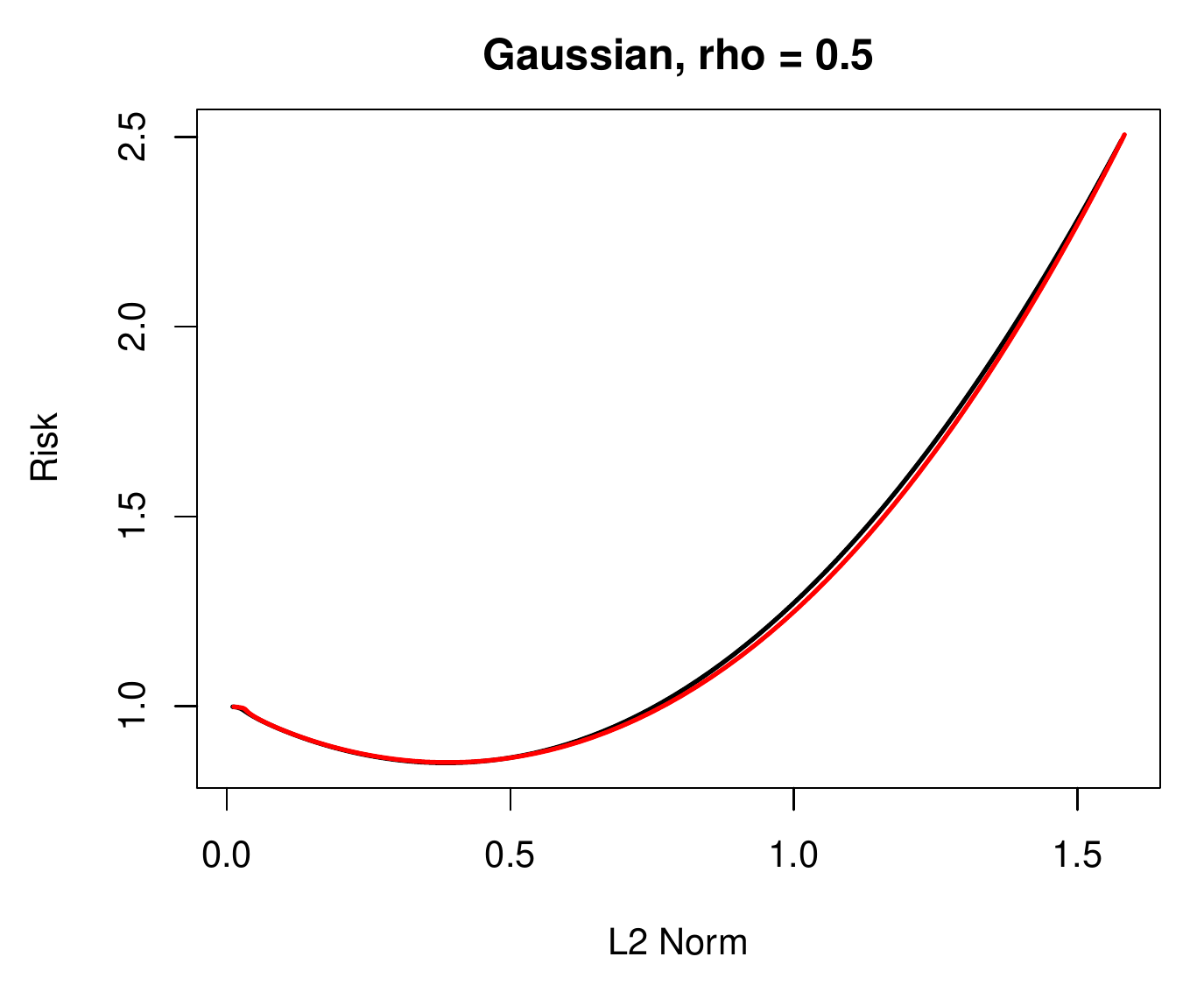}
\includegraphics[width=0.475\textwidth]{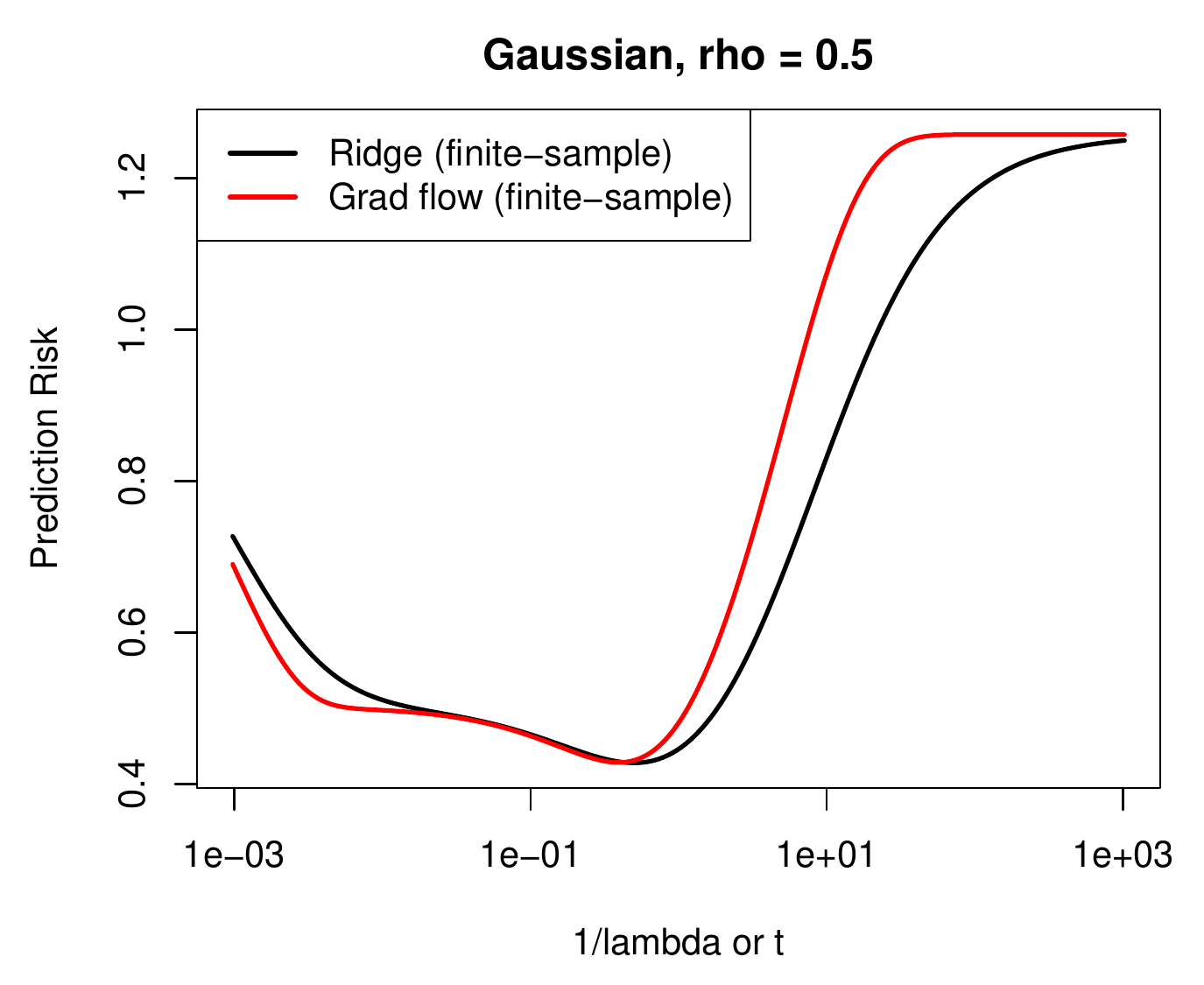} 
\includegraphics[width=0.475\textwidth]{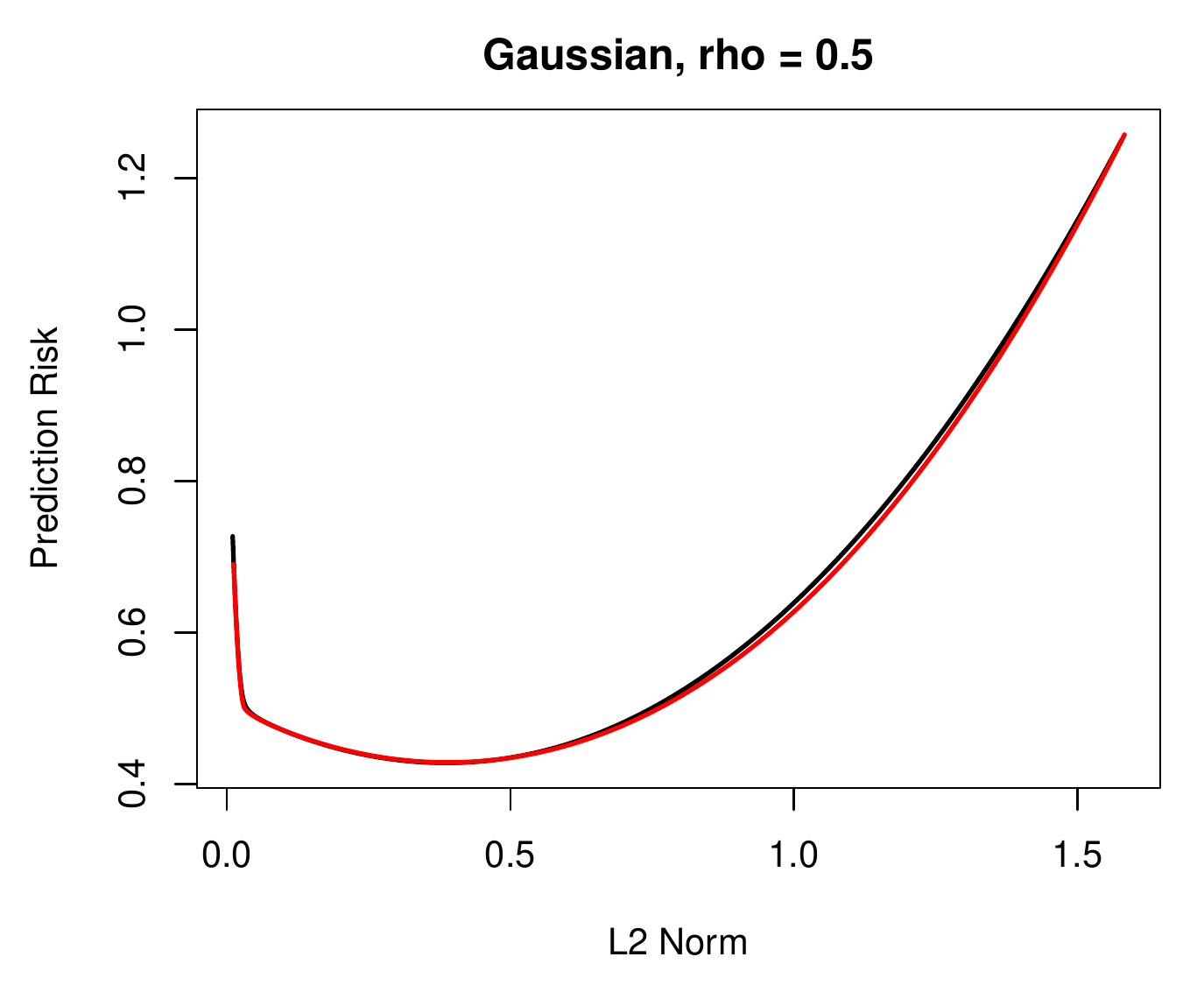}
\caption{\it \small Gaussian features, with $n=500$ and $p=1000$.}
\label{fig:risk_gaus_hi} 
\end{figure*}

\begin{figure*}[p]
\centering
\includegraphics[width=0.475\textwidth]{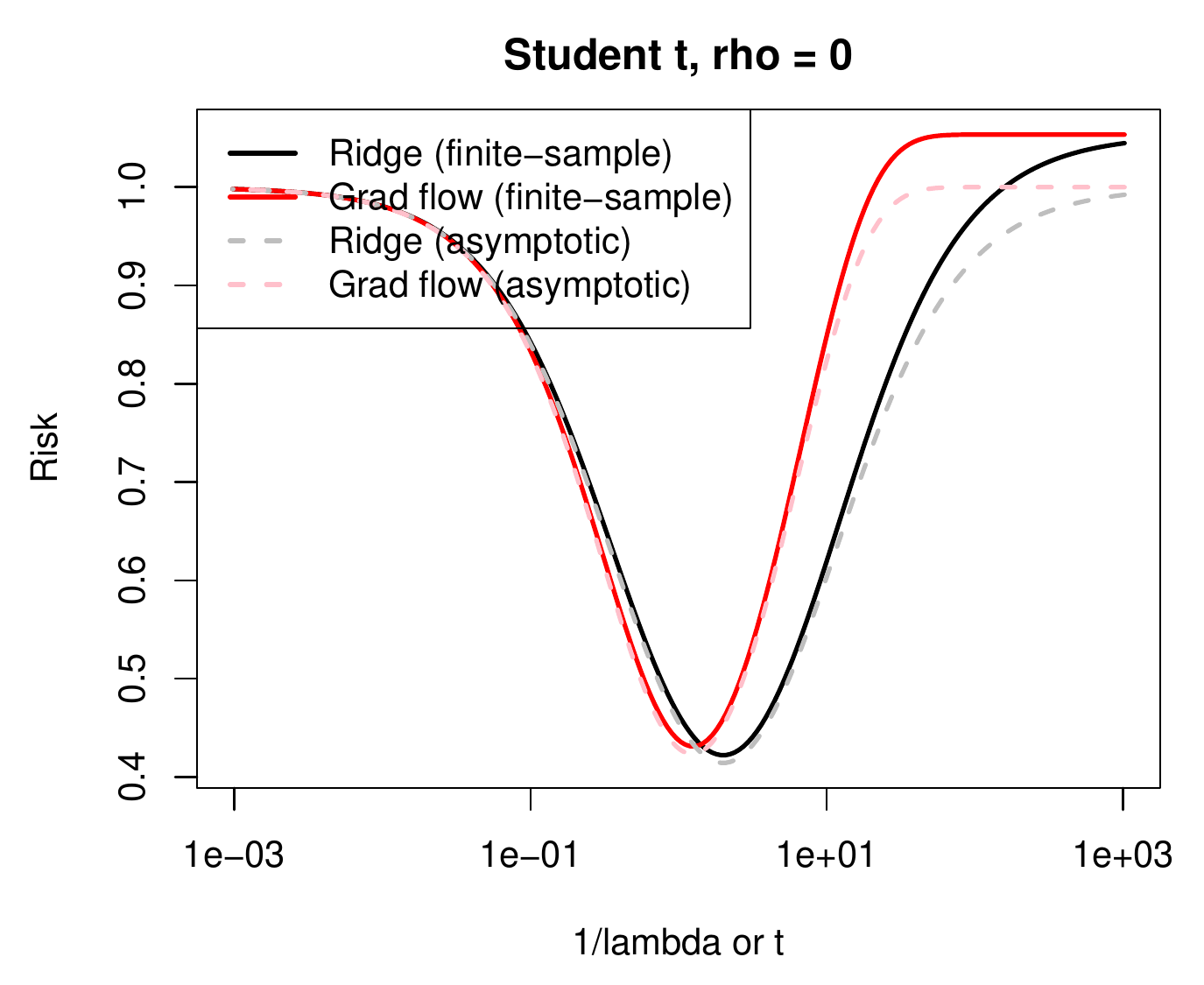} 
\includegraphics[width=0.475\textwidth]{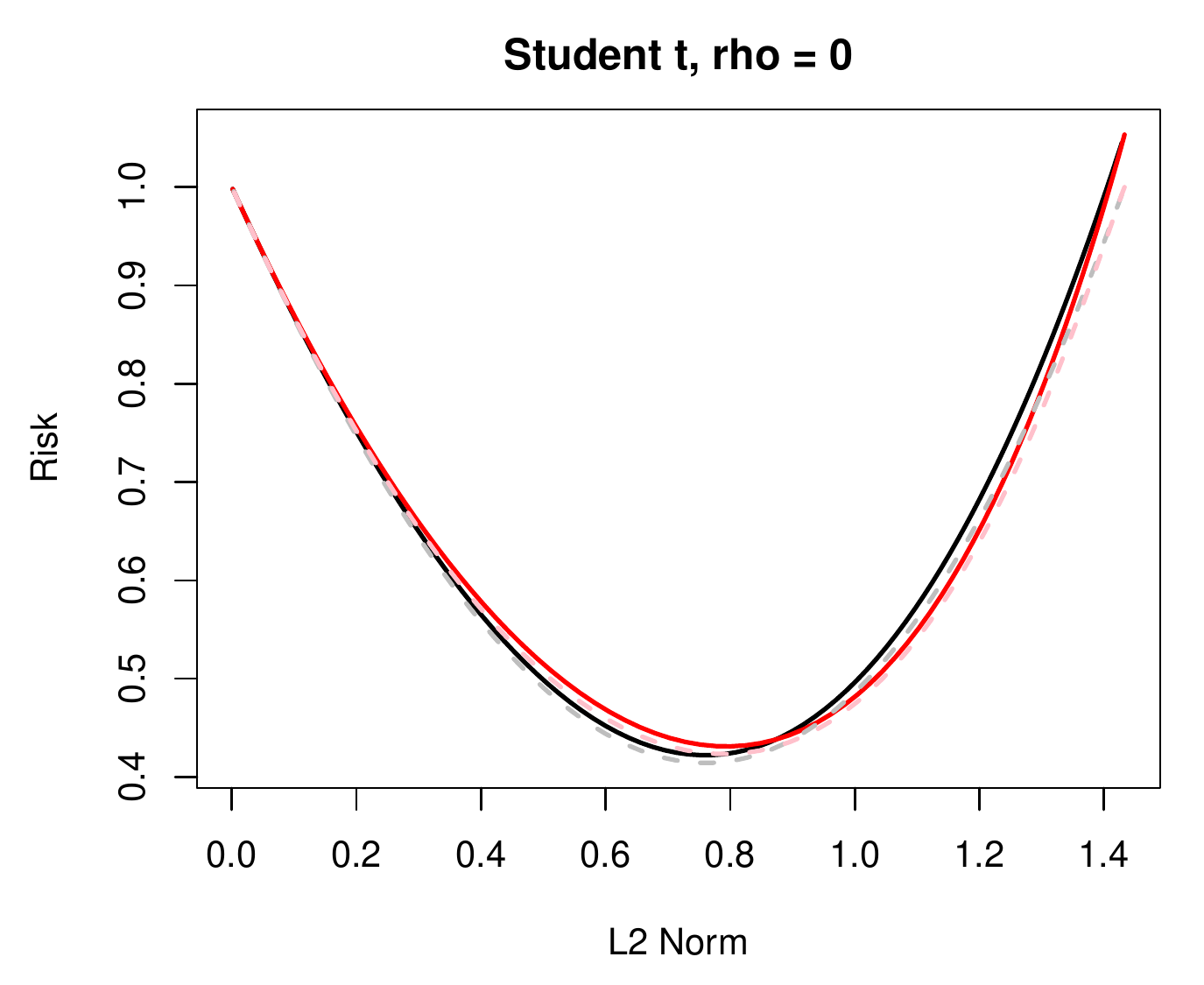}
\includegraphics[width=0.475\textwidth]{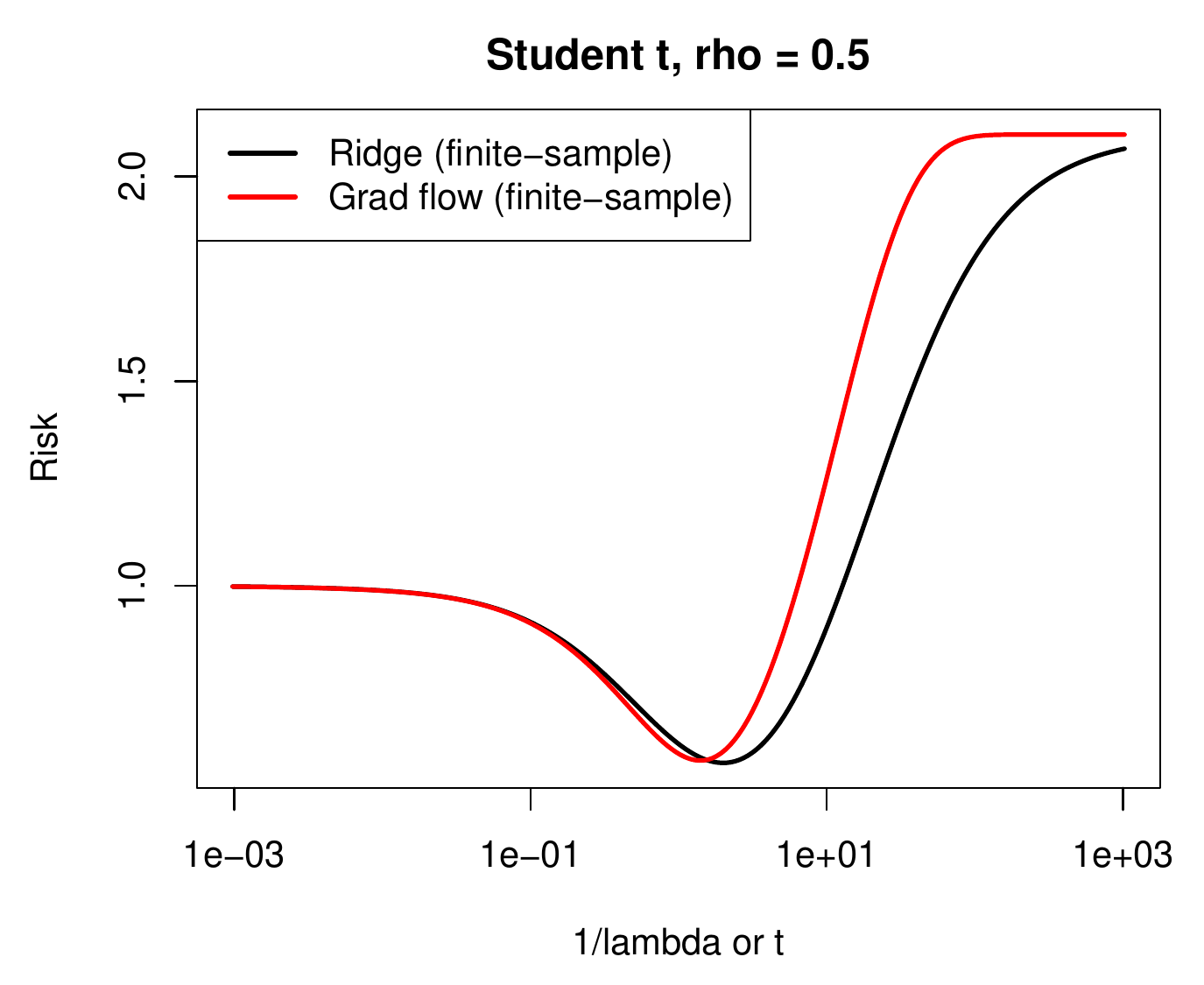} 
\includegraphics[width=0.475\textwidth]{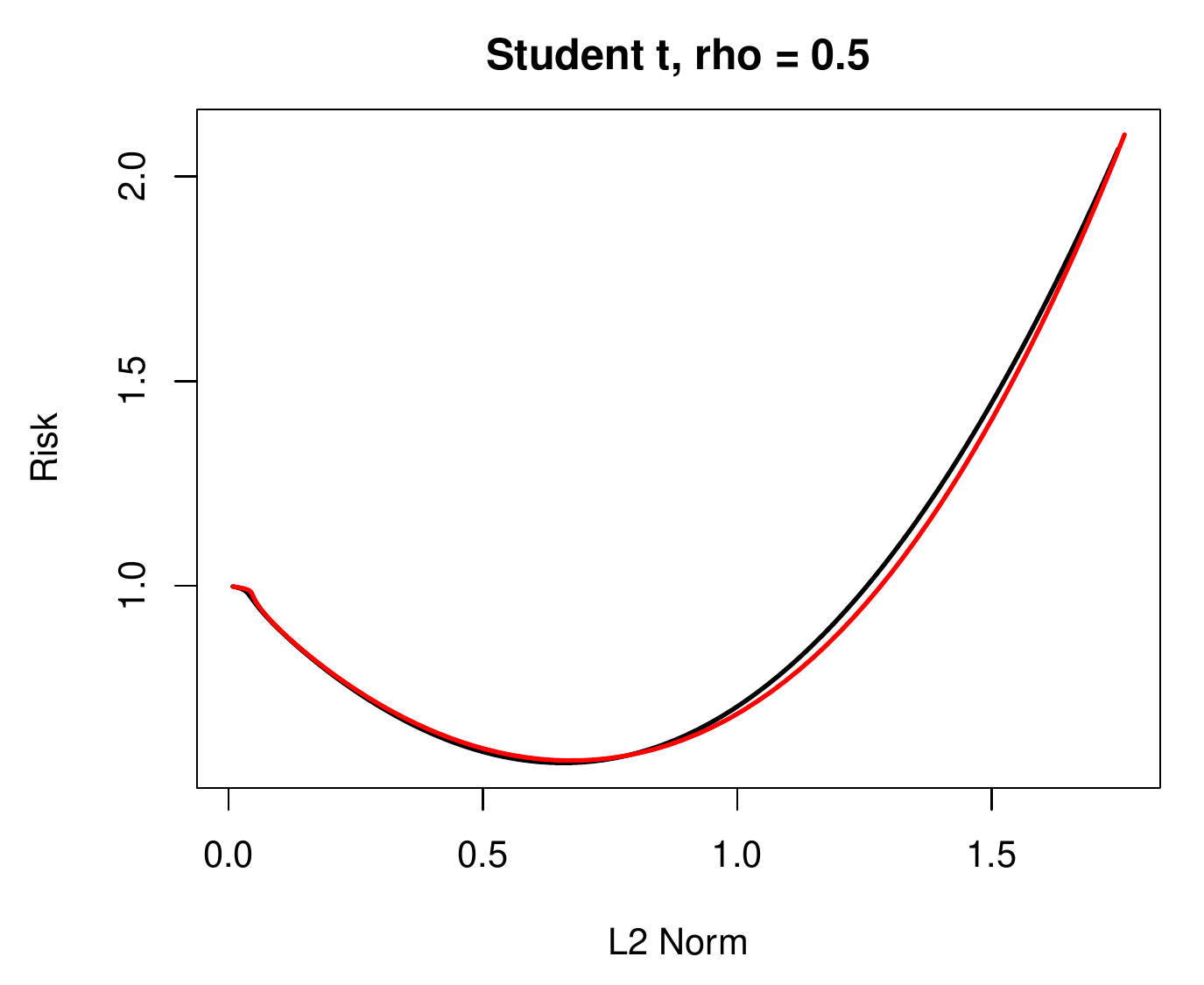}
\includegraphics[width=0.475\textwidth]{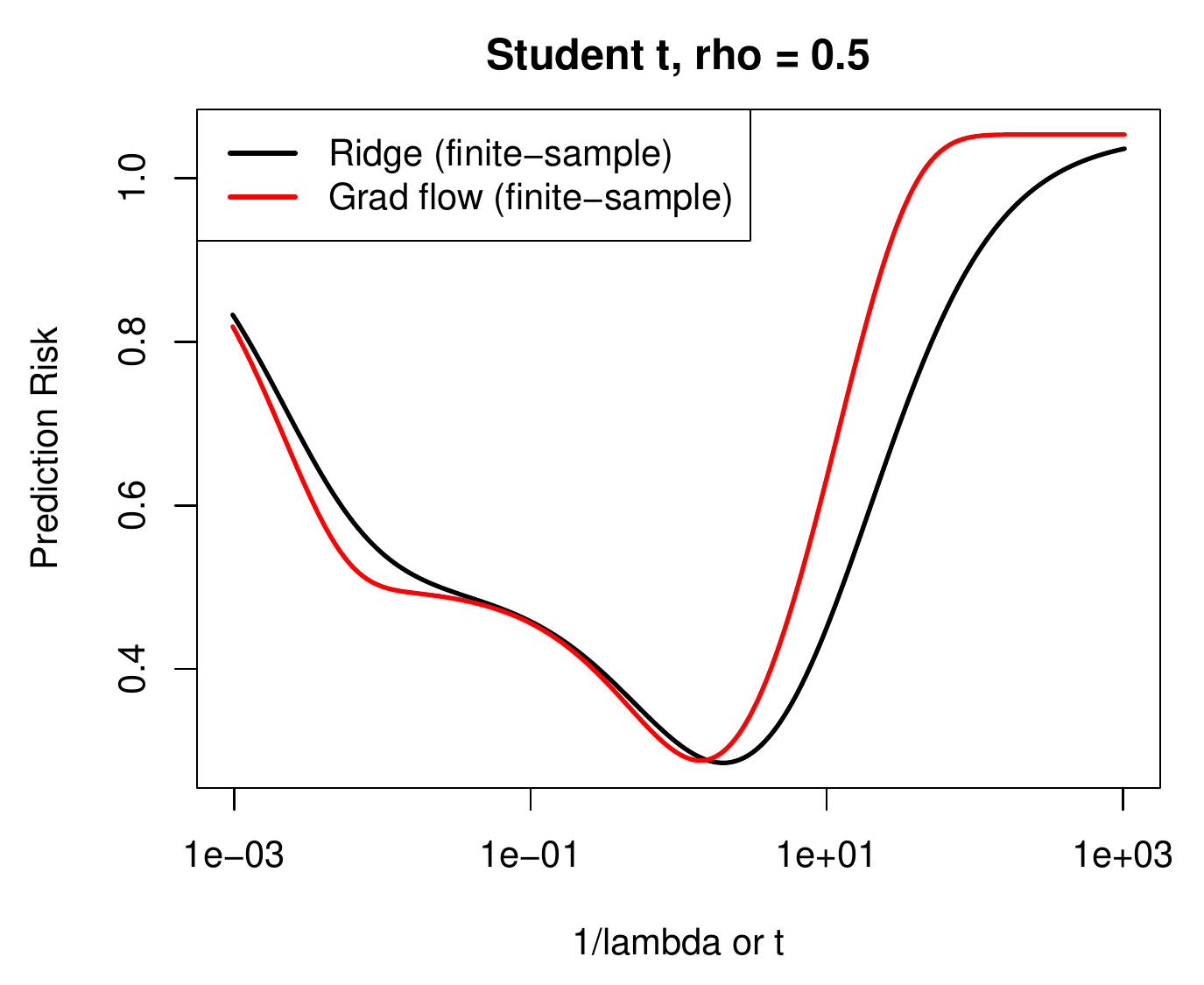} 
\includegraphics[width=0.475\textwidth]{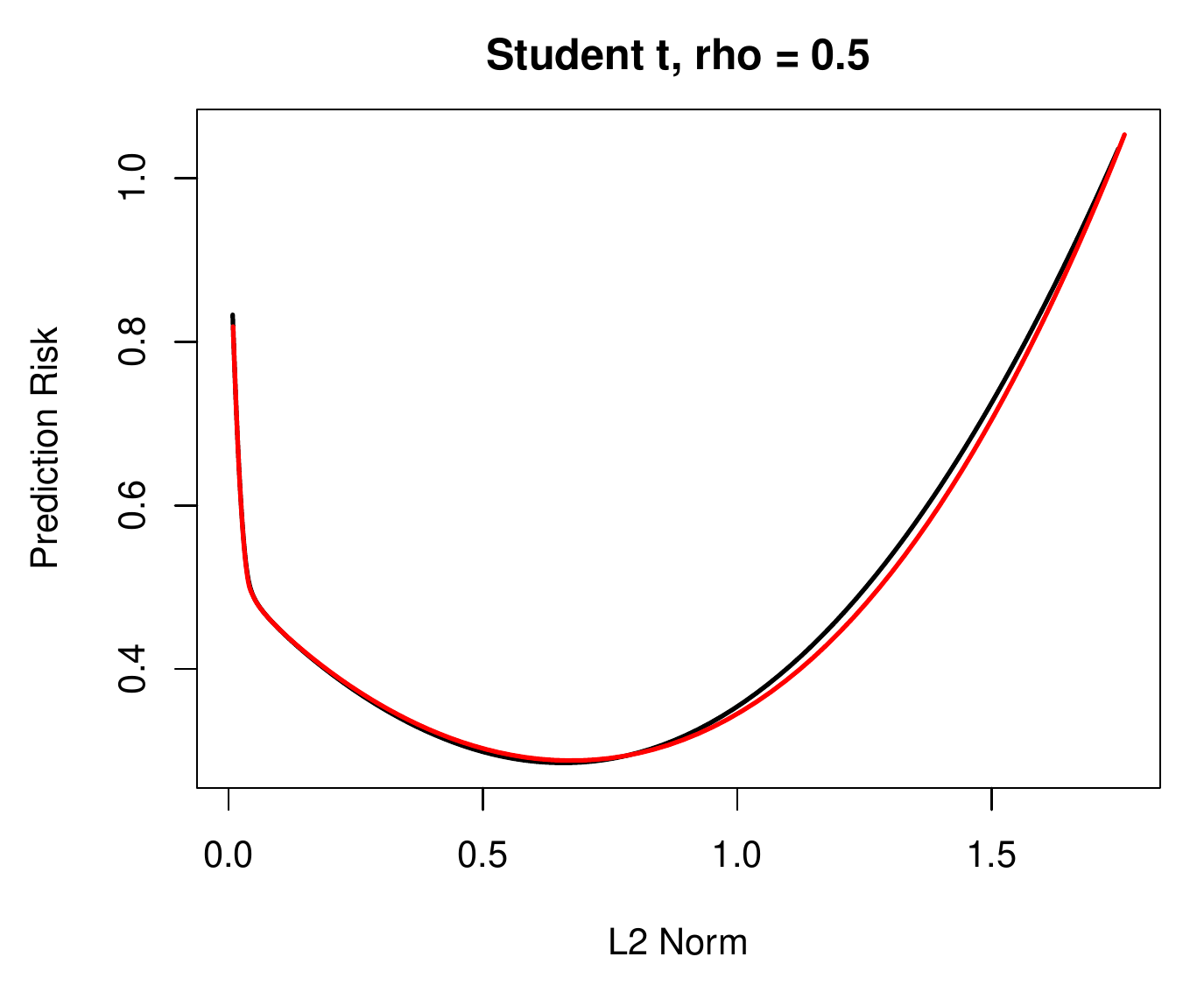}
\caption{\it \small Student t features, with $n=1000$ and $p=500$.}
\label{fig:risk_t_lo} 
\end{figure*}

\begin{figure*}[p]
\centering
\includegraphics[width=0.475\textwidth]{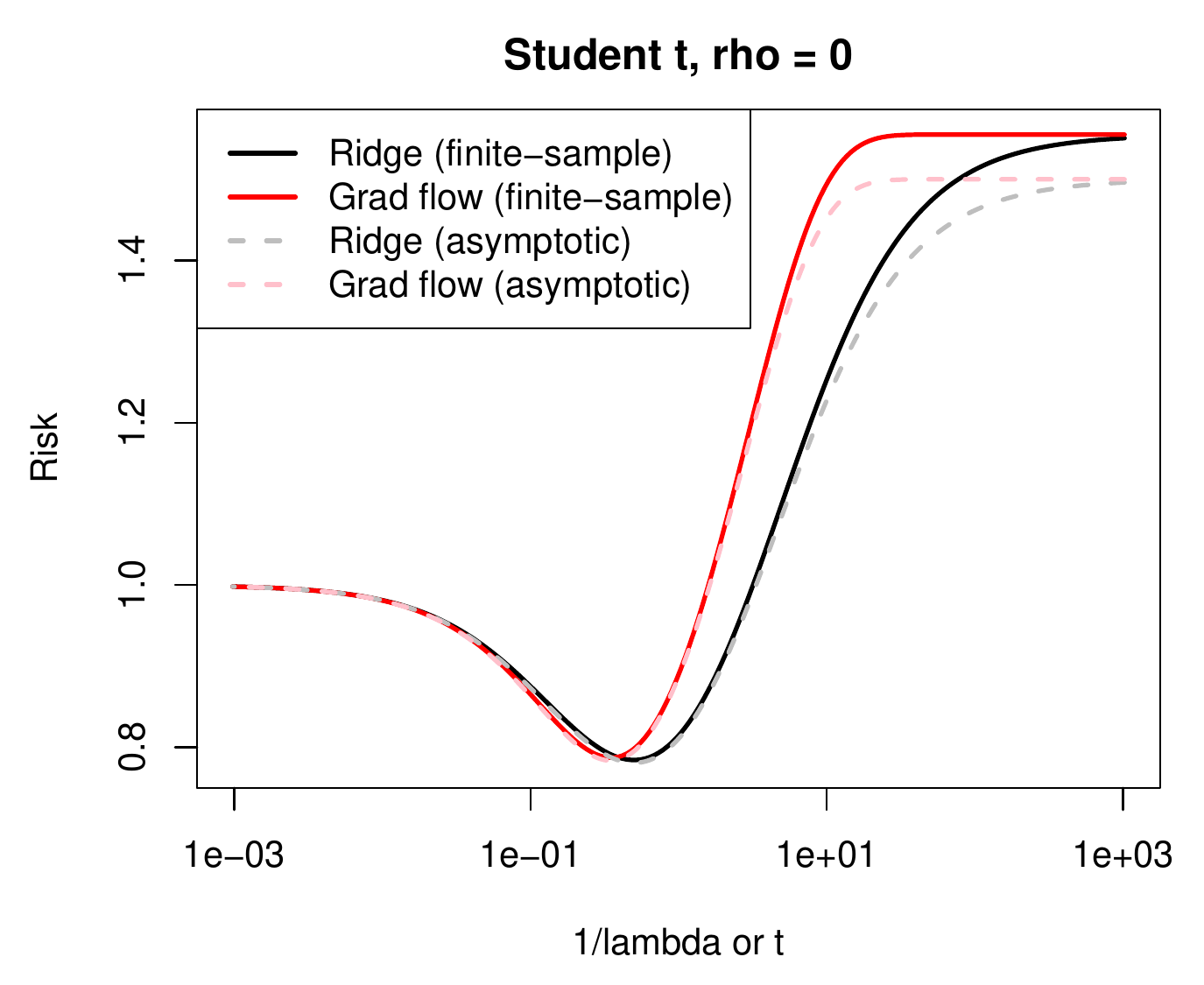} 
\includegraphics[width=0.475\textwidth]{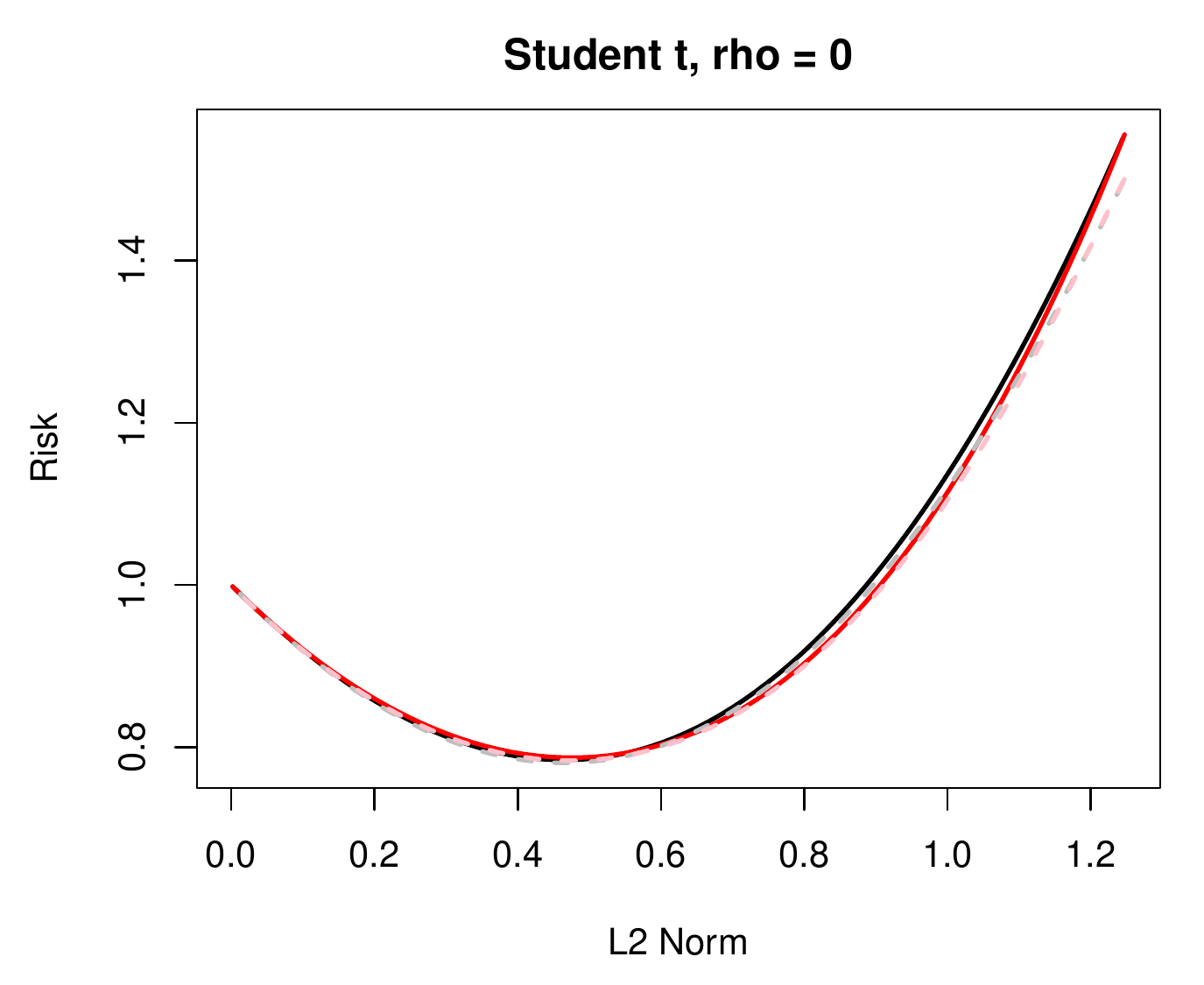}
\includegraphics[width=0.475\textwidth]{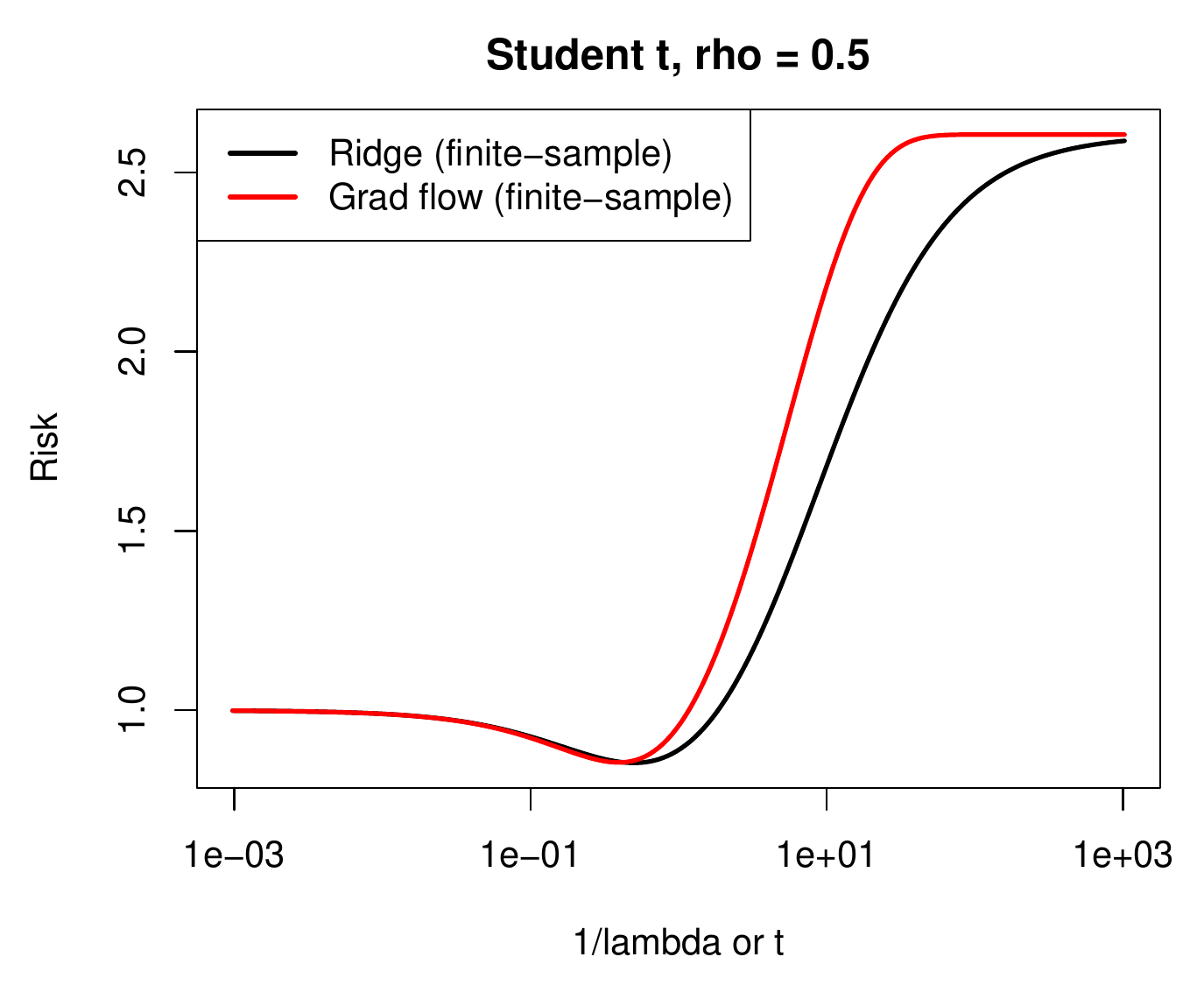} 
\includegraphics[width=0.475\textwidth]{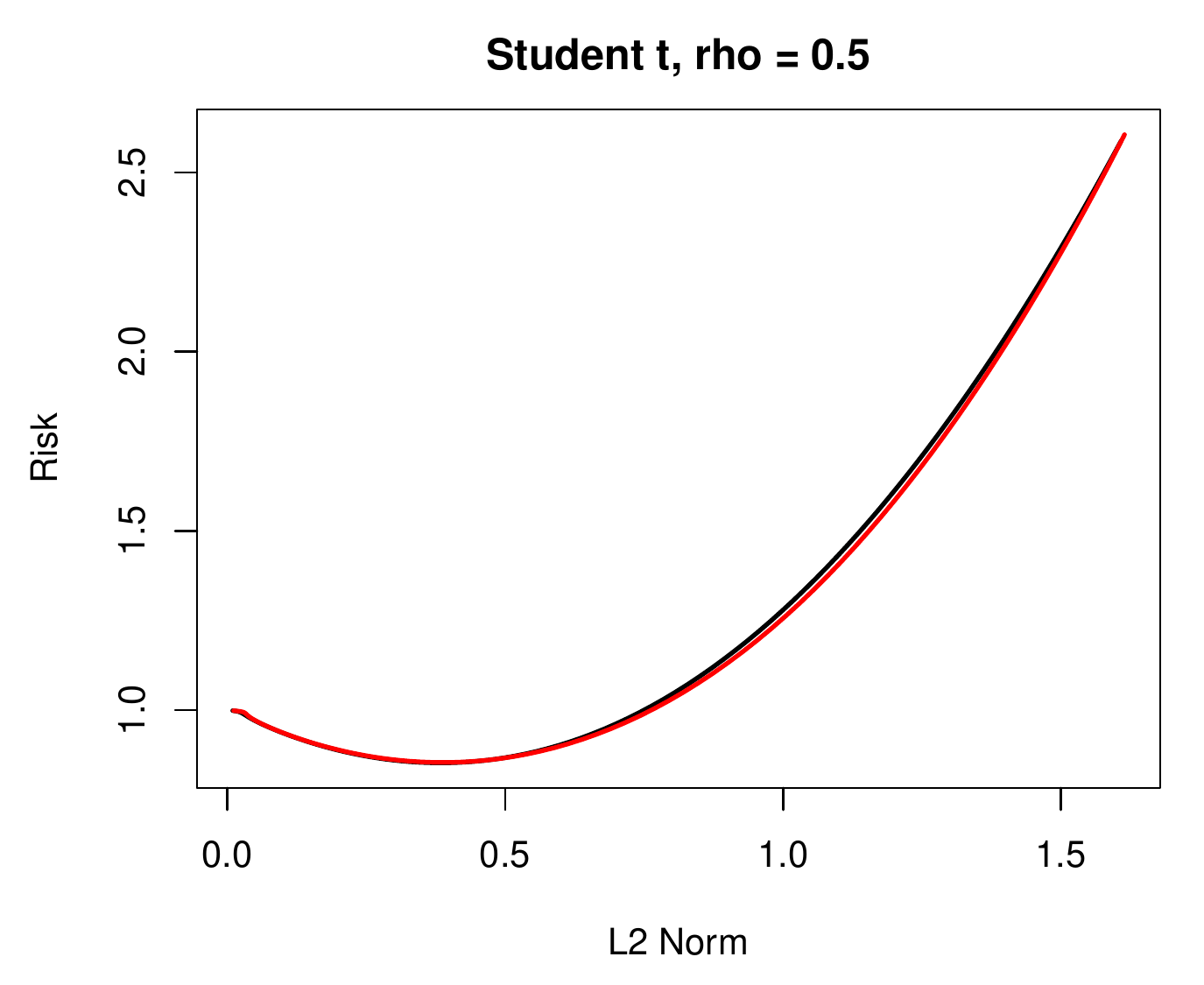}
\includegraphics[width=0.475\textwidth]{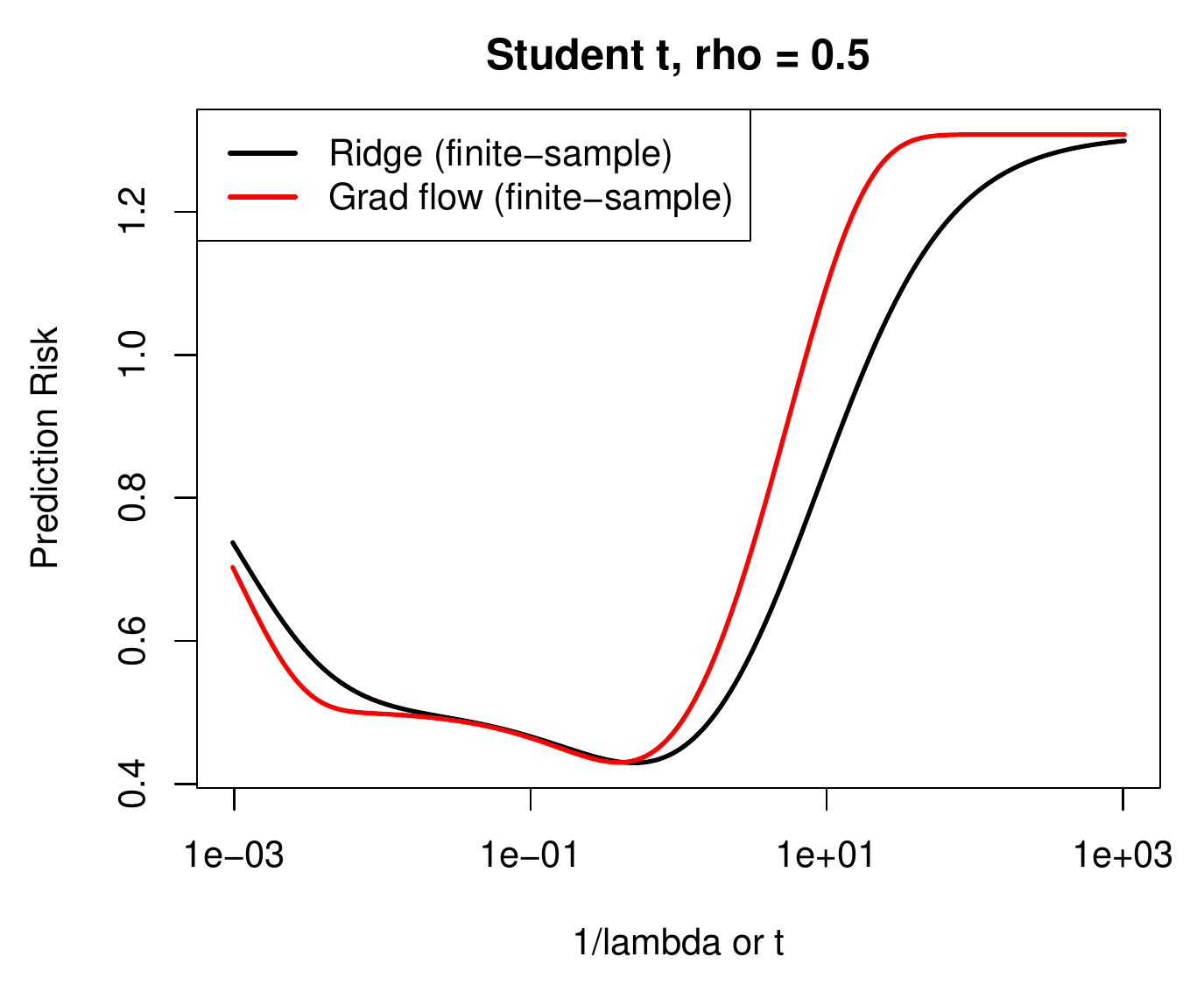} 
\includegraphics[width=0.475\textwidth]{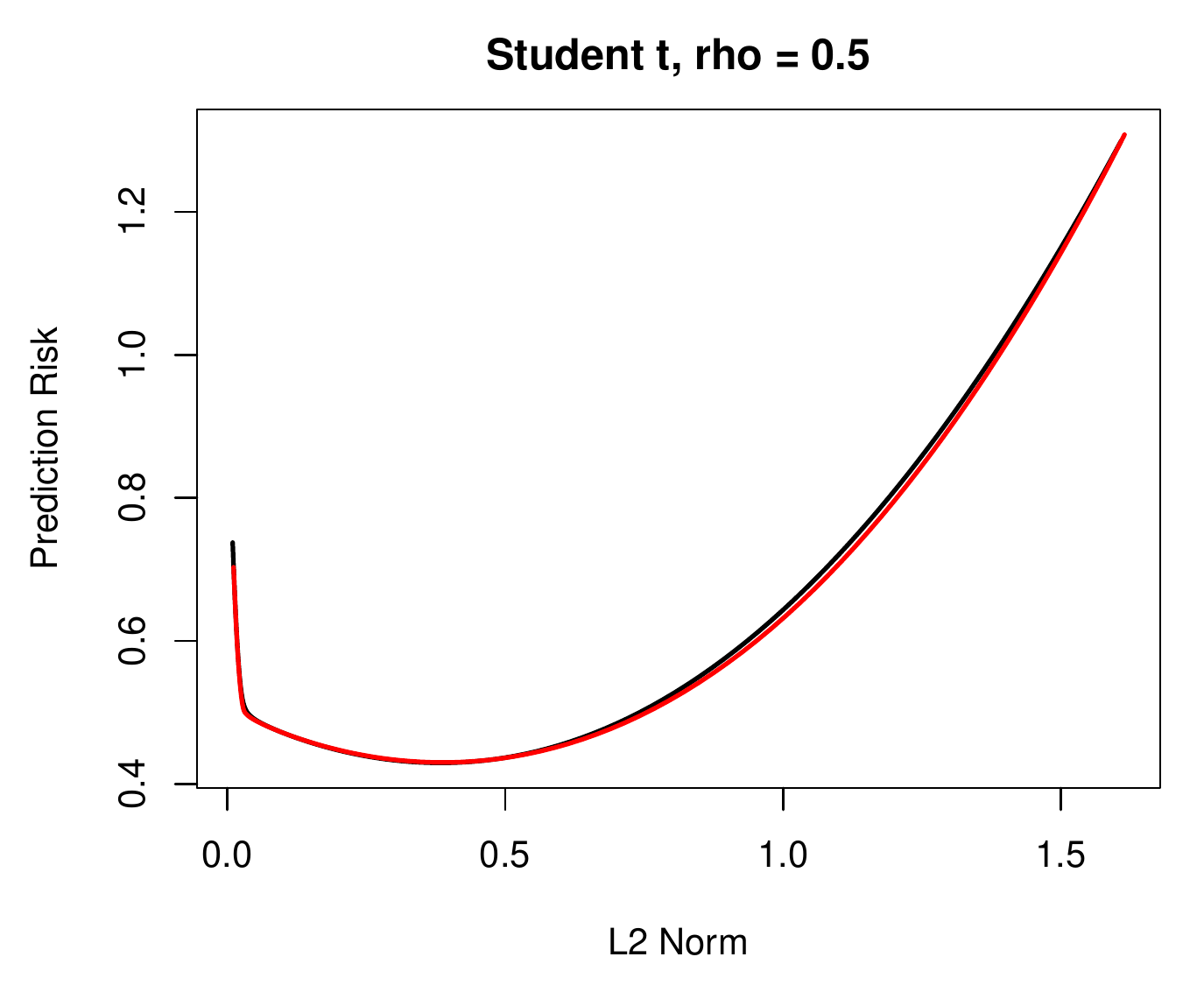}
\caption{\it \small Student t features, with $n=500$ and $p=1000$.}
\label{fig:risk_t_hi} 
\end{figure*}

\begin{figure*}[p]
\centering
\includegraphics[width=0.475\textwidth]{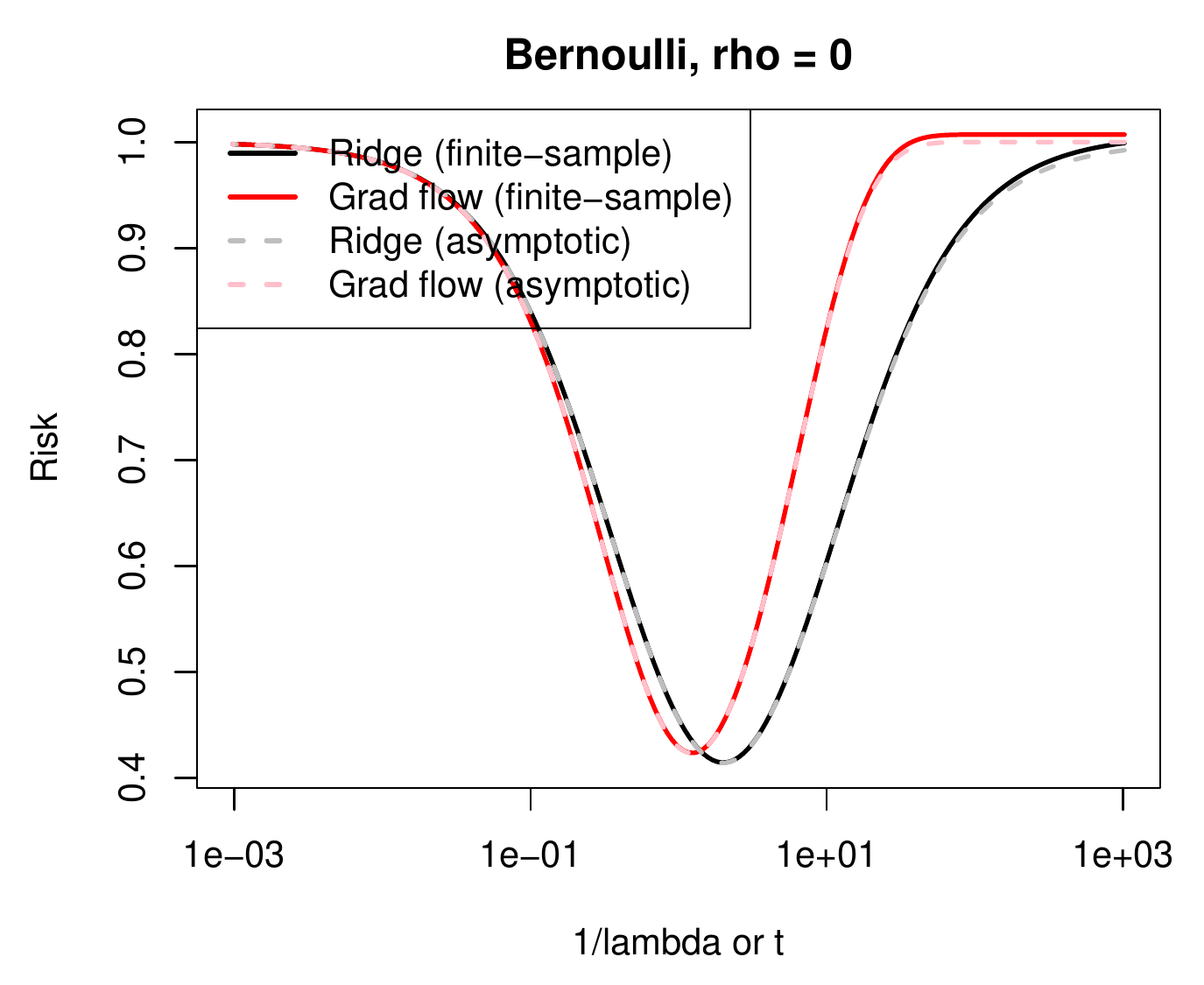} 
\includegraphics[width=0.475\textwidth]{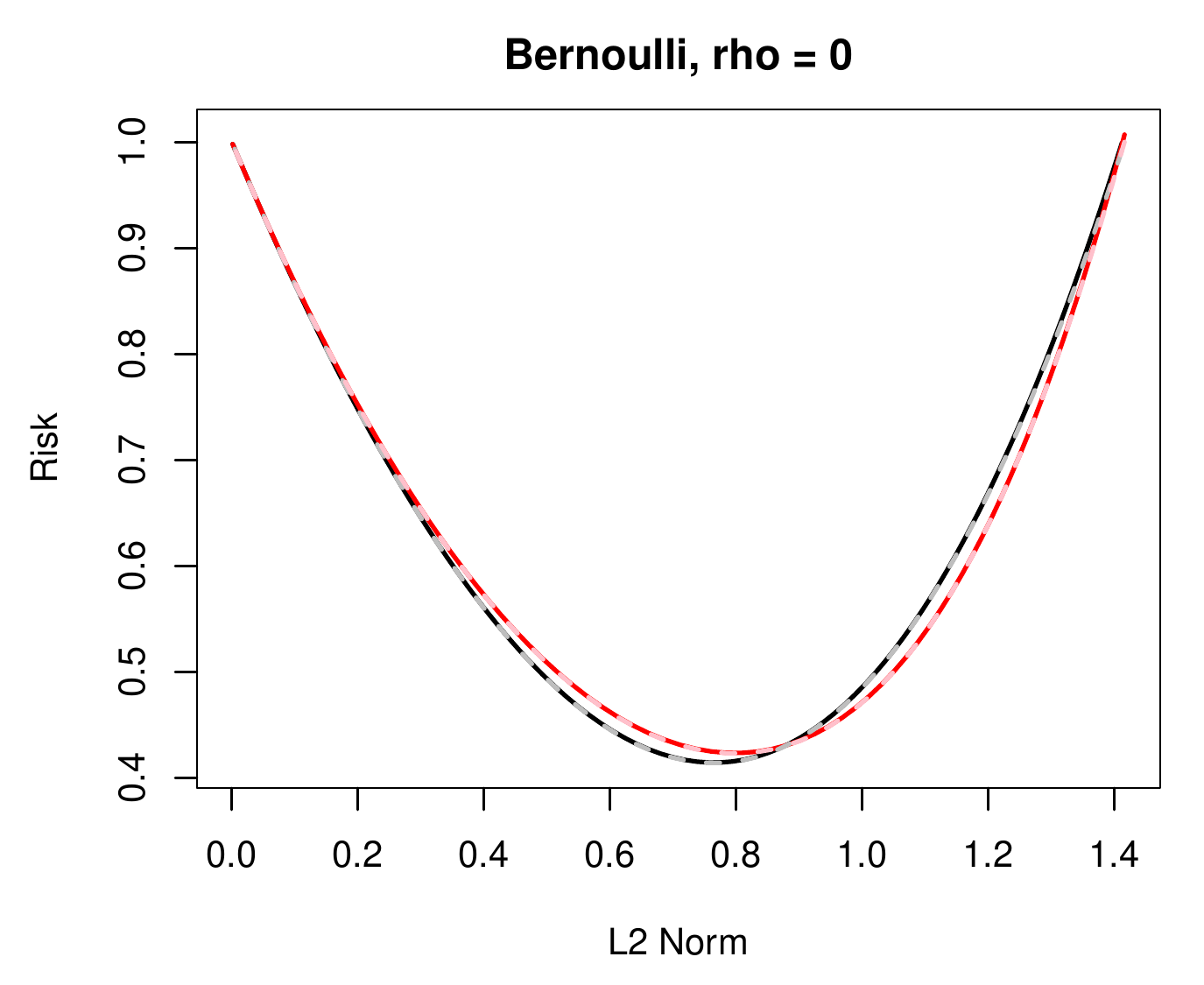}
\includegraphics[width=0.475\textwidth]{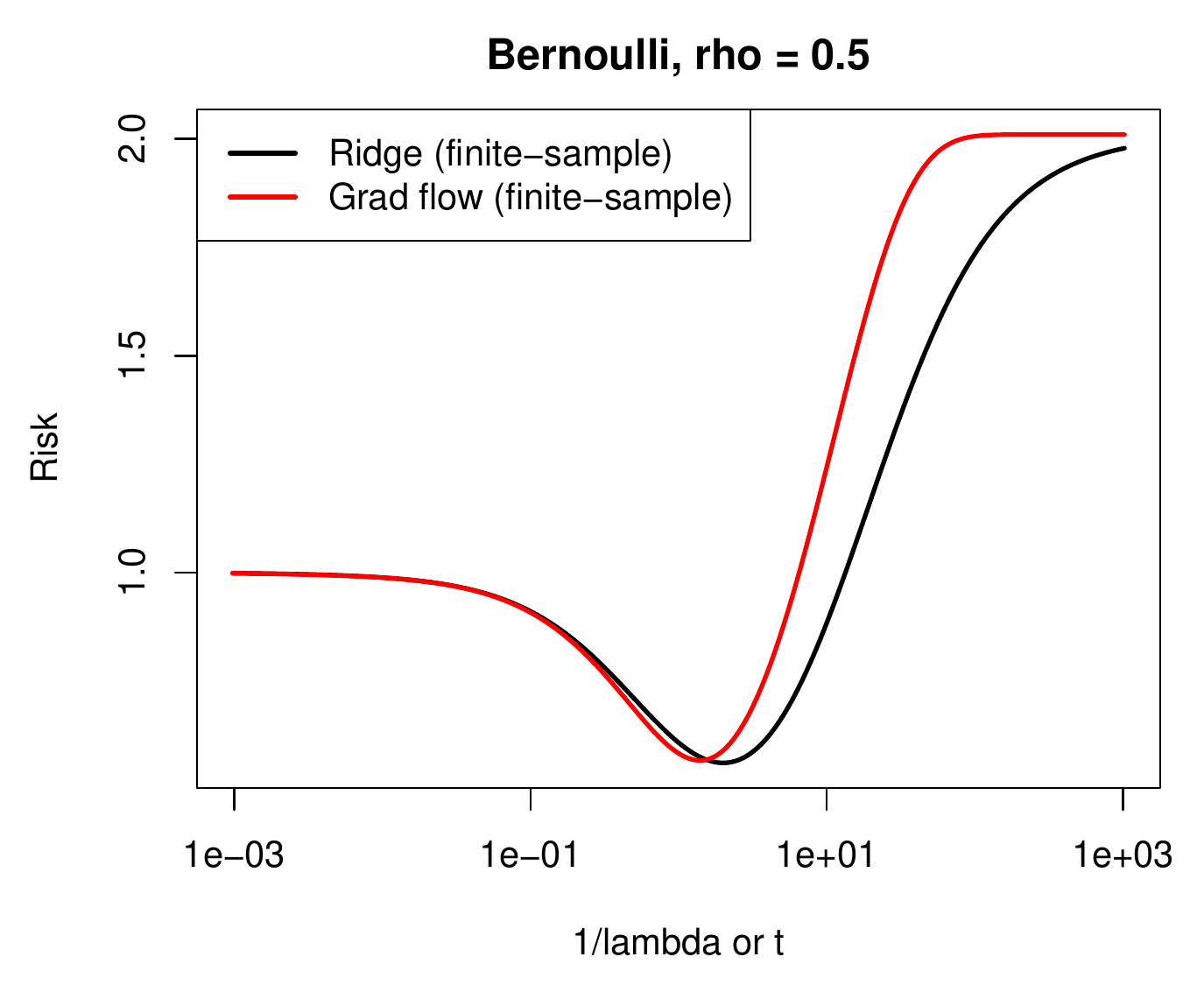} 
\includegraphics[width=0.475\textwidth]{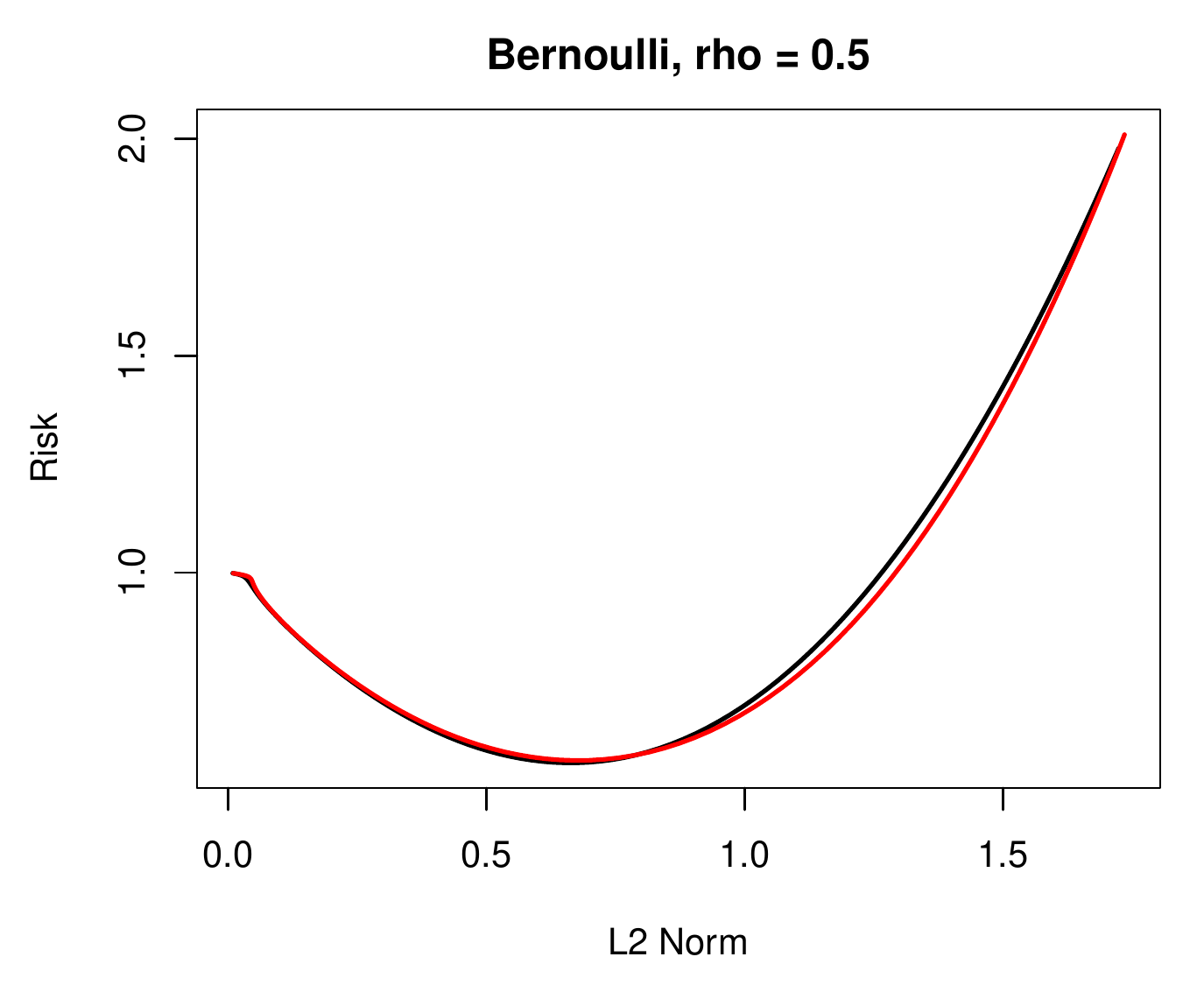}
\includegraphics[width=0.475\textwidth]{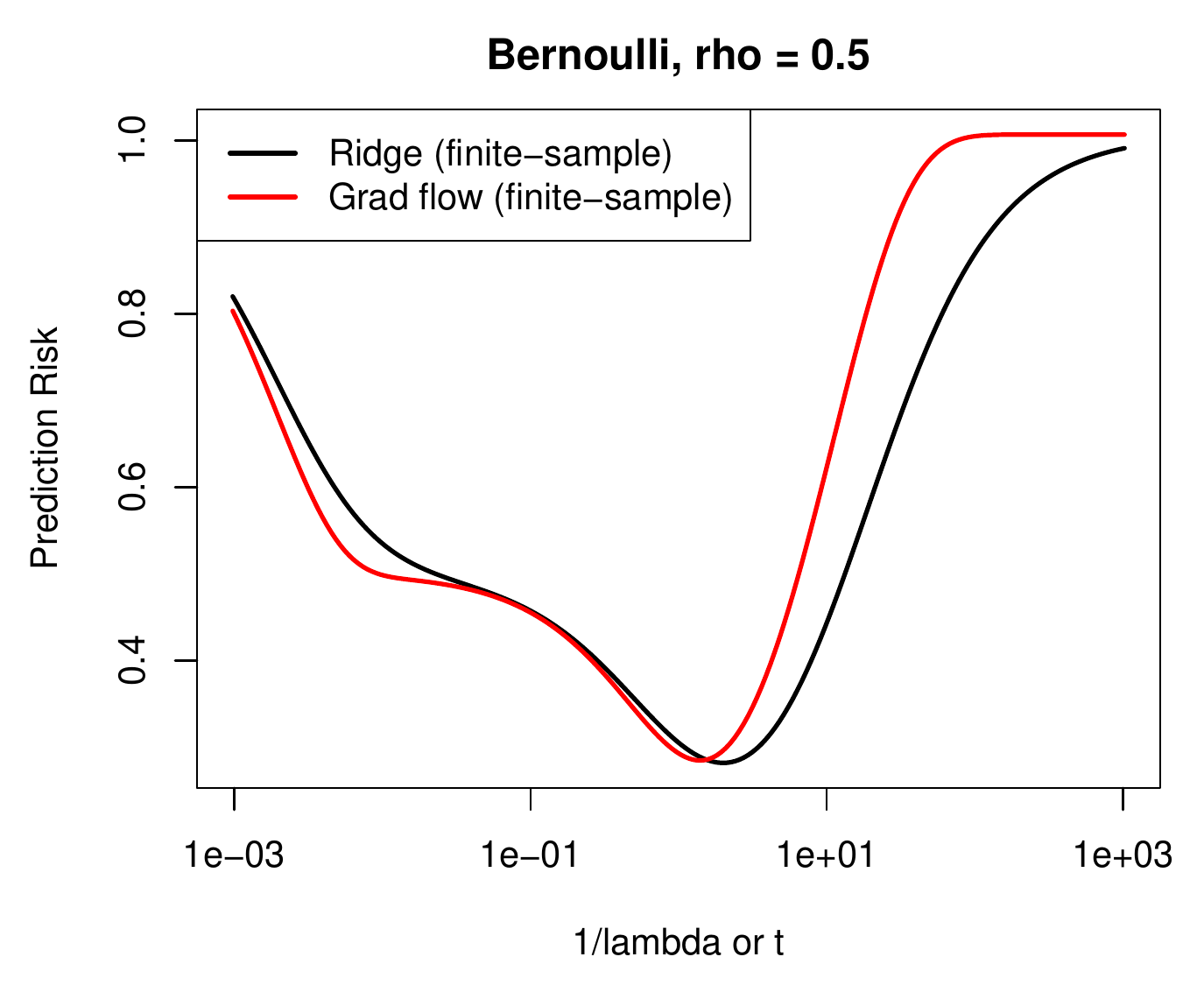} 
\includegraphics[width=0.475\textwidth]{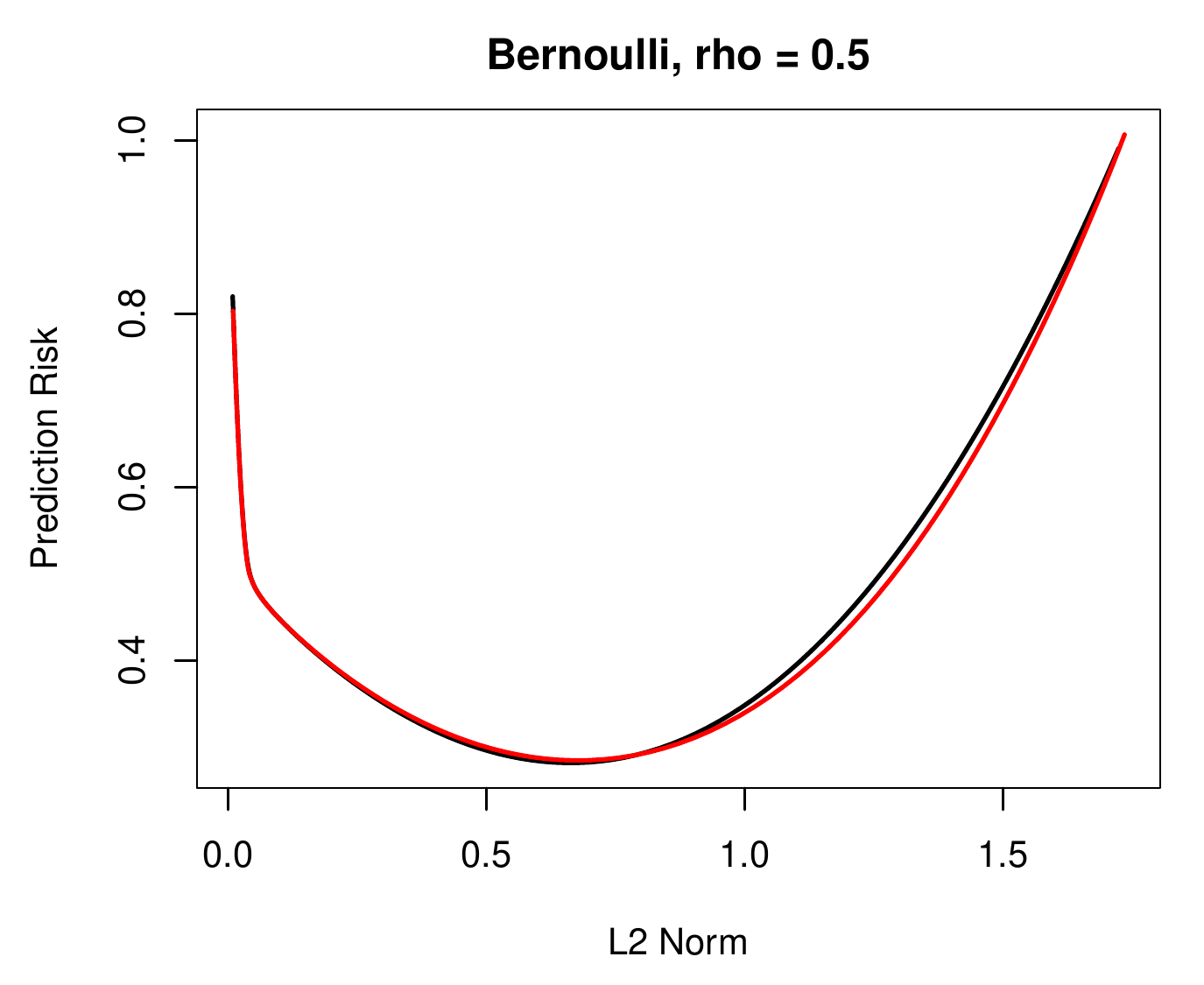}
\caption{\it \small Bernoulli features, with $n=1000$ and $p=500$.}
\label{fig:risk_bern_lo} 
\end{figure*}

\begin{figure*}[p]
\centering
\includegraphics[width=0.475\textwidth]{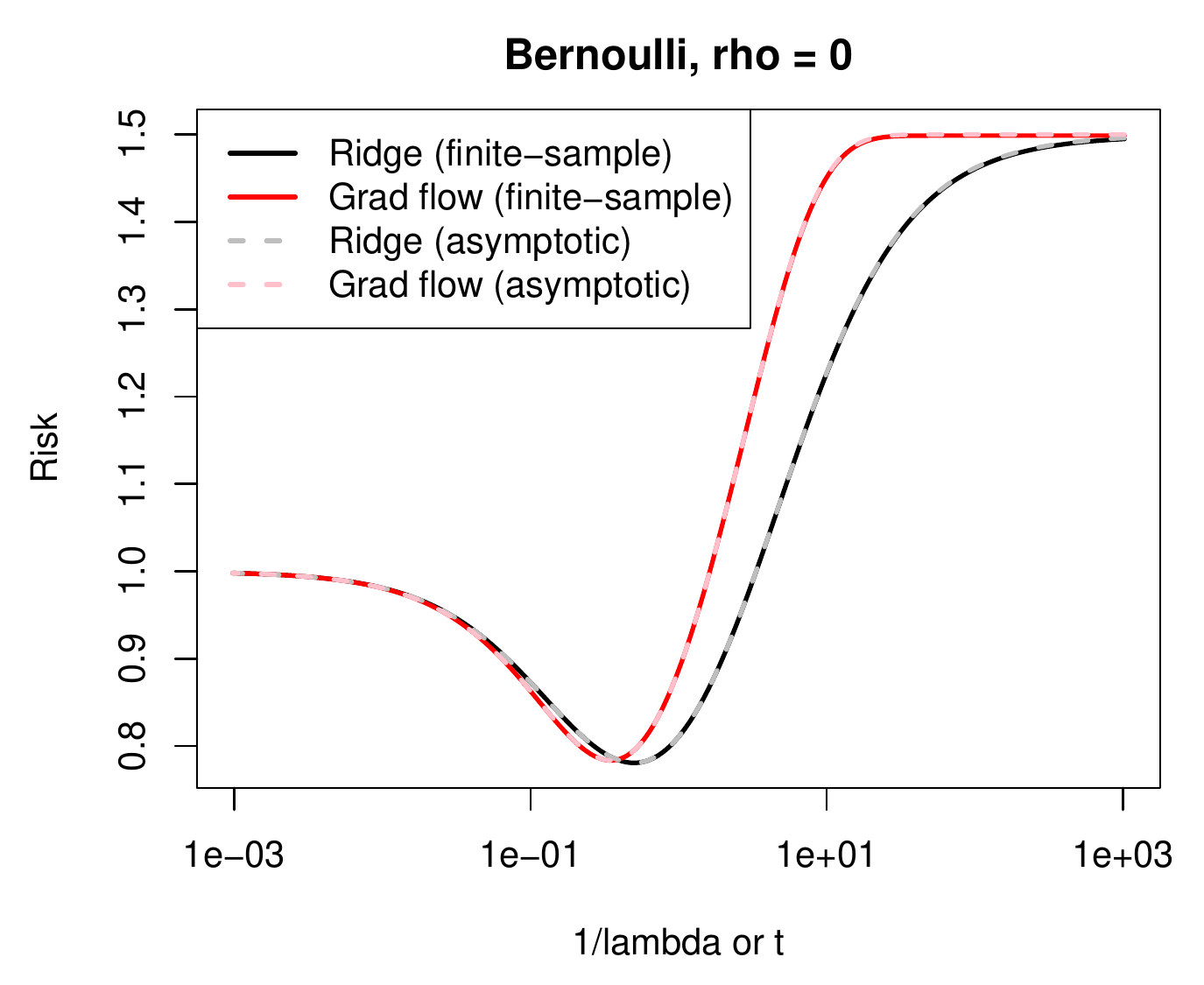} 
\includegraphics[width=0.475\textwidth]{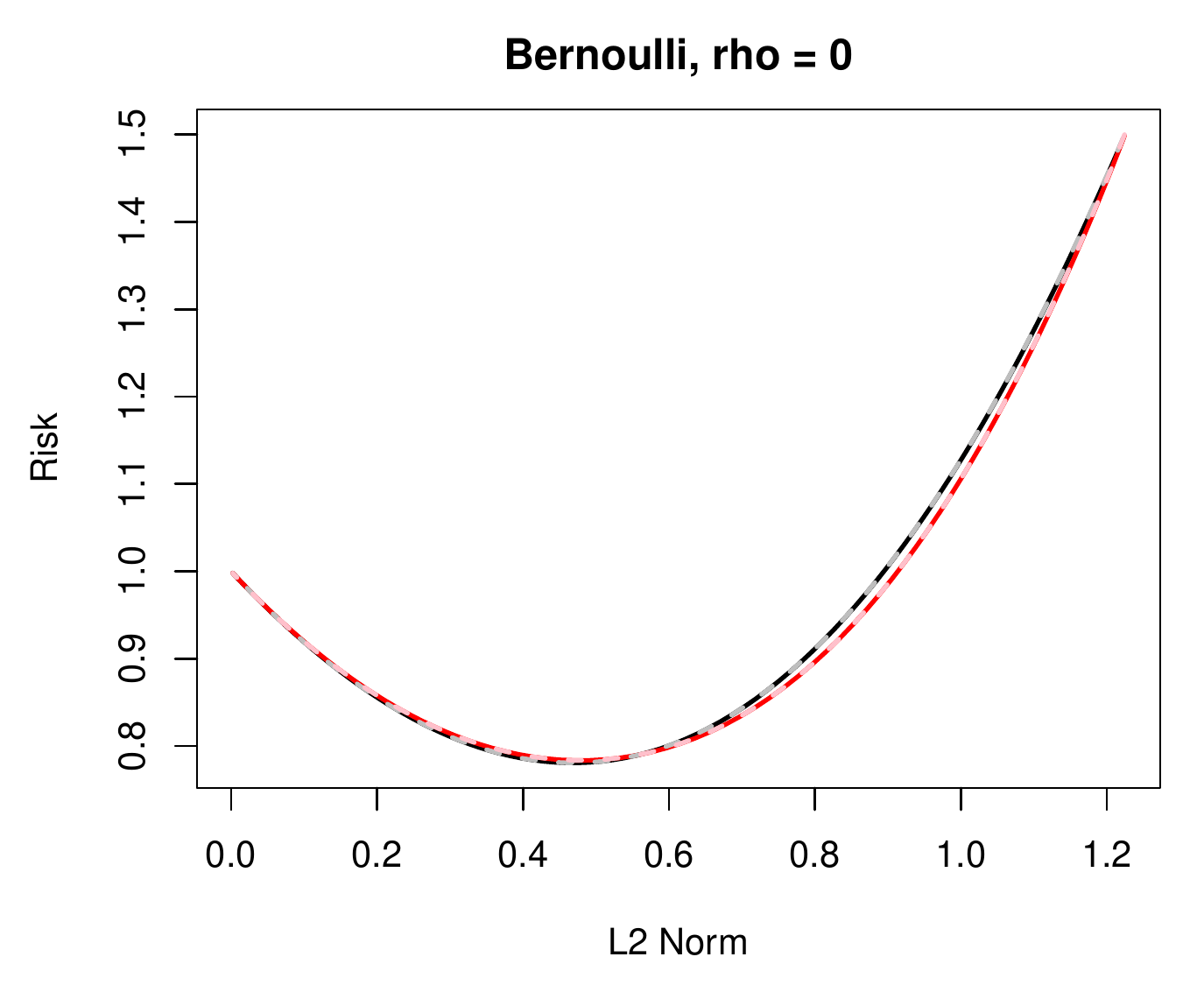}
\includegraphics[width=0.475\textwidth]{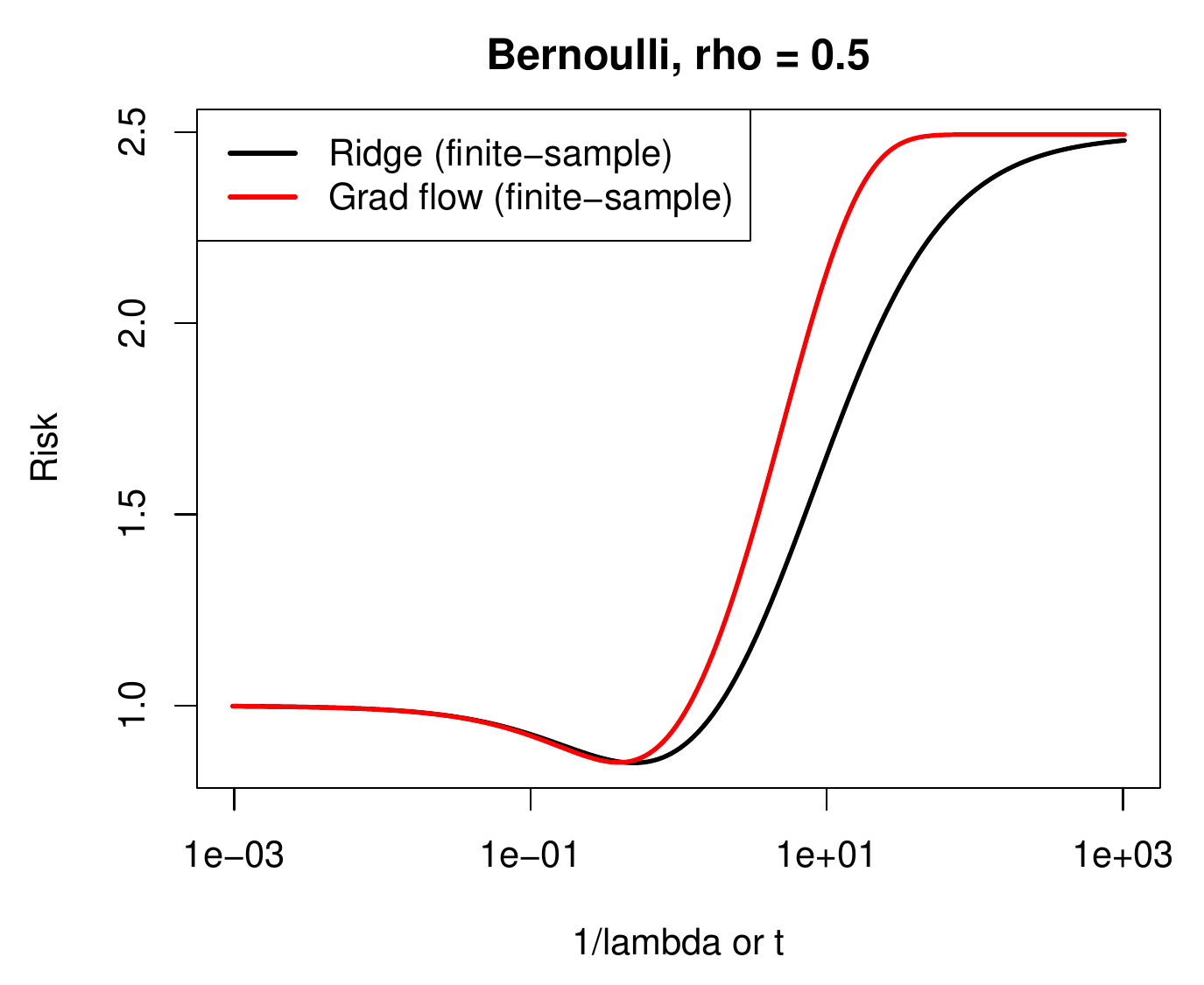} 
\includegraphics[width=0.475\textwidth]{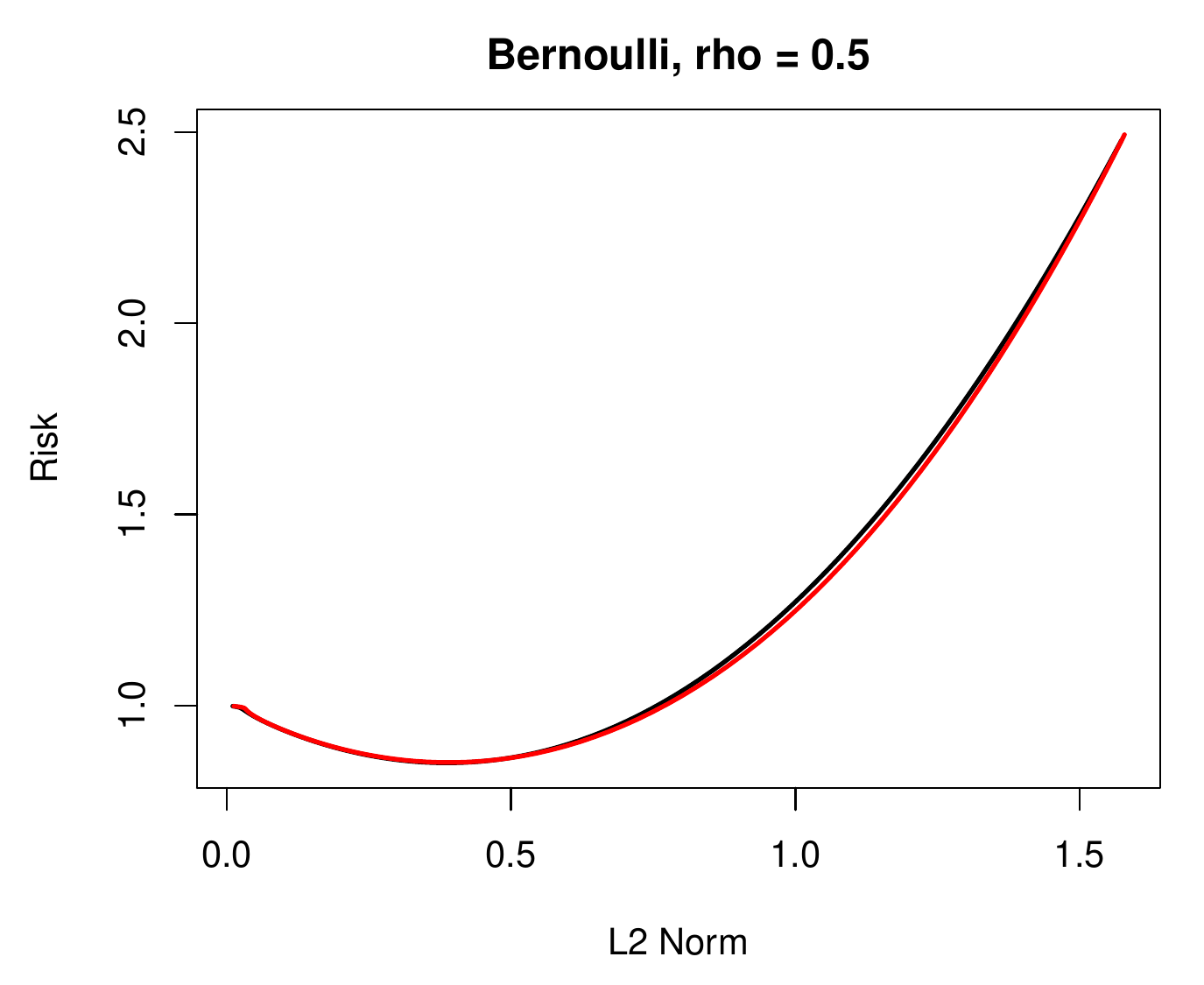}
\includegraphics[width=0.475\textwidth]{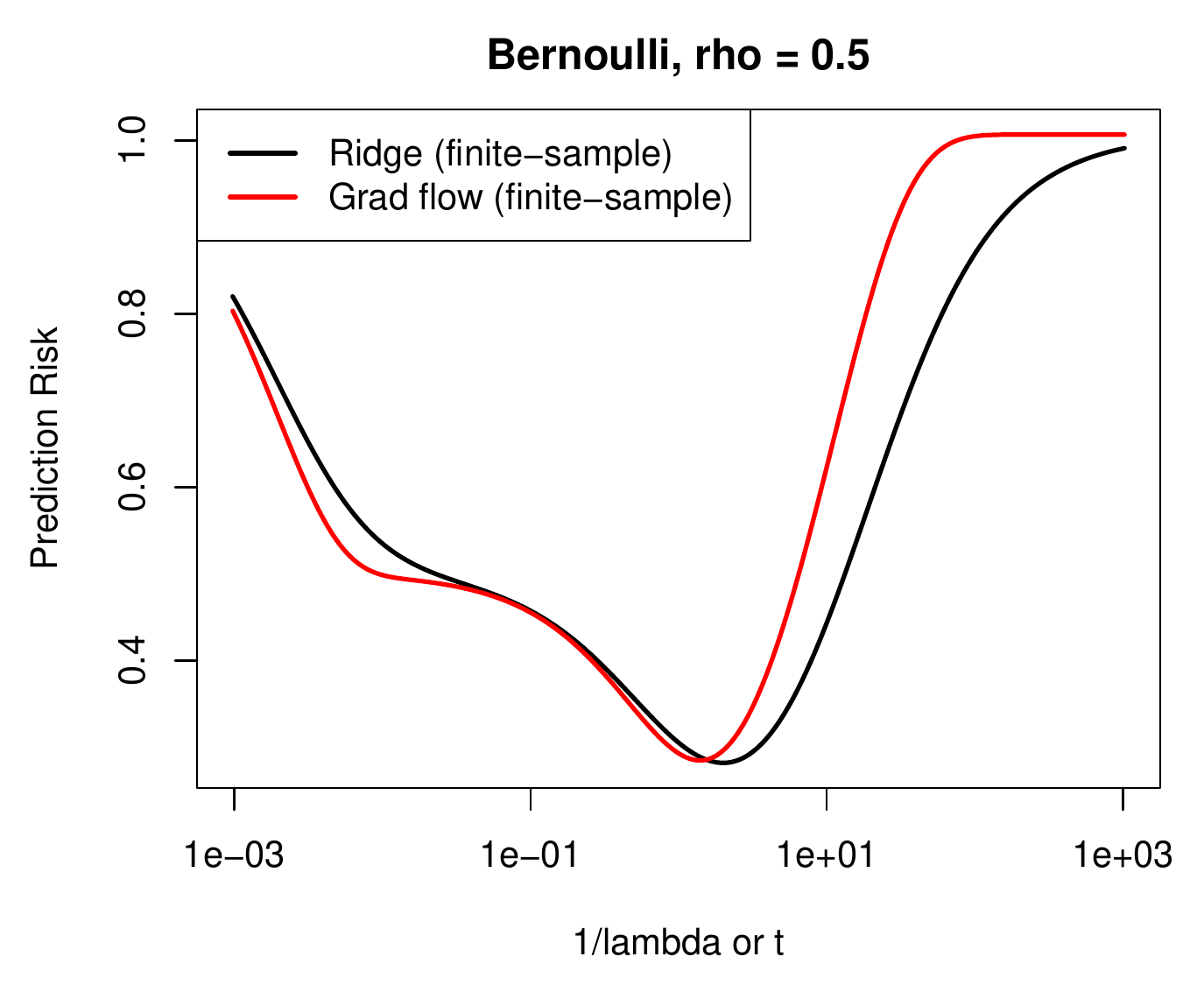} 
\includegraphics[width=0.475\textwidth]{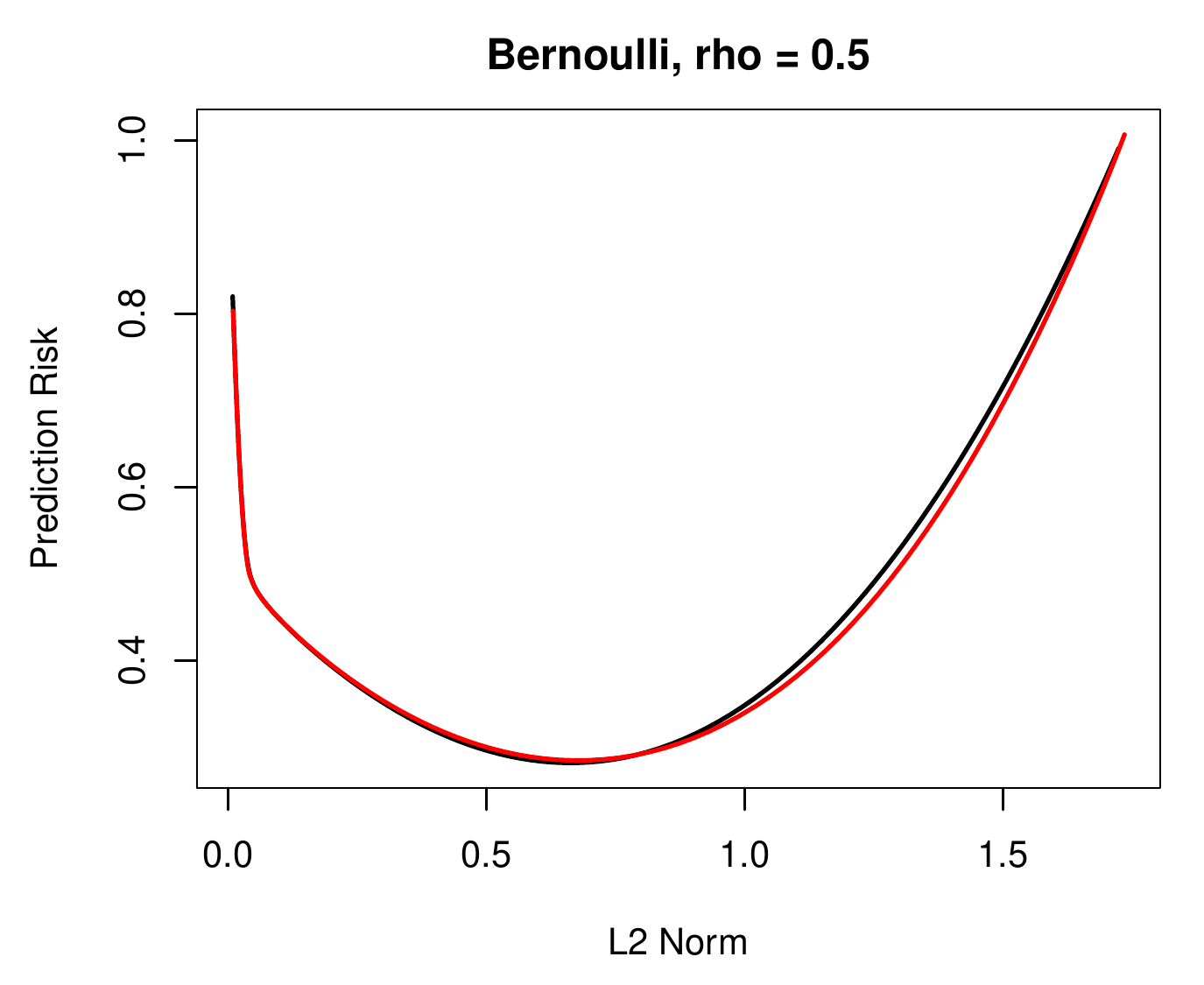}
\caption{\it \small Bernoulli features, with $n=500$ and $p=1000$.}
\label{fig:risk_bern_hi} 
\end{figure*}

\end{document}